\documentclass{article}
\usepackage[utf8]{inputenc}
\usepackage[T1]{fontenc}
\usepackage{amsmath,amsthm} 
\usepackage{amssymb}
\usepackage{bm}
\usepackage{mathtools}
\usepackage{todonotes}
\usepackage{xcolor}
\usepackage{enumitem}
\usepackage{csquotes}
\usepackage{pgffor}
\usepackage{pgfplots}
\usetikzlibrary{arrows.meta}
\usetikzlibrary{patterns}
\pgfplotsset{compat=1.14}
\usepgfplotslibrary{colorbrewer,patchplots}

\usepackage{mdframed}
\usepackage{xparse}
\usepackage[a4paper,top=2cm,bottom=2cm,left=3cm,right=3cm,marginparwidth=1.75cm]{geometry}

\usepackage{adjustbox, caption, subcaption}
\usepackage{tikz}
\usetikzlibrary{calc,positioning}

\usepackage[ruled]{algorithm2e}

\SetKwBlock{Iterate}{Iterate}{}

\usepackage[
natbib,
maxcitenames=3, 
maxbibnames=10,  
sorting=nyt,
sortcites,
defernumbers,
backend=biber,
backref,
doi=false,
url=false,
giveninits
]{biblatex}
\usepackage{tikz}
\usepackage{hyperref}
\definecolor{cadmiumgreen}{rgb}{0.0, 0.42, 0.24}
\definecolor{hotpink}{rgb}{1.0, 0.41, 0.71}
\definecolor{cgreen}{RGB}{0, 180, 100}
\definecolor{revc}{RGB}{180, 0, 100}
\hypersetup{
    colorlinks,
    linkcolor=cgreen,
    citecolor=cgreen,
    urlcolor=blue,
    linktocpage
}
\newcommand{\srev}[1]{%
#1\color{black}}
\newcommand{\rev}[1]{%
#1%
}
\usepackage[
capitalize, 
noabbrev,
nameinlink]{cleveref}

\usepackage{authblk}

\newcommand{\R}{\mathbb{R}}
\newcommand{\N}{\mathbb{N}}

\graphicspath{{./figures/}}

\DeclareMathOperator*{\esssup}{ess\,sup}
\DeclareMathOperator*{\argmin}{arg\,min}
\DeclareMathOperator*{\argmax}{arg\,max}

\newcommand{\norm}[1]{\left\|#1\right\|}
\newcommand{\abs}[1]{\left|#1\right|}
\newcommand{\set}[1]{\left\{#1\right\}}
\newcommand{\st}{\,:\,}

\DeclareMathOperator{\supp}{supp}

\newtheorem{theorem}{Theorem}[section]
\newtheorem*{theorem*}{Theorem}
\newtheorem{proposition}[theorem]{Proposition}
\newtheorem{lemma}[theorem]{Lemma}
\newtheorem{corollary}[theorem]{Corollary}
\newtheorem*{corollary*}{Corollary}
\theoremstyle{definition}
\newtheorem{definition}[theorem]{Definition}

\newcounter{asscount}
\newtheorem{assume}{Assumption}[asscount]

\Crefname{assume}{Assumption}{Assumptions}

\newtheoremstyle{bfnoteonly}%
{}{}%
{\itshape}{}%
{\bfseries}{}%
{ }%
{\thmnote{#3}}
\theoremstyle{bfnoteonly}
\newtheorem*{nothm}{}

\theoremstyle{plain}
\newcommand{\thistheoremname}{}
\newtheorem{genericthm}[section]{\thistheoremname}

\newtheorem*{genericthm*}{\thistheoremname}
\newenvironment{namedtheorem*}[1]
  {\renewcommand{\thistheoremname}{#1}%
   \begin{genericthm*}}
  {\end{genericthm*}}

\newtheorem{example}{Example}
\theoremstyle{remark}
\newtheorem{remark}[theorem]{Remark}

\newlist{propenum}{enumerate}{1} 
\setlist[propenum]{label=(\roman*),
                    ref=\theproposition~(\roman*)}
\crefalias{propenumi}{proposition}
\numberwithin{equation}{section}

\renewcommand{\vec}[1]{\mathbf{#1}}
\newcommand{\closure}[1]{\overline{#1}}
\newcommand{\cl}[1]{\closure{#1}}

\newcommand{\CMS}{\mathcal{S}}

\DeclareMathOperator{\Lip}{Lip}
\newcommand{\sign}{\operatorname{sign}}
\newcommand{\dom}{\operatorname{dom}}
\newcommand{\loss}{\ell}
\newcommand{\Inp}{\mathcal{X}}
\newcommand{\Inpp}{\mathcal{Z}}
\newcommand{\Oup}{\mathcal{Y}}
\newcommand{\hyp}{h}
\newcommand{\Hyp}{\mathcal{H}}
\newcommand{\budget}{\varepsilon}
\newcommand{\x}{x}
\newcommand{\xatt}{\tilde{\x}}
\newcommand{\xx}{z}
\newcommand{\func}{\mathcal{E}}
\newcommand{\funccp}{\rev{E}}
\newcommand{\ffunc}{\rev{\func}}
\newcommand{\nimap}{\psi}
\newcommand{\mto}{\rightrightarrows}
\newcommand{\Ban}{\mathcal{X}}
\newcommand{\WBan}{\Ban}
\newcommand{\chara}{\chi}
\newcommand{\dual}{\mathcal{J}}

\newcommand{\funccpc}{\funccp^{\mathrm{c}}}
\newcommand{\funccpd}{\funccp^{\mathrm{d}}}

\newcommand{\funccps}{\funccp^{\mathrm{sl}}}
\newcommand{\unitb}{B_1}

\newcommand{\Prob}{\mathcal{P}}
\newcommand{\seq}[1]{\left(#1^{n}\right)_{n\in\N}}
\renewcommand{\div}{\mathrm{div}}
\newcommand{\EC}{\mathrm{EC}}
\newcommand{\AC}{\mathrm{AC}}
\newcommand{\ddim}{d}
\newcommand{\Id}{Id}
\newcommand{\pot}{E}
\newcommand{\potc}{\pot^\text{c}}
\newcommand{\potd}{\pot^\text{d}}

\newcommand{\FG}{\mathrm{FGS}}
\newcommand{\IFG}{\mathrm{IFGS}}
\newcommand{\Clip}{\mathrm{Clip}}
\newcommand{\actfun}{\alpha}
\NewDocumentCommand{\xm}{e{^_}}{%
  x^{\IfValueT{#1}{#1}}_{\IfValueT{#2}{#2}}%
}%
\NewDocumentCommand{\sx}{e{^_}}{%
  \bar{x}^{\IfValueT{#1}{#1}}_{\IfValueT{#2}{#2}}%
}%
\NewDocumentCommand{\vx}{e{^_}}{%
  \tilde{x}^{\IfValueT{#1}{#1}}_{\IfValueT{#2}{#2}}%
}%
\NewDocumentCommand{\xfg}{e{^_}}{%
  x^{\IfValueT{#1}{#1}}_{\FG\IfValueT{#2}{,#2}}%
}%
\NewDocumentCommand{\xifg}{e{^_}}{%
  x^{\IfValueT{#1}{#1}}_{\IFG\IfValueT{#2}{,#2}}%
}%
\NewDocumentCommand{\xs}{e{^_}}{%
  x^{\IfValueT{#1}{#1}}_{\mathrm{si}\IfValueT{#2}{,#2}}%
}%
\NewDocumentCommand{\sxs}{e{^_}}{%
  \bar{x}^{\IfValueT{#1}{#1}}_{\rev{\mathrm{si}}\IfValueT{#2}{,#2}}%
}%
\NewDocumentCommand{\vxs}{e{^_}}{%
  \tilde{x}^{\IfValueT{#1}{#1}}_{\rev{\mathrm{si}}\IfValueT{#2}{,#2}}%
}%
\newcommand{\velo}{\rev{\bm{v}}}

\definecolor{mediumelectricblue}{rgb}{0.21, 0.61, 0.99}

\newcommand{\nc}{\color{black}}
\newcommand{\muatt}{\tilde{\mu}}

\makeatletter
\define@key{todonotes}{TR}[]{%
    \setkeys{todonotes}{author=TR,color=mediumelectricblue,size=small}}%
\define@key{todonotes}{TR4TR}[]{%
    \setkeys{todonotes}{author=For TR,color=gray,
    size=small}}%
\define@key{todonotes}{MB}[]{%
    \setkeys{todonotes}{author=MB,color=red,size=small}}%
\define@key{todonotes}{LW}[]{%
    \setkeys{todonotes}{author=LW,color=red,size=small}}%
\makeatother   

\definecolor{hotpink}{RGB}{255, 2, 141}
\definecolor{darkbeige}{RGB}{172, 147, 98}
\definecolor{emerald}{RGB}{1, 160, 73}
\definecolor{apple}{RGB}{118, 205, 38}
\definecolor{sky}{RGB}{130, 202, 252}

\title{Adversarial flows: A gradient flow characterization of adversarial attacks}
\author[1]{Lukas Weigand}
\author[1,*]{Tim Roith}
\author[1,2]{Martin Burger}

\affil[1]{Helmholtz Imaging, Deutsches Elektronen-Synchrotron DESY, Notkestr. 85, 22607 Hamburg, Germany}
\affil[2]{Fachbereich Mathematik, Universit\"at Hamburg, Bundesstr. 55, 20146 Hamburg, Germany}
\affil[*]{Corresponding author: tim.roith@desy.de}
\begin{document}

\maketitle
\begin{abstract}
A popular method to perform adversarial attacks on neuronal networks is the so-called fast gradient sign method and its iterative variant. In this paper, we interpret this method as an explicit Euler discretization of a differential inclusion, where we also show convergence of the discretization to the associated gradient flow. To do so, we consider the concept of $p$-curves of maximal slope in the case $p=\infty$. We prove existence of $\infty$-curves of maximum slope and derive an alternative characterization via differential inclusions. Furthermore, we also consider Wasserstein gradient flows for potential energies, where we show that curves in the Wasserstein space can be characterized by a representing measure on the space of curves in the underlying Banach space, which fulfill the differential inclusion. The application of our theory to the finite-dimensional setting is twofold: On the one hand, we show that a whole class of normalized gradient descent methods (in particular signed gradient descent) converge, up to subsequences, to the flow, when sending the step size to zero. On the other hand, in the distributional setting, we show that the inner optimization task of adversarial training objective can be characterized via $\infty$-curves of maximum slope on an appropriate optimal transport space.

\par\vskip\baselineskip\noindent
\end{abstract}

\tableofcontents

\section{Introduction}\label{ch: intro}

This paper considers gradient flows in metric spaces, following the seminal work by \textcite{ambrosio05}. 
There, the authors introduce the concept of $p$-curves of maximal slope, with its origins dating back to \cite{DeGiorgi80}. This concept is further generalized in \cite{rossi2008metric}.  As for our main contribution, we study the less known limit case $p=\infty$ and adapt current theory to this setting.
The main incentive for our work is the adversarial attack problem as introduced in \cite{goodfellow2014explaining, szegedy2013intriguing}. Here one considers a classification task, where a classifier $\hyp:\Inp\to\Oup$---typically parametrized as a neural network---is given an input $\x\in\Inp$, which it \textit{correctly} classifies as $y\in\Oup$, where $\Oup$ is assumed to be a subset of a finite dimensional vector space. The goal is to obtain a perturbed input $\xatt\in\Inp$, the \textit{adversarial example}, which is misclassified, while its difference to $\x$ is \enquote{imperceptible}. In practice, the latter condition is enforced by requiring that $\xatt$ has at most distance $\budget$ to $\x$ in an $\ell^p$ distance, where $\budget>0$ is called the \textit{adversarial budget}. Given some loss function $\loss:\Oup\times\Oup\to\R$, one then formulates the \textit{adversarial attack} problem \cite{goodfellow2014explaining, szegedy2013intriguing},
\begin{align}\label{eq: adver}\tag{AdvAtt}
\sup_{\xatt \in \cl{B_\budget}(\x)} \loss(\hyp(\xatt),y).
\end{align}
The above problem is also called an \textit{untargeted} attack, since we are solely interested in the misclassification. This is opposed to \textit{targeted} attacks, where one prescribes $y_{\text{target}}\in\Oup$ and wants to obtain an adversarial example, s.t. $\hyp(\xatt) = y_{\text{target}}$. This basically amounts to changing the loss function in \labelcref{eq: adver}, namely to $-\loss(\cdot,y_{\text{target}})$, without changing the inherent structure of the problem, which is why we do not consider it separately in the following. Methods for generating adversarial examples include \rev{first-order} attacks \cite{moosavi2016deepfool,brendel2019accurate,pintor2021fast}, momentum-variants \cite{dong2018boosting}, second order attacks \cite{jang2017objective} or even zero order attacks, not employing the gradient of the classifier \cite{brendel2017decision,ilyas2018black}. Especially for classifiers induced by neural networks, it was noticed \rev{in \cite{szegedy2013intriguing}} that approximate maximizers of \labelcref{eq: adver} completely corrupt the classification performance, even for a very small budget $\budget$. %
\rev{%
This observation created severe concerns about the robustness and reliability of neural networks (see e.g. \cite{kurakin2018adversarial}) and has sparked a general interest in both the adversarial attack and also the defense problem. The connection between the attack and defense task, was already introduced in \cite{goodfellow2014explaining}, where the authors
}
propose \textit{adversarial training} (similarly derived in \cite{madry2017towards,kurakin2016adversarial})\rev{. Here,} the standard empirical risk minimization is modified to
\begin{align}\label{eq:advtrain}\tag{AdvTrain}
\inf_{h\in \mathcal{H}} \sum_{(\x,y)\in \mathcal{T}}
\sup_{\xatt\in \cl{B_\budget}(\x)} \loss(h(\xatt), y)
\end{align}
for a training set $\mathcal{T}\subset \Inp\times \Oup$ and a \rev{hypothesis class} $\mathcal{H}\subset \{\hyp|\hyp\st\Inp\to\Oup\}$. Since this requires solving \labelcref{eq: adver} for every data point $x$, the authors then propose an efficient one-step method, called \textbf{F}ast \textbf{G}radient \textbf{S}ign \textbf{M}ethod (FGSM),
\begin{align}\label{eq: FGSM}\tag{FGSM}
\xfg=\x+\budget\, \sign(\nabla_\x \ell(h(\x),y)).
\end{align}
The motivation, as provided in \cite{goodfellow2014explaining}, was to consider a linear model $\x\mapsto \langle w, \x\rangle$, with weights $w$. The maximum over the input $x$ constrained to the budget ball $\cl{B^\infty_\budget}(x)$ is then attained in a corner of the hypercube, which validates the use of the sign. From a practical perspective, also for more complicated models, the sign operation ensures, that $\xfg \in \partial B_\budget^\infty(\x)$, i.e., $\xfg$ uses all the given budget in the $\ell^\infty$ distance, after just one update step. This adversarial training setup was similarly employed in \cite{wong2020fast, shafahi2019adversarial,madry2017towards,Roth2020} and analyzed as regularization of the empirical risk in \cite{bungert2023geometry,bungert2024}. For other strategies to obtain robust classifiers, we refer, e.g., to \cite{bungert2021clip,gouk2020regularisation,Krishnan2020,pauli2021training}. In situations, where only the attack problem is of interest, multistep methods are feasible, which led to the iterative FGS method \cite{kurakin2018adversarial,kurakin2016adversarial}
\begin{align}\label{eq: IFGSM}\tag{IFGSM}
\xifg^{k+1}=\Pi_{\cl{B^p_\budget}(\x)}(\xifg^k+\tau\, \sign(\nabla_\x \loss(\hyp(\xifg^k),y)) ),
\end{align}
where $\tau>0$ now defines a step size and $\Pi_{\cl{B^p_\budget}(x)}$ denotes the orthogonal \rev{projection to the $\budget$-ball in the $\ell^p$-norm} around the original image. Originally, the case $p=\infty$ was employed, where the projection is then a simple clipping operation. Other choices of $p$ are usually limited to $\{0,1,2\}$, which is also due to the computational effort of computing the projection (see \cite{pintor2021fast} for $p=0$ and \cite{duchi2008efficient} for $p=1$). Signed gradient descent can also be interpreted as a form of normalized gradient descent in the $\ell^\infty$ topology as in \cite{cortes2006finite}, where our framework allows considering a general $\ell^q$ norm.
\begin{figure}[t]
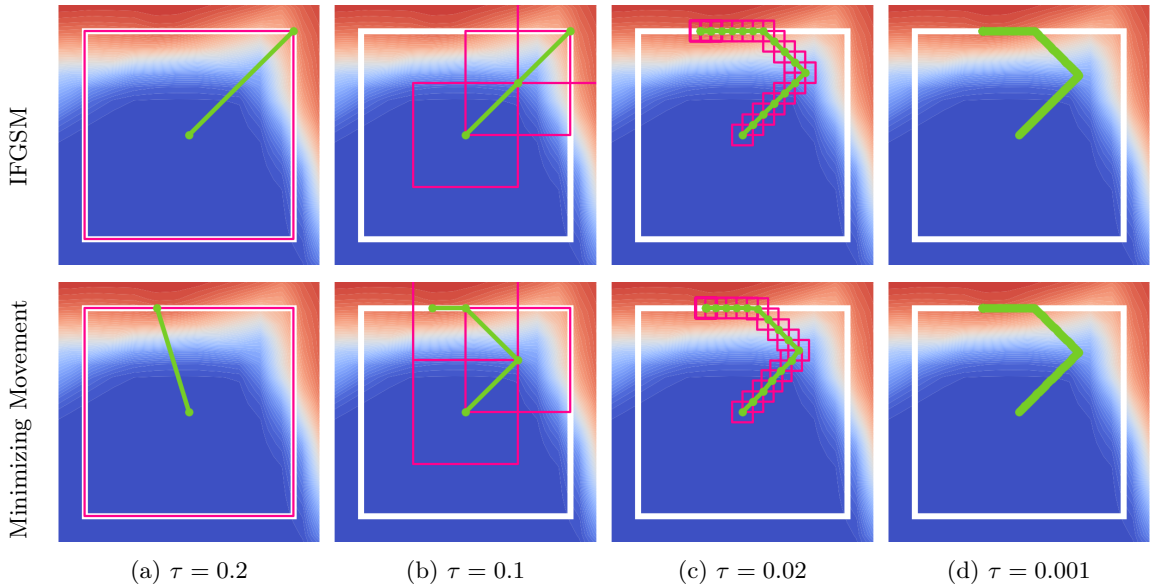

\rotatebox[origin=c]{90}{\makebox[2cm]{\small IFGSM \phantom{g}}}\quad%
\foreach \n in {0.2,0.1,0.02,0.001}{%
\begin{subfigure}[c]{.23\textwidth}%
\includegraphics[width=\textwidth, trim={2cm 0cm 2cm 0cm},clip]{results/Flow_fgsm\n.png}%
\end{subfigure}\hfill%
}%

\vspace{.2cm}

\rotatebox[origin=c]{90}{\makebox[3.3cm]{\small \phantom{----} Minimizing Movement}
}\quad%
\foreach \n in {0.2,0.1,0.02,0.001}{%
\begin{subfigure}[c]{.23\textwidth}%
\includegraphics[width=\textwidth, trim={2cm 0cm 2cm 0cm},clip]{results/Flow_MinMov\n.png}%
\caption{$\tau=\n$}%
\end{subfigure}\hfill%
}%
\caption{Behavior of \labelcref{eq: IFGSM} (top) and the minimizing movement scheme \labelcref{eq: MinMov} (bottom), for a binary classifier---parametrized as a neural network---on $\R^2$, a budget of $\budget=0.2$, and $\tau\in\{0.2, 0.1, 0.02, 0.001\}$. The white box indicates the maximal distance to the initial value, the pink boxes indicate the step size $\tau$ of the scheme. Details on this experiment can be found in \cref{sec:num}.}%
\label{fig:IFGSM}%
\end{figure}%
Apart from the adversarial setting, signed gradient descent, without the projection step, is an established optimization algorithm itself, see e.g., \cite{zhang2020sign,mohammadi2023sign} for other applications. The idea of using signed gradients can also be found in the RPROP algorithm \cite{Riedmiller93}. The convergence to minimizers of signed gradient descent and its variants was analyzed in \cite{li2023faster,balles2020geometry,MOULAY201929,chzhen2023signsvrg}. A slightly different kind of projected version, using linear constraints, was considered in \cite{CHEN2020109156}, where the authors also considered a continuous time version, however, the results therein and the considered flow are not directly connected to our work here.
We consider the limit $\tau\to 0$ of signed gradient descent and the projected variant \labelcref{eq: IFGSM}, \rev{for which} we derive a gradient flow characterization\rev{. This} is visualized in \cref{fig:IFGSM}.
In the Euclidean setting with a differentiable energy $\func:\R^\ddim\to\R$ and $p\in(1,\infty)$, a differentiable curve $u:[0,T]\to\R^\ddim$ is a $p$-curve of maximum slope, if it solves the $p$-gradient flow equation
\begin{align*}
\left(\abs{u'}(t)\right)^{p-2}\, u'(t) = -\nabla \func(u(t)).
\end{align*}
Here, we also refer to \cite{bungert2021nonlinear,bungert2020asymptotic} for a study of gradient-flow type equations in Hilbert spaces, for non-differentiable functionals. Following the approach in \cite{DeGiorgi80,ambrosio05,degiovanni1985evolution,marino1989curves}, the above equation is equivalent to 
\begin{align*}
\frac{d}{dt} (\func\circ u)\leq 
-\frac{1}{p}\abs{u'}^p - \frac{1}{q}\abs{\nabla\func(u)}^q,
\end{align*}
where $1/p + 1/q = 1$. The strength of this approach, is that all derivatives in the above inequality, have meaningful generalizations to the metric space setting, which we repeat in the next section. Motivated by signed gradient descent, in this paper we draw the connection to the case $p=\infty$. In the Euclidean setting, with a differentiable functional $\func$, the energy dissipation inequality we derive for $p=\infty$, reads
\begin{align*}
\abs{u'}&\leq 1,\\
\frac{d}{dt}(\func\circ u)&\leq -\abs{\nabla\func(u)}.
\end{align*}
Intuitively, a $\infty$-curve of maximal slope minimizes 
the energy $\func$ as fast as possible under the restriction that its velocity $\abs{u'}$ is bounded by $1$. Like in \cite{ambrosio05}, our results consider general metric spaces, Banach spaces and Wasserstein spaces, which are further detailed in the following sections. 
Typically, curves of maximum slope can be approximated via a minimizing movement scheme, which in our case translates to 
\begin{align*}
\xm_\tau^{k+1}\in \argmin_{\x\in\WBan} \{ \func(\x)\st \norm{\x - \xm_\tau^k}\leq \tau \},
\end{align*}
where $\xm^0_\tau = \x^0$ is a given initial value.
A main insight, explored in \cref{ch: AdAt}, is that under certain assumption \labelcref{eq: FGSM,eq: IFGSM} fulfill this scheme, if we replace the energy by a semi-implicit version.

A further aspect, is the characterization of adversarial attacks in the distributional setting, where the sum is replaced by an integral over the data distribution $\mu$. Interchanging the integral and the supremum (see \cref{cor:maxsup}), yields the characterization of adversarial training \labelcref{eq:advtrain} as a distributionally robust optimization (DRO) problem,
\begin{align}\label{eq:DRO}\tag{DRO}
\inf_{\hyp\in\Hyp}\sup_{\tilde \mu: D(\mu, \tilde \mu)\leq \budget} \int\loss(\hyp(\x), y) d\tilde\mu(\x,y),
\end{align}
where $D$ denotes a distance on the space of distributions. This formulation of adversarial training was the \rev{subject} of many studies in recent years, see, e.g., \cite{bungert2023geometry,zheng2019distributionally,bungert2024,bungert2024gamma,bungert2023begins}.
Typically, the distance $D$ is chosen an optimal transport distance, 
\begin{align*}
D(\mu,\Tilde{\mu})\coloneqq \inf_{\gamma\in \Gamma(\mu,\Tilde{\mu} )} \int_\gamma c((x,y),(\Tilde{x},\Tilde{y}))^2 d\gamma,
\end{align*}
with \rev{$\Gamma(\mu, \muatt)$} denoting the set of all couplings and the cost
\begin{align}\label{eq:advdistance}
c((x,y),(\xatt,\Tilde{y}))\coloneqq \begin{cases}
\rev{\|x-\xatt\|} &\text{if } y=\Tilde{y},\\
+\infty &\text{if } y\neq \Tilde{y}.
\end{cases}
\end{align}
The goal here is then to derive a characterization of curves $\mu:[0,T]\to \mathcal{W}_p$, where $\mathcal{W}_p$ denotes the $p$-Wasserstein space. In this regard, we mention the related work \cite{zheng2019distributionally}, where the authors proposed to solve the inner optimization problem
\begin{align*}
    \sup_{\tilde \mu: D(\mu, \tilde \mu)\leq \budget} \int\loss(\hyp(\x), y) d\tilde\mu(\x,y)
\end{align*}
by disintegrating the data distribution $d\mu(x,y)=d\mu_y(x)d\nu(y)$ (see \cref{app:help}), and calculating for $\nu$-a.e. $y\in\Oup$ the corresponding $2$-gradient flow in $\mathcal{W}_2$ with initial condition $\mu_y^0$. As shown in \cite{ambrosio05} solving this gradient flow is equivalent to solving the partial differential equation 
\begin{align}\label{eq: 2GradFlow}
\begin{aligned}
    \partial_t (\mu_y)_t=\nabla \cdot((\mu_y)_t \nabla_x \loss(\hyp(\x), y)) &\quad \text{on }  (0,T)\\
    (\mu_y)_0=\mu_y^0,&
\end{aligned}
\end{align}
which is to be understood in the distributional sense. The authors in \cite{zheng2019distributionally} then approximate a maximizer by  $d\tilde{\mu}(x,y)\approx d(\mu_y)_T (x)d\nu(y) $, where $T$ has to be chosen small enough such that the approximation is still within the $\budget$ ball around $\mu$.

In the following, we first provide the necessary notions for gradient flows in metric spaces and then proceed to discuss the main contributions and the outline of this paper.
\subsection{Setup}

We give a brief recap on classical notation and preliminaries on evolution in
metric spaces. More details can be found in \cite{ambrosio05, Mielke2012}. In the following, we denote by $(\CMS,d)$ a complete metric space, while $\WBan$ denotes a Banach space. We consider a proper functional $\func: \CMS\rightarrow (-\infty,+\infty]$, i.e., the effective domain 
$\dom(\func) \coloneqq \{ \x\in \CMS\st \func(\x) < \infty\}$ is assumed to be nonempty. Throughout this paper we denote by
\begin{align*}
B_\tau(\x) := \{\xatt\in\CMS\st d(\x, \xatt) < \tau\}, \qquad
\cl{B_\tau}(\x) := \{\xatt\in\CMS: d(\x, \xatt) \leq \tau\}
\end{align*}
the ball and its closed variant, induced by the given metric $d$, where we employ the abbreviation $B_\tau(0) = B_\tau$. In the finite dimensional case, we write $B^p_\tau$ to denote the ball induced by the $\ell^p$ norm on $\R^d$. Note, that there is a notation conflict with $d$ denoting both the distance and the dimension of the finite-dimensional space $\R^d$. However, the concrete meaning is always clear from the context.

\paragraph{Metric derivative}
We consider curves $u:[0,T]\rightarrow \CMS$ with $T>0$ for which we want to have a notion of velocity. For this sake, we need a generalization of the absolute value of the derivatives, which is provided by the \emph{metric derivative} as introduced by \textcite{ambrosio90}. Here, one usually considers $p$-absolutely continuous curves \cite{ambrosio05}, i.e., for $p\in[1,\infty]$ there exists $m\in L^p(0,T)$ such that%
\begin{align}\label{eq:pcont}
d(u(t), u(s)) \leq \int_s^t m(r) dr
\end{align}
for all $0\leq s<t \leq T$. The set of all $p$-absolutely continuous curves is denoted by $AC^p(0,T; \CMS)$. We are especially interested in the case $p=\infty$, where the condition in \cref{eq:pcont} is equivalent to the Lipschitzness of the curve, i.e., the existence of a constant $L\geq 0$ such that
\begin{align*}
d(u(t), u(s)) \leq L\ (t-s) 
\end{align*}
for all $0\leq s<t \leq T$. For the special case $p=\infty$, we have the following result as a special case of \cite[Theorem 1.1.2]{ambrosio05}.
\begin{lemma}[Metric Derivative]\label{lem:metricderiv}
Let $u:[0,T]\rightarrow \CMS$ be a Lipschitz curve with Lipschitz constant $L$, then the limit
\begin{align*}
|u'|(t)\coloneqq\lim_{s\rightarrow t}\frac{d(u(s),u(t))}{|s-t|}
\end{align*}
exists for a.e. $t\in [0,T]$ and is referred to as the metric derivative. Moreover, the function $t\mapsto |u'|(t)$ belongs to $L^\infty(0,T)$ with $\||u'|\|_{L^\infty(0,T)}\leq L$, and
\begin{align*}
    d(u(s),u(t))\leq \int_s^t |u'|(r) dr \quad \ \text{for all } 0\leq s\leq t \leq T.
\end{align*}
\end{lemma}
\begin{remark}\label{rem:minmetric}
The metric derivative $|u'|$ is actually minimal in the sense that for every $m$ satisfying \labelcref{eq:pcont},
$$ |u'|(t)\leq m(t) \quad \text{for a.e. }t\in (0,T).$$
\end{remark}
\begin{remark}\label{rm:RadDer}
If $\CMS=\WBan$ is a Banach space and satisfies the Radon--Nikodým property (c.f. \cite[p. 106]{ryan2002introduction}), e.g., if
it is reflexive, then $u\in AC^p(0,T;\WBan)$ if and only if
\begin{itemize}
    \item $u$ is differentiable a.e. on $(0,T)$
    \item $u'(t)\in L^p(0,T;\WBan)$
    \item $u(t)-u(s)=\int_s^t u'(t) dr$ for $0\leq s\leq t\leq T$.
\end{itemize}
\end{remark}

\paragraph{Upper gradients}
We consider \emph{upper gradients} as a generalization of the absolute value of the gradient in the metric setting. Namely, 
we employ the following definitions from 
\cite[Definition 1.2.1]{ambrosio05} and \cite[Definition 1.2.2]{ambrosio05}.
\begin{definition}\label{def: strGrad}
A function $g:\CMS\to [0, +\infty]$ is called a \emph{strong upper gradient} for $\func$, if for every absolutely continuous curve $u:[0,T]\to\CMS$ the function $g\circ u$ is Borel and 
\begin{align}\label{eq: StUpGrad}
    |\func(u(t)) -\func(u(s))|\leq \int_s^t g(u(r)) |u'|(r) dr \quad \forall\ 0\leq s\leq t\leq T
\end{align}
If $(g\circ u)\, |u'| \in L^1(0,T)$ then $\func\circ u$ is absolutely continuous and 
\begin{align}\label{eq:strgradprop}
|(\func\circ u)'(t)|\leq g(u(t))|u'|(t) \quad \text{for a.e. } t\in(0,T).   
\end{align}
\end{definition}

\begin{definition}
A function $g:\CMS\to [0, +\infty]$ is called a \emph{weak upper gradient} for $\func$, if for every absolutely continuous curve $u:[0,T]\to \CMS$ that fulfills
\begin{enumerate}[label=(\roman*)]
\item $(g\circ u)\ \abs{u^\prime}\in L^1(0,T)$,
\item $\func \circ u$ is a.e. in $(0,T)$ equal to a function $\psi:(0,T)\to\R$ with bounded variation, 
\end{enumerate}
it follows that
\begin{align*}
\abs{\psi^\prime} \leq (g\circ u)\ \abs{u^\prime} \text{ a.e. in } (0,T).
\end{align*}
\end{definition}
\begin{remark}
We note that for a  function $\psi$ with bounded variation, i.e.,
\begin{align*}
\sup\left\{
\sum_{i=0}^{N-1} \abs{\psi(t_{i+1}) - \psi(t_i)}\st 0=t_0<\ldots <t_N=T
\right\}
< \infty,
\end{align*}
we have that the derivative $\psi^\prime$ exists a.e. in the interval $(0,T)$, see \cite[Theorem 9.6, Chapter IV]{saks37}.
\end{remark}
\begin{remark}
Admissible curves $u$ in the above definition fulfill that {$u^{-1}(\CMS\setminus \dom(\func))$} is a null set, because of (ii). Therefore, the behavior of $g$ outside of $\dom(\func)$ is negligible.
\end{remark}

\paragraph{Metric slope} We now consider the \emph{metric slope}, as defined in \cite{DeGiorgi80}, as a special realization of a weak upper gradient. Intuitively, the slope gives the value of the maximal descent at a point $u$ at a infinitesimal small distance.
\begin{definition}
For a proper functional $\func:\CMS\rightarrow (-\infty,+\infty]$, the \emph{local slope} of $\func$ at $\x\in\dom(\func)$ is defined as
\begin{align*}
|\partial \func|(\x) \coloneqq \limsup_{\xx\rightarrow \x} \frac{(\func(\x)-\func(\xx))^+}{d(\x,\xx)}.
\end{align*}
\end{definition}
The definition of the slope does in fact yield an upper gradient, which is provided by the following statement from \cite{ambrosio05}.
\begin{theorem}[{\cite[Theorem 1.2.5]{ambrosio05}}]
Let $\func$ be a proper functional, then the function $\abs{\partial \func}$ is a weak upper gradient.
\end{theorem}

\paragraph{Curves of maximal slope}
Curves of maximal slope were introduced in \textcite{DeGiorgi80} and are a possible generalization of a gradient evolution in metric spaces. They are usually formulated for the case $p\in(1,\infty)$ as follows, see, e.g., \cite{ambrosio05}.
\begin{definition}[$p$-Curves of maximal slope]\label{def:maxslope}
For $p\in(1,\infty)$ we say that an absolutely continuous curve $u:[0,T]\to\CMS$ is a $p$-curve of maximal slope, for the functional $\func$ with respect to an upper gradient $g$, if $\func\circ u$ is a.e. equal to a non-increasing map $\psi$ and
\begin{align}\label{eq:curvmax}
\psi^\prime(t)\leq 
-\frac{1}{p} \abs{u^\prime}^p(t)
-\frac{1}{q} g^q(u(t))
\end{align}
for almost every $t\in(0,T)$ and $1=\frac{1}{p}+\frac{1}{q}$.
\end{definition}
For $p\in (1,\infty)$, the existence of such curves is provided, see for example \cite{ambrosio05}.

\subsection{Main results}\label{sec:mainresults}
\rev{
Here, we summarize the main contributions of this paper. The most important one is the development and application of a gradient flow framework that allows for a theoretical study of adversarial attacks. Concerning the theory of metric gradient flows, we introduce notions tailored to this application and also provide adapted proofs, as detailed below. Here, it should be noted however that many of our results in metric and Banach spaces can be obtained from the theory of doubly nonlinear equations \cite{rossi2008metric,mielke2013nonsmooth}. Therefore, the main contribution from this side, is to draw the connection between the previously mentioned works and the field of adversarial attacks. On top of that, the proofs that are adapted to our scenario, allow for additional insights into the concrete application we consider. 
Beyond single adversarial examples we also treat distributional adversaries which we link to curves of maximal slope in the $\infty$ -Wasserstein space. For potential energies we derive a (to our knowledge novel) characterization of 
curves of maximal slope via the superposition principle, which highlights the connection between single adversarial attacks and the distributional adversary. We give more details on the results below.
}

In \cref{ch: existence}, we \rev{extend} the notion of $p$-curves of maximal slope to the case $p=\infty$, for Lipschitz curves $u$. As hinted in the introduction, in the limit $p\to\infty$ of \cref{def:maxslope}, we replace \labelcref{eq:curvmax} by the following  condition,
\begin{align*}
\abs{u'}(t)&\leq 1,\\
\quad\psi'(t)&\leq -g(u(t)).
\end{align*}
Such curves are then called $\infty$-curves of maximal slope. \rev{%
We want to highlight that similar considerations already appeared in the early works of De Giorgi, see for example, \cite[Definition 1.2]{DeGiorgi80} and \cite[Ex. 1.3]{degiorgi1993new}. For our concrete setup here, we dedicate 
\cref{ch: existence} to an existence proof of such curves. We note that this can also be obtained as a corollary of a more general existence result in \cite[Thm. 3.5]{rossi2008metric}. Therein, the authors prove existence of curves of maximal slope fulfilling
\begin{align*}
\nimap'(t) \leq -f^*(g(u(t))) - f(\abs{u'}(t))
\end{align*}
for a convex and lower semicontinuous function $f:[0,\infty)\to[0,\infty\rev{]}$. When choosing $f=\chara_{[0,1]}$, we recover our notion of $\infty$-curves of maximal slope. Although, the existence proof in \cref{ch: existence} employs similar concepts, we choose to include it here. On the one hand, the treatment of this specific case allows for certain arguments, that are not directly possible in the general case. On the other hand, this already introduces the main steps for the \rev{convergence proof in \cref{sec: Banach}}, which can not directly be deduced from \cite{rossi2008metric}. The existence result in \cref{thm: exist}, is summarized below.
\begin{nothm}[\rev{Existence:}]
Under the assumptions specified in \cref{ch: existence}, for every $\func:\mathcal{S}\to (-\infty,+\infty]$ and
for every $\x^0 \in \dom(\func)$ there exists a \rev{1-}Lipschitz curve $u:[0,T]\to\mathcal{S}$ with $u(0)=\x^0$ which is an $\infty$-curve of maximum slope for $\func$ with respect to its \rev{strong} upper gradient $|\partial\func |$.
\end{nothm}
}
In \cref{sec: Banach}, we consider the specific case of $\infty$-curves of maximal slope in a Banach space $\Ban$, and an \rev{energy $\funccp$ that is a $C^1$ perturbation of a convex function. Note, that here and in the following, when the functional takes the role of a $C^1$-perturbation as in \cref{sec: Banach}, we use the symbol $\funccp$ instead of $\func$.} %
We derive an equivalent characterization of $\infty$-curves of maximal slope via a differential inclusion. \rev{We note that this differential inclusion can be obtained from \cite[Prop. 8.2]{rossi2008metric}, with the same choice of $f$ as for the existence result above. The statement in our setting can be found in \cref{thm: GradFlowForm} and is summarized below.}

\begin{nothm}[\rev{Differential inclusion:}]
Let $\funccp: \Ban \rightarrow (-\infty,+\infty]$ satisfy \labelcref{eq: slope}  and $u:[0,1] \rightarrow \Ban$ be an a.e. differentiable Lipschitz curve. Let further $\funccp\circ u$ be a.e. equal to a non-increasing function $\nimap$, then the following are equivalent:
\begin{enumerate}[label=(\roman*)]
\item     
$|u'|(t)\leq 1$ and $\nimap'(t)\leq -|\partial \funccp|(u(t))$ 
for a.e. $t\in[0,1]$,
\item 
$u'(t)\ \rev{\in}\ \partial \|\cdot\|_*(-\xi) \quad \forall \xi \in \partial^\circ \funccp(u(t)) \not= \emptyset,$ for a.e. $t\in[0,1],$
\end{enumerate}
where $\partial^\circ \funccp(u(t))$ denotes \rev{the elements of minimal norm} of $\partial \funccp(u(t))$.
\end{nothm}
%
%
\rev{
For an energy $\funccp=\funccpd+\funccpc$ consisting of a differentiable part $\funccpd$ and a convex part $\funccpc$ we consider the linearization in the differentiable part around a point $z$,
\begin{align*}
\funccps(\x; \xx)\coloneqq \funccpd(\xx)+\langle D \funccpd(\xx),\x-\xx\rangle+\funccpc(\x).
\end{align*}
This then leads us to the semi-implicit minimizing movement scheme in \cref{def: SeLiSch}
\begin{align*}
\xs_{\tau}^{k+1}\in\argmin_{\x\in\cl{B_\tau}(\xs_\tau^k)} \funccps(\x;\xs_\tau^k),
\end{align*}
which we also employ to approximate curves of maximal slope.
}
%
%
In the case of $p=2$\rev{,} we refer to \cite{fleissner2019gamma,stefanelli2022new} for other works that also consider approximate minimizing movement schemes. This semi-implicit scheme is useful, since in the finite dimensional adversarial setting it allows us to choose $-\loss(\hyp(\cdot),y)$ as the differentiable part, and additionally to incorporate the budget constraint via the indicator function $\chara_{\cl{B_\budget}(\x)}$. \rev{We denote by $\sxs_\tau$ the step function associated to the iterates $\xs_{\tau}^{k}$, see \cref{def: SeLiSch}}. We can show that up to a subsequence, this scheme also \rev{converges} to \rev{an} $\infty$-curve of maximum slope, in the topology $\sigma$ as specified in \cref{asm: 01}. The result can be found in \cref{thm:semilincvg}, which we hint at below.
\begin{nothm}[\rev{Convergence to curves of maximal slope:}] Under the assumptions specified in \cref{sec: Banach}, there exists a $\infty$-curve of maximal slope $u$ and a
subsequence of $\tau_n=T/n$ such that
\begin{align*}
\sxs_{\tau_n}(t) \stackrel{\sigma}{\rightharpoonup}u(t) \text{ as } n\rightarrow \infty \quad \forall t\in [0,T].
\end{align*}
\end{nothm}
In order to better understand the connection between the differential inclusion and \labelcref{eq: IFGSM}, we want to highlight, that $\infty$-curves of maximum slope yield a general concept, which is not directly tied to signed gradient descent and the choice of the projection. The intuition behind $\infty$-curves is rather connected to employing normalized gradient descent (NGD) \cite{cortes2006finite}. Choosing $(\R^d, \norm{\cdot}_p)$ as the underlying Banach space, in \cref{ch: AdAt} we see that for $1/p + 1/q=1$ the following iteration fulfills the semi-implicit minimizing movement scheme,
\begin{align*}
\x^{k+1} = 
\x^k + \tau\
\sign(\nabla_\x \funccp(\x^k))\cdot \left(\frac{\abs{\nabla_\x \funccp(\x^k)}}{\norm{\nabla_\x \funccp(\x^k)}_q}\right)^{q-1},
\end{align*}
where the absolute value and multiplication is understood entrywise. Choosing $p=2$ or $p=\infty$ recovers the notion of NGD as in \cite{cortes2006finite}. \rev{Normalized gradient methods have gained significant attention outside the adversarial context. For example, in the context of saddle point evasion \cite{hazan2015beyond,murray2019revisiting,levy2016power}, subgradient corruption \cite{turan2021robustness}, machine learning \cite{cutkosky2020momentum} and even variational quantum algorithms \cite{suzuki2021normalized}.} In the setting of adversarial attacks, \rev{normalization} means, that we want to ensure that the iterates exploit the maximum allowed budget (locally on $\cl{B_\tau}(\x^k)$ ball) in each step. This was similarly observed in \cite{dong2018boosting}.
As long as the iterates stay within the given budget $\budget$, \rev{one can directly show} that \labelcref{eq: IFGSM} is an explicit solution to the semi-implicit scheme and therefore converge to $\infty$-curves of maximum slope. In the more interesting case, where the projection has an effect, we need to ensure that minimizing on $\cl{B_\tau^p}(\x)$ and then projecting to $\cl{B_\budget^p}(\x^0)$ is equivalent to directly minimizing on $\cl{B_\budget^p}(\x^0)\cap \cl{B_\tau^p}(\x)$. We show this property for the case $p=\infty$ \rev{in \cref{lem:projtrick}}. Employing the convergence result for the semi-implicit minimizing movement scheme, then yields the convergence up to subsequences of \labelcref{eq: IFGSM}, employing the $\ell^\infty$ norm. \rev{Denoting by $\xifg_\tau^k$ the $k$-th iterate obtained in \labelcref{eq: IFGSM} with stepsize $\tau$, \cref{cor:budgettime} then presents the following result.}
\begin{nothm}[\rev{Convergence of IFGSM:}]
Under the assumptions specified in \cref{ch: AdAt}, for $T>0$, there exists a $\infty$-curve of maximal slope $u:[0,T]\to \R^d$, with respect to $\funccp$, and a subsequence of $\tau_n:=T/n$ such that
\begin{align*}
\norm{\xifg_{ \tau_{n_i}}^{\lceil t/\tau_{n_i} \rceil} - u(t)}\xrightarrow{i\to\infty} 0\qquad\text{ for all } t\in[0,T].
\end{align*}
\end{nothm}

In \cref{sec: Wass} we consider potential energies 
$$ \func: W_\infty(\Ban)\ni\mu\mapsto \int  \pot(x) d \mu(x),$$
where in our context, the potential $\pot: \Ban\rightarrow (-\infty,+\infty]$ has the form $\pot(\x) = -\loss(\hyp(\x),y)$. The basis for our main result in this section is given by \textcite[Thm. 3.1]{lisini2014absolutely}, which is repeated as \cref{thm: curveDe} in this paper. Namely, we characterize absolutely continuous curves $\mu\in\AC^p(0,T; \rev{\mathcal{W}}_p)$ by a measure $\eta$ on the space of curves $u:[0,T]\to \WBan$, which is concentrated on $\AC^p(0,T;\WBan$). Using this representation, in \cref{thm: WInfCurve} we show that being a $\infty$-curve of maximum slope in the Wasserstein space is equivalent to the differential inclusion on the underlying Banach space, for $\eta$-a.e. curve.

\begin{nothm}[\rev{Characterization of curves in Wasserstein space:}]
Under the assumptions specified in \cref{thm: WInfCurve}, for a curve $\mu \in \AC^\infty(0,T\rev{;} \rev{\mathcal{W}}_\infty)$ with $\eta$ from \cref{thm: curveDe}, the following statements are equivalent:
\begin{enumerate}[label=(\roman*)]
\item%
The curve $\mu$ is $\infty$-curve of maximal slope w.r.t. to the weak upper gradient $\abs{\partial\func}$.
\item For $\eta$-a.e. curve $u\in C(0,T;\WBan)$ it holds, that $E\circ u$ is for a.e. $t\in(0,T)$ equal to a non-increasing map $\psi_u$ and
$$ u'(t)\ \rev{\in}\  \partial \|\cdot\|_*(-\xi) \quad \forall \xi \in \partial^\circ \pot(u(t)) \not= \emptyset, \quad \text{for a.e. } t\in(0,T).$$
\end{enumerate} 
\end{nothm}
When applying this result to adversarial training, we slightly deviate from the Wasserstein setting, by choosing the extended distance in \labelcref{eq:advdistance} and the associated \rev{transport} distance, in order to prohibit mass transport into the label direction. \rev{%
Here, we want to refer to other works considering distributional adversarial attacks, e.g.,
\cite{staib2017distributionally, zheng2019distributionally, pydi2020adversarial, bungert2023begins,bungert2023geometry,pydi2021many,sinha2017certifying,mehrabi2021fundamental}. We }%
can adjust the arguments in \cref{sec: Wass} to derive an analogous result for the energy $\ffunc(\mu)\coloneqq
\int -\loss(\hyp(x),y) d\mu(x,y)$, which we state in \cref{thm: WInfCurveAdv}. Here, we only enforce the budget constraint, by setting the end time of the flow to $T=\budget$. 
\subsection{Outline}

The paper is organized as follows: In \cref{ch: existence}, we start by introducing $\infty$-curves of maximal slope, as the limit case of $p$-curves of maximal slope. \cref{ch: ExPr} then provides an existence result for those curves in a general metric setting. The underlying assumptions for its proof are stated in \cref{ch: assum}.

In \cref{sec: Banach}, we consider $\infty$-curves of maximal slope when the underlying metric space is a Banach space. \cref{ch: C1Per} introduces $C^1$-perturbations of convex functions as a convenient class of functionals, that covers 
most of the energies we consider in this paper.
In \cref{ch: DifIn}, we derive equivalent characterizations of $\infty$-curves of maximal slope via a doubly nonlinear differential inclusion. This section is concluded by investigating \rev{first-order} approximation techniques of those differential inclusions in \cref{ch: LinTi}.

\cref{sec: Wass} is devoted to $\infty$-curves of maximal slope, when the underlying space is the $\infty$-Wasserstein space. For potential energies we give an equivalent characterization of $\infty$-curves  of maximal slope via a probability measures $\eta$ on the space $C(0,T;\WBan)$ which is concentrated on $\infty$-curves of maximal slope on the underlying Banach space $\WBan$. From $\eta$, we can then derive a corresponding continuity equation for those curves of maximal slope.

In \cref{ch: AdAt}, we discuss the application of differential inclusions derived in \cref{sec: Banach} to generate adversarial examples. We show that the popular \textit{fast gradient sign method} (FGSM) and its iterative variant (IFGSM) are simple \rev{first-order} approximations of $\infty$-curves of maximal slope. In \cref{ch: DisAd}, we rewrite adversarial training as a distributional robust optimization problem and discuss the usage of \rev{$\infty$-curves} in the corresponding probability space, to generate distributional adversaries.
\section{Infinity Flows in metric spaces}\label{ch: existence}

In this section, we generalize the notion of $p$-curves of maximal slope to the case $p=\infty$. We consider the convex function $f(x)=\frac{1}{p}|x|^p$, which allows us to express the energy dissipation inequality \labelcref{eq:curvmax} in \cref{def:maxslope} as follows, 
\begin{align}\label{eq:curvmaxgeneral}
\psi'(t)\leq-f(|u'|(t))-f^*(g(u(t))),
\end{align}
where $f^*(x^*)=\frac{1}{q} |x^*|^q$ denotes the convex conjugate of $f$. Considering the above inequality for arbitrary convex functions $f$ leads to the general framework as introduced in \textcite{rossi2008metric}. For our setting, we consider the \rev{indicator} function, which is obtained as the following pointwise limit,
\begin{align*}
\frac{1}{p} \abs{x}^p \xrightarrow[]{p\to\infty}\chi_{[-1,1]}(x) = 
\begin{cases}
0 \text{ if } |x|\leq 1,\\
+\infty \text{ else},
\end{cases} 
\end{align*}
where $\chi_{[-1,1]}$ is a convex function with conjugate $\chi_{[-1,1]}^*(x^*) = x^*$. Using $f=\chi_{[-1,1]}$ in \cref{eq:curvmaxgeneral} forces the curves of maximal slope to obey $\abs{u^\prime}\leq 1$ almost everywhere and the energy dissipation inequality becomes 
\begin{align*}
    \psi'(t)\leq -\rev{(g\circ u)}(t),
\end{align*}
which motivates the following definition.
\begin{definition}[$\infty$-Curve of maximal slope]\label{def: CMS}
We say an absolutely continuous curve $u:[0,T] \rightarrow \CMS$ is an $\infty$-curve of maximal slope for the functional $\func$ with respect to an upper gradient $g$, if $\func \circ u$ is a.e. equal to a non-increasing map $\psi$ and
\begin{align}\label{eq:InfFlow}\tag{InfFlow}
\begin{split}
|u'|(t)&\leq 1,\\
\quad\psi'(t)&\leq -\rev{(g\circ u)(t)},
\end{split}
\end{align}
holds for a.e. $t\in(0,T)$.
\end{definition}

\begin{remark}
We note that the condition $\abs{u^\prime}\leq 1$ a.e., implies that $u$ is a Lipschitz curve with Lipschitz constant $1$, see \cref{lem:metricderiv}.
\end{remark}

\rev{
\begin{remark}(Dissipation equality)\label{rm:DisE}
If $g$ is a strong upper gradient of $\func$ and $\psi: [0,T] \rightarrow \R$ is finite  then by \cref{def: strGrad} and \labelcref{eq:InfFlow} 
$$ | \func(u(t))-\func(u(s))| \leq \int_s^t g(u(r))|u'|(r)dr\leq \int_s^t g(u(r)) dr\leq \int_s^t -\psi'(r) dr\leq \psi(s)-\psi(t)<+\infty,  $$
where in the last inequality we use that non-increasing functions are differentiable a.e. and an upper bound on the second fundamental theorem of calculus holds \cite[Proposition 1.6.37]{tao2011introduction}. This in particular implies that $\func\circ u$ is absolutely continuous and $\psi(t)=(\func\circ u)(t)$ for all $t\in(0,T)$  (see \cref{lm: help1}).
Furthermore, \cref{rm:RadDer} implies
 \begin{align*}
     \func(u(t))-\func(u(s))=\int_s^t (\func\circ u)'(r) dr \quad \text{for }0\leq s\leq t \leq T
 \end{align*}
 and we can estimate 
 $$ \func(u(t))-\func(u(s))=\int_s^t (\func\circ u)'(r) dr\leq \int_s^t -g(u(r)) dr \quad \text{for }0\leq s\leq t \leq T $$
 and on the other hand using \labelcref{eq:strgradprop}, we obtain
 \begin{align*}
 \func(u(t))-\func(u(s))&=\int_s^t (\func \circ u)'(r) dr\geq
  \int_s^t -|(\func \circ u)'(r)| dr  \\
  &\geq \int_s^t -g(u(r)) |u'|(r) dr \geq \int_s^t -g(u(r)) dr
 \end{align*}
 for $0\leq s\leq t \leq T$.
 Therefore, the energy dissipation equality 
\begin{align}\label{eq:EnergyDisEqual}\tag{EnDisEq}
     \func(u(t))-\func(u(s))= \int_s^t -g(u(r)) dr 
 \end{align}
 holds for every $0\leq s\leq t \leq T$ .
\end{remark}
}

\begin{example}
As an easy example, let us look at the quadratic energy $\func: \x\mapsto \frac{1}{2} \x^2$ on the space $(\CMS, d) = (\mathbb{R},|\cdot - \cdot|)$. Its metric slope and thus weak upper gradient is given by $\abs{\partial\func}(\x) = \abs{\frac{d}{dx}\func}(\x)=\abs{\x}$. We choose $\x^0=1$ as the starting point, then the corresponding $\infty$-curve of maximal slope is
$$u(t)=\begin{cases}
    1-t &\text{if } 0\leq t\leq 1,\\
    0 &\text{if }t>1
\end {cases}.$$
We directly observe that $\abs{u'}\leq 1$ and
\begin{align*}
\func(u(t)) = 
\begin{cases}
\frac{1}{2} (1-t)^2 &\text{if } 0\leq t\leq 1,\\
0 &\text{if }t>1,
\end{cases}
\end{align*}
is a non-increasing map with 
\begin{align*}
\frac{d}{dt} \func(u(t)) = 
\begin{cases}
t-1 &\text{if } 0\leq t < 1,\\
0 &\text{if }t>1,
\end{cases}
= -\abs{u(t)} = - \abs{\partial \func}(u(t)),
\end{align*}
and therefore the conditions \labelcref{eq:InfFlow} are fulfilled. Here we can already observe a typical behavior of $\infty$-curves of maximal slope. They have a constant velocity of $1$ until they hit a local minimum where they stop abruptly.
\end{example}

The rest of this section is devoted to an existence proof for $\infty$-curves of maximal slope. 

\subsection{Assumptions for existence}\label{ch: assum}
Here, we state the assumptions needed for the proof of existence. Approximations of curves of maximal slope are constructed via a minimizing movement scheme. To guarantee convergence of those approximations, a form of relative compactness is essential. This is \rev{guaranteed} by \cref{asm: 02}. Furthermore, relative compactness with respect to the topology induced by the metric $d(\cdot,\cdot)$ may not be given. However, relative compactness with respect to a weaker topology $\sigma$ is sufficient, as long as it is compatible with the topology induced by the metric $d(\cdot,\cdot)$, \cref{asm: 01}. These assumptions were also employed in \cite{ambrosio05}.
\begin{mdframed}
\begin{assume}[Weak topology] \label{asm: 01}
In addition to the metric topology, $(\CMS,d)$ is assumed to be endowed with a Hausdorff topology $\sigma$. We assume that $\sigma$ is compatible with the metric $d$, in the sense that $\sigma$ is weaker than the topology induced by $d$ and $d$ is sequentially $\sigma$-lower semicontinuous, i.e.,
\begin{align*}
(\x^n,\xx^n)\stackrel{\sigma}{\rightharpoonup}(\x,\xx) \Longrightarrow \liminf_{n\rightarrow \infty} d(\x^n,\xx^n)\geq d(\x,\xx).
\end{align*}
\end{assume}
\end{mdframed}

\begin{mdframed}
\begin{assume}[Relative compactness] \label{asm: 02}
Every $d$-bounded set contained in sublevels of $\func$ is relatively $\sigma$-sequentially compact, i.e.,
\begin{gather*}
\text{if}\quad \{\x^n\}_{n\in\N} \subset \CMS \quad\text{with}\quad \sup_{n\in\N} \func(\x^n) < +\infty,\quad \sup_{n,m} d(\x^n,\x^m)< +\infty,\\
\text{then }\seq{\x} \text{ admits a } \sigma \text{-convergent subsequence.}
\end{gather*} 
\end{assume}
\end{mdframed}
\cref{asm: lc,asm: 03b} ensure the \rev{lower semicontinuity of the} energy functional and the lower semicontinuity of its metric slope. These regularity assumptions are required for the energy dissipation inequality  during the limiting process in the proof of \cref{thm: exist}.

\begin{mdframed}
\addtocounter{asscount}{1}
\setcounter{assume}{0}
\rev{
\begin{assume}[Lower semicontinuity]\label{asm: lc} We assume sequential $\sigma$-lower semicontinuity of $\func$ for bounded sequences, namely,
\begin{align} \label{eq: adAss}
\left.
\begin{aligned}
\sup_{n,m\in \mathbb{N}}\left\{  d(\x^n,\x^m) \right\} <+\infty,\\
\x^n \stackrel{\sigma}{\rightharpoonup}\x
\end{aligned}\right\}
\Longrightarrow \func(\x)\leq  \liminf_{n\to\infty} \func(\x^n).
\end{align}
\end{assume}
}
\end{mdframed}

\begin{mdframed}
\begin{assume}[Lower semicontinuity of Slope] \label{asm: 03b}
In addition, we ask that $|\partial \func|$ \rev{is a strong upper gradient and it} is sequentially 
$\sigma$-lower semicontinuous on  $d$-bounded sublevels of $\func$.
\end{assume}  
\end{mdframed}

\begin{remark}
    The proof of existence is possible with a wide variety of regularity assumptions on the energy $\func$, which can be tailored to a variety of different situations. For example, if the sequentially \rev{$\sigma$-lower} semicontinuous envelope of $|\partial \func|$ 
    \begin{align*}
    |\partial^- \func|\coloneqq\left\{ \liminf_{n\rightarrow \infty}|\partial \func|(\x^\rev{n})\st\x^\rev{n}\stackrel{\sigma}{\rightharpoonup}\x,\sup_n{d(\x^\rev{n},\x),\func(\x^\rev{n})}<+\infty\right\}
\end{align*}
    is a strong upper gradient, one can drop \cref{asm: 03b} and instead prove existence of curves of maximal slope with respect to $|\partial^- \func|$. 
    Further, if $|\partial \func |$ (or $|\partial^- \func|$ respectively) is only a weak upper gradient (compare \cite[Theorem 2.3.3]{ambrosio05}) then \cref{asm: lc}  has to be replaced by continuity of the energy.
\end{remark}

\subsection{Minimizing movement for $p=\infty$}\label{sec:minmov}
The minimizing movement scheme is an implicit time discretization of curves of maximal slope. The existence of curves of maximal slope is proven by sending the discrete time \rev{step} $\tau$ of the minimizing movement scheme to $0$. 
For the time interval $[0,T]$ and some $n\in\N$, we use an equidistant time discretization $t^k=k \cdot \tau$ for $k\in\{0,...,n\}$ with $\tau=T/n$. Starting with $\x_\tau^0=\x^0$ the classical minimizing movement scheme to approximate $p$-curves of maximum slope reads
\begin{align*}
\x^{k+1}_\tau \in \argmin_{\xatt\in\CMS} 
\left\{
\frac{1}{p \tau^{p-1}}d^p(\xatt,\x_\tau^k)+\func(\xatt)
\right\}.
\end{align*}
Taking formally the limit $p\rightarrow \infty$ under the constraint $d(\xatt,\x_\tau^k)\leq \tau$, we arrive at the corresponding minimizing movement scheme for $p=\infty$, which we define in the following.
\begin{definition}[Minimizing Movement Scheme for $p=\infty$]\label{def: MinMov}
For $\tau=T/n$ and $\x^0_\tau = \x^0$ we consider the iteration defined for $k\in\N_0$ as
\begin{align}\label{eq: MinMov}\tag{MinMove}
\x_\tau^{k+1}\in \rev{\argmin_{\xatt\in\CMS}} \{ \func(\xatt)\st d(\xatt,\x_\tau^k)\leq \tau \}.
\end{align}
We define the step function $\sx_\tau$ by
\begin{align*}
\sx_\tau(0)=\x^0, \quad \sx_\tau(t)=\x^k_\tau \text{ if } t\in (t_\tau^{k-1},t^k_\tau], k\geq1.
\end{align*}
Furthermore we define 
\begin{align*}
|\x'_\tau|(t)\coloneqq \frac{d(\x^k_\tau,\x_\tau^{k-1})}{t_\tau^k-t_\tau^{k-1}} \text{ if } t \in(t_\tau^{k-1},t_\tau^k),
\end{align*}
as the metric derivative of the corresponding piecewise affine linear interpolation. 
\end{definition}
\cref{asm: 02,asm: lc} guarantee the existence of minimizers in \labelcref{eq: MinMov} via the direct method in the calculus of variations \cite{dacorogna2007direct}, which ensures that the minimizing movement scheme can be defined.
Now for all $\x\in\CMS$ we set
\begin{align} \label{eq: MinTau}
\func_\tau(\x)\coloneqq\min_{\xatt\in \rev{\cl{B_\tau}}(\x)} \func(\xatt).
\end{align}
\begin{remark}\label{rem:infconv}
The function defined in \cref{eq: MinTau} is similarly employed in \cite{armstrong2010easy,bungert2023uniform,bungert2024ratio} and the proof strategy as displayed in \cref{fig: DerSlop} resembles the max-ball arguments as in the previously mentioned works. 
The expression in \labelcref{eq: MinTau} can also be seen as the infimal convolution \cite{hausdorff1919halbstetige,fenchel1953convex} of $\func$ and $\chara_{\cl{B_\tau(0)}}$, i.e., $\func_\tau = \chara_{\cl{B_\tau}}\, \square\, \func$ and can also be considered as the limit $p\to\infty$ of the Moreau envelope \cite{moreau1965proximite},
\begin{align*}
\inf_{\xatt} \left\{\func(\xatt) + \frac{1}{p}\norm{\x - \xatt}^p\right\}
\end{align*}
which is typically defined for $p=2$.
\end{remark}
\rev{%
\begin{remark}
More recently, similar schemes to the one defined in \labelcref{eq: MinMov} have been introduced in an optimization context in \cite{gruntkowska2025ball}. Here, the operation on the right-hand side of \labelcref{eq: MinMov} was labelled the \enquote{ball-proximal} or \enquote{brox} operator.
\end{remark}%
}
\rev{%
The next lemma gives an equivalent 
characterization of the metric slope and provides its relation to the minimizing movement scheme. In fact, it is a special case of \cite[Lem. 3.1.5, Rem. 3.1.7]{ambrosio05}. For completeness, we provide an adapted proof in \cref{app:help}.
}

\begin{lemma}\label{lm: EqiSlope}
For all  $\x \in \dom(\func)$ we have that
\begin{align}\label{eq:AltSlop}
|\partial \func|(\x)=\limsup_{\tau \rightarrow 0^+} \frac{\func(\x)-\func_\tau(\x)}{\tau}.
\end{align}
\end{lemma}

Further, we are interested in the behavior of the mapping $\tau\mapsto\func_\tau(\x)$ when varying $\tau$. By definition, it is monotone decreasing in $\tau$ and thus differentiable a.e. This allows us to derive an integral inequality that gives an upper bound to $\func_\tau(\x)$ as $\tau$ increases.

\begin{lemma}[Differentiability of $\func_\tau(\x)$]\label{lem: Dif}
For $\x\in \dom(\func)$ the derivative $\frac{d}{d\tau} \func_\tau(\x)$ exists for a.e. $\tau \in (0,+\infty)$ and
\begin{align} \label{eq: IntIn}
    \func_{\tau_1}(\x)+\int_{\tau_1}^{\tau_2} \frac{d}{d\tilde{\tau}} \func_{\tilde{\tau}}(\x) d\tilde{\tau}\geq \func_{\tau_2}(\x) \quad \text{for }  0 \leq \tau_1\leq \tau_2< +\infty.
\end{align}
Furthermore 
\begin{align}\label{eq: DerSlop}
    \frac{d}{d\tau} \func_\tau(\x)\leq -|\partial \func|(\x_{\mathrm{min},\tau}) \quad \text{ for a.e. }  \tau \in (0,+\infty),
\end{align}
where 
\begin{align} \label{eq: MinMiz}
\x_{\mathrm{min},\tau}\in \argmin_{\xatt}\{\func(\xatt)\st d(\x,\xatt)\leq \tau\}.
\end{align}
\end{lemma}
\begin{proof}
Let $\x\in\dom(\func)$, for any $\tau^*<\infty$ we know that the mapping $\tau\mapsto \func_\tau(\x)$ 
is monotone decreasing on $[0,\tau^*]$ and 
thus its variation can be bounded,
$$ 
\func_0(\x) - \func_{\tau^*}(\x)=
\func(\x)-\func(\x_{\mathrm{min},\tau^*})<\infty.$$
Employing \cite[Theorem 9.6, Chapter IV]{saks37}, this yields that the derivative exists for almost every $t\in(0,\tau^*)$ and that \cref{eq: IntIn} holds. 
 %
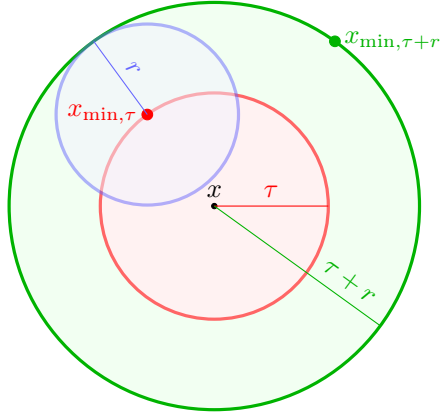
\begin{figure}%
\centering%
\begin{tikzpicture}
\filldraw[color=green!70!black, fill=green!5, very thick](0,0) circle (2.7);
\filldraw[color=red!60, fill=red!5, very thick](0,0) circle (1.5);
\filldraw[color=blue!60, fill=blue!5, very thick,opacity=0.5] (0,0) ++(0.7*180:1.5) circle (1.2);
\filldraw[](0,0) circle (1 pt) node [above] {$\x$};
\filldraw[red] (0,0) ++(0.7*180:1.5) circle[radius=2pt] node [left,] {$\x_{\text{min},\tau}$};
\draw[blue!60]++(.7*180:1.5) -- ++(.7*180:1.2) node [midway,above, sloped] {$r$};
\filldraw[green!70!black] (0,0) ++(0.3*180:2.7) circle[radius=2pt] node [right,] {$\x_{\text{min},\tau+r}$};
\draw[red](0,0) -- (1.5,0) node [midway,above] {$\tau$};
\draw[green!70!black](0,0) -- ++(1.8*180:2.7) node [near end,above, sloped] {$\tau + r$};
\end{tikzpicture}
\caption{Visualization of the ball inclusion used for the proof of \cref{eq: DerSlop}.}\label{fig: DerSlop}
\end{figure}%
To show \cref{eq: DerSlop}, we observe that  $$B_r(\x_{\mathrm{min},\tau})\subset B_{\tau+r}(\x) \text{ and thus }\func_{\tau+r}(\x)\leq \func_r(\x_{\mathrm{min},\tau}),$$ see \cref{fig: DerSlop}, which yields
\begin{align*}
-\left(\frac{\func(\x_{\mathrm{min},\tau})-\func_{\tau+r}(\x)}{r}\right)
\leq 
-\left(\frac{\func(\x_{\mathrm{min},\tau})-\func_{r}(\x_{\mathrm{min},\tau})}{r}\right).
\end{align*}
It follows that
\begin{align*}
    \frac{d}{d\tau} \func_\tau(\x)&=\lim_{r\rightarrow 0} \frac{\func_{\tau+r}(\x)-\func_{\tau}(\x)}{r}\\
    &= -\lim_{r\rightarrow 0} \frac{\func(\x_{\mathrm{min},\tau})-\func_{\tau+r}(\x)}{r}\\
    &\leq -\limsup_{r\rightarrow 0}  \frac{\func(\x_{\mathrm{min},\tau})-\func_{r}(\x_{\mathrm{min},\tau})}{r}=-|\partial \func|(\x_{\mathrm{min},\tau}),
\end{align*}
where we used the characterization of the slope from \cref{lm: EqiSlope}. 
\end{proof}
\subsection{Proof of existence}\label{ch: ExPr}
Together with the previous lemmas, we are now able to prove the existence of $\infty$-curves of maximal slope. Besides the piecewise constant interpolation $\sx$, we use a variational interpolation. This interpolation, combined with estimate in \labelcref{eq: D}, later yields the differential inequality \labelcref{eq:InfFlow}.

\begin{definition}[De Giorgi variational interpolation]\label{def: DeGiorgi}
We denote by $\vx_\tau:[0,T]\rightarrow \mathcal{S}$ any interpolation of the discrete values satisfying
\begin{align*}
\vx_\tau(t)
\in \argmin_{\xatt} \{ \func(\x)\st d(\xatt,\x_\tau^{k-1})\leq t- t^{k-1}_\tau \}
\end{align*}
if $t\in (t_\tau^{k-1},t^k_\tau]$ and $ k\geq1$.
Furthermore, we define
\begin{align}\label{eq: DDef}
D_\tau(t)\coloneqq \frac{d}{dt} \func_{(t-t_\tau^{k-1})}(\x^{k-1}_\tau)\quad \text{ if } t\in(t_\tau^{k-1},t_\tau^k].
\end{align}
\end{definition}
Employing \cref{lem: Dif}, the above definition directly yields
\begin{align}\label{eq: D}
    \func(\vx_\tau(s))+\int_s^t D_\tau(r) dr\geq \func(\vx_\tau(t)) \quad\forall\  0\leq s\leq t\leq T,
\end{align}
which is used in the following existence proof, \cref{thm: exist}. We employ the arguments of \cite[Ch. 3]{ambrosio05} and transfer them to our setting, where a crucial statement is the refined version of Ascoli--Arzelà in \cite[Prop. 3.3.1]{ambrosio05}, which is repeated for convenience, in the appendix, see \cref{prop:AA}.
\rev{As detailed in \cref{ch: intro} this can also be obtained via the results in \cite{rossi2008metric}. Nevertheless, we include a proof here, since this introduces the main arguments for the proof of \cref{thm:semilincvg}.
}
\begin{theorem}[Existence of $\infty$-curves of maximal slope]\label{thm: exist}
Under the \cref{asm: 01,asm: 02,asm: lc,asm: 03b}  
for every $\x^0 \in \dom(\func)$ there exists a \rev{1-}Lipschitz curve $u:[0,T]\to\mathcal{S}$ with $u(0)=\x^0$ which is an $\infty$-curve of maximum slope for $\func$ with respect to its \rev{strong} upper gradient $|\partial\func|$ \rev{and $u$ satisfies the energy dissipation equality 
\begin{align}\label{eq: EnergyDissip}
\func(u(0))=\func(u(t))+\int_0^t |\partial\func|(u(r))dr \quad \text{for all } t\in[0,T].    
\end{align}
}
\end{theorem}

\begin{proof}
We consider the set of all possible iterates in the minimizing movement scheme $K=\{\x^i_{\tau_n}\st 0\leq i\leq n, n\in \N\}\subset\CMS$. Recalling \cref{def: MinMov}, for every $n\in\N$ and $i,j \in\{0,\ldots,n\}$, we have the estimate
\begin{align*}
d(\x_{\tau_n}^i,\x_{\tau_n}^j)\leq\sum_{k=i}^{j-1} d(\x_{\tau_n}^k,\x_{\tau_n}^{k+1})
\leq (j-i)\ {\tau_n} \leq 
T,
\end{align*}
and therefore for every $n,m\in\N$ and $0\leq i\leq n,\ 0\leq j\leq m$, we have
\begin{align*}
d(\x_{\tau_n}^i,\x_{\tau_m}^j)\leq d(\x_{\tau_n}^i,\x^0) + d(\x^0,\x_{\tau_m}^j)\leq 2T.
\end{align*}
Furthermore, since $\x^0\in\dom(\func)$ we also know that 
$\func(\x_{\tau_n}^i) \leq \func(\x^0) < \infty$ and thus $K$ is a $d$-bounded set, contained in sublevels of $\func$. Using relative compactness, i.e., \cref{asm: 02}, this ensures that $\cl{K}$ is a $\sigma$-sequentially compact set and therefore fulfills \ref{def:AA1} of \cref{prop:AA}. In order to apply the latter, it remains to choose a function $\omega$ that fulfills \ref{def:AA2}. For this we consider the sequence of curves $\abs{\x'_{\tau_n}}:[0,T]\to\R$, which is by definition bounded in $L^\infty(0,T)$, i.e.,
$$\norm{\abs{\x'_{\tau_n}}}_{L^\infty(0,T)}\leq 1,
\qquad \text{for every }n\in\N.$$ 
\rev{For fixed $0\leq s \leq t \leq T$ let us define 
\begin{align}
s(n)\coloneqq \min_{k\in\{0,\ldots, n\}} \{ k\cdot \tau_n \st s \leq k\cdot \tau_n   \},\qquad%
t(n)\coloneqq \min_{k\in\{0,\ldots, n\}} \{ k\cdot \tau_n \st t \leq k\cdot \tau_n \}.
\end{align}
Using the triangle inequality and the fact that the distance between two consecutive iterates is bounded by $\tau$, we obtain
\begin{align}\label{eq:DiEst}
\limsup_{n\rightarrow +\infty} d(\Bar{x}_{\tau_n}(s),\Bar{x}_{\tau_n}(t))
&\leq \limsup_{n\rightarrow +\infty} \sum_{i=1}^{\frac{t(n)-s(n)}{\tau_n}} d(\Bar{x}_{\tau_n}(s(n)+(i-1)\tau),\Bar{x}_{\tau_n}(s(n)+i\tau_n))\\
&\leq \lim_{n\rightarrow +\infty} (t(n)-s(n))=|t-s |\eqqcolon \omega(s,t).
\end{align}%
}%
Therefore, \ref{def:AA2} in \cref{prop:AA} is fulfilled, allowing us to apply \cite[Proposition 3.3.1]{ambrosio05} to extract another subsequence such that 
\begin{align*}
\sx_{\tau_n}(t) \stackrel{\sigma}{\rightharpoonup}u(t) \text{ as } n\rightarrow \infty \quad \forall t\in [0,T], \quad u \text{ is } d\text{-continuous in } [0,T].
\end{align*}
This in particular ensures $u(0)=\x^0$ and \labelcref{eq:DiEst} together with \cref{asm: 01} yields 1-Lipschitzness of $u$, since for $s\leq t$ we have
\begin{align*}
d(u(s), u(t)) \leq \liminf_{n\to\infty} 
d(\sx_{\tau_n}(s),\sx_{\tau_n}(t))\leq t-s.
\end{align*}
By construction it holds that $d(\sx_{\tau_n},\vx_{\tau_n})\leq \tau$, which also yields
\begin{align*}
\vx_{\tau_n}(t) \stackrel{\sigma}{\rightharpoonup}u(t) \text{ as } n\rightarrow \infty \quad \forall t\in [0,T].
\end{align*}
\rev{
Observing that $\Tilde{x}_{\tau_n}(0)=x^0=u(0)$ independent of $n$, we take the limes inferior for \cref{eq: D} and use \cref{asm: lc} and Fatou's lemma to obtain for all $t\in[0,T]$
\begin{align*}
\func(u(0))&\geq \liminf_{n\rightarrow\infty} 
\left\{
\func(\Tilde{x}_{\tau_n}(t))-\int_0^t D_{\tau_n}(r)dr
\right\}
\geq \func(u(t))+\int_0^t \liminf_{n\rightarrow\infty} -D_{\tau_n}(r)dr\\
&\geq \func(u(t)) +\int_0^t |\partial \func(u(r))|dr.
\end{align*}
The last inequality follows by the estimate
\begin{align*}
    |\partial \func|(u(t))\leq \liminf_{n\rightarrow \infty} |\partial \func|(\Tilde{x}_{\tau_n}(t))\leq \liminf_{n\rightarrow \infty} -D_{\tau_n}(t)\quad\text{ for a.e.}\quad t\in(0,T),
\end{align*}
which is a consequence of \labelcref{eq: DDef} and \labelcref{eq: DerSlop} and the $\sigma$-lower semicontinuity of the slope.
On the other hand we know that $|\partial \func|$ is a strong upper gradient and $|u'|(r)\leq 1$ for a.e. $r \in[0,T]$, such that
\begin{align*}
    \func(u(0))\leq \func(u(t))+\int_0^t |\partial \func |(u(r)) | u'|(r)dr\leq \func(u(t)) +\int_0^t |\partial \func |(u(r))dr.
\end{align*}
In particular the equality 
\begin{align*}
    \func(u(0))= \func(u(t))+\int_0^t |\partial \func |(u(r)) | u'|(r)dr= \func(u(t)) +\int_0^t |\partial \func |(u(r))dr
\end{align*}
must hold.
It follows that $t\mapsto \func(u(t))$ is locally absolutely continuous and
\begin{align*}
    \frac{d}{dt}\func(u(t))=-|\partial \func|(u(t)) |u'|(t)=-|\partial \func|(u(t)) \text{ for a.e. }t\in(0,T).
\end{align*}
}

\end{proof}

\section{Banach space setting}\label{sec: Banach}
In this section, we consider the Banach space setting, i.e., we assume that $\mathcal{S}=\Ban$, where $\Ban$ is a 
Banach space with norm $\|\cdot\|$ and $(\Ban^*,\|\cdot\|_*)$ denoting its dual. \rev{%
In this section, we assume the functional to be a $C^1$-perturbation (see \cref{ch: C1Per}) and use the symbol $\funccp$ to distinguish it from general functionals $\func$ in the previous section.} We want to give an equivalent characterization of curves of maximum slope in terms of differential inclusions. Following \cite[Ch. 1]{ambrosio05}, for a functional $\funccp:\Ban\to(-\infty,\infty]$ we employ the 
Fréchet subdifferential $\partial \funccp \subset \Ban^*$, where for $\x\in\operatorname{dom}(\funccp)$ we define
\begin{align}\label{eq: SubDif}
\xi\in \partial \funccp(\x) \Leftrightarrow 
\liminf_{\xx\to \x} \frac{\funccp(\xx) - \funccp(\x) - \langle \xi, \xx-\x\rangle}{\|\xx-\x\|}\geq 0
\end{align}
with $\operatorname{dom}(\partial \funccp) = \{\x\in \Ban\st \partial\funccp (\x)\neq \emptyset\}$. 
Assuming that $\partial \funccp(\x)$ is weakly$^*$ closed for every $\x\in\operatorname{\dom}(\partial\funccp)$---which holds true in particular, if $\Ban$ is reflexive or $\funccp$ is a so called $C^1$-perturbation of a convex function (see \cref{pro:C1})---we furthermore define
\begin{align*}
\partial^\circ \funccp (\x) := 
\argmin_{\xi \in \partial \funccp(\x)} \|\xi\|_*\subset \partial\funccp(\x).
\end{align*}
\rev{Note, that $\partial^\circ\funccp(x)$ is still potentially multivalued, however all elements have the same dual norm. This justifies using the notation $\norm{\partial^\circ \funccp(x)}_*=\min\{\|\xi\|_*\st\xi\in\partial\funccp(\x)\}$ in the following.}
\subsection{On $C^1$-perturbations of convex functions}\label{ch: C1Per}
Functions that can be split into a convex function $\funccpc$ and a differentiable part $\funccpd$, i.e., $\funccp=\funccpc+\funccpd$, are called $C^1$-perturbations of convex functions. 
This particular class of functions exhibits a variety of useful properties. We collect the ones, that are relevant for our setting in the following proposition, which is a combination of Corollary 1.4.5 and Lemma 2.3.6 in \cite{ambrosio05}. 

\begin{proposition}[$C^1$-perturbations of convex functions] \label{pro:C1}
If $\funccp:\Ban\rightarrow (-\infty,+\infty]$ admits a decomposition $\funccp=\funccpc+\funccpd$, into a proper, lower semicontinuous convex function $\funccpc$ and a $C^1$-function $\funccpd$, then 
\begin{propenum}
\item\label{pro:C1:i} $\partial \funccp=\partial\funccpc+D\funccpd$,
\item\label{pro:C1:ii} 

$$\left.
\begin{aligned}
\xi^{\rev{n}}\in \partial \funccp(\x^{\rev{n}}),\\
\x^{\rev{n}}\to \x\in \operatorname{dom}(\partial\funccp),\\
\xi^{\rev{n}} \rightharpoonup^* \xi
\end{aligned}\right\}
\Rightarrow
\left\{
\begin{aligned}
\xi \in \partial \funccp(\x),\\
\funccp(\x^{\rev{n}})\to \funccp(\x),
\end{aligned}\right.
$$
\item\label{pro:C1:iii} \rev{$|\partial \funccp|(\x)=\|\partial^\circ \funccp(\x)\|_* \quad \forall \x\in \Ban,
$}
\item \label{pro:C1:iv} $|\partial \funccp|$ is $\|\cdot\|$-lower semicontinuous,
\item \label{pro:C1:v}$|\partial \funccp|$ is a strong upper gradient of $\funccp$.
\end{propenum}
\end{proposition}
%

\rev{Considering Banach spaces that fulfill \cref{asm: 01,asm: 02} with their strong topology and energies that are $C^1$ perturbations, the existence of $\infty$-curves of maximum slope follows directly by \cref{thm: exist}.} 
An important example of such a Banach space $\Ban$, is the Euclidean space, since our motivating application, namely adversarial attacks, usually employs a finite dimensional image space. We formulate this result in the following corollary.
\begin{corollary}[Existence for $C^1$-perturbations in finite dimensions]\label{cor: FinC1}
Let $\Ban=(\mathbb{R}^d,\|\cdot\|)$ and $\funccp:\mathbb{R}^d\rightarrow (-\infty,+\infty]$ 
admits a decomposition $\funccp=\funccpc+\funccpd$, into a proper, lower semicontinuous convex function $\funccpc$ and a $C^1$-function $\funccpd$. For every $\x^0\in \dom(\funccp)$, there exists at least one curve of maximal slope in the sense of \cref{def: CMS} with $u(0)=\x^0$. \rev{ Further this curve  satisfies the energy dissipation equality \cref{eq: EnergyDissip}. }
\end{corollary}
\begin{proof}
We choose $\sigma$ to be the norm topology, such that \cref{asm: 01,asm: 02} are fulfilled \rev{and $\funccp$ fulfills \cref{asm: lc}.} By \cref{pro:C1} $\abs{\partial \funccp}$ is lower semicontinuous \rev{and a strong upper gradient}. Therefore also \cref{asm: 03b} is fulfilled and the application of \cref{thm: exist} yields the desired result.
\end{proof}
%

In the infinite dimensional case, existence is harder to prove. Usually $\sigma$ is chosen as the weak or weak* topology, such that when $\Ban$ is reflexive or a dual space, the Banach--Alaoglu theorem yields compactness and that \cref{asm: 01,asm: 02} are fulfilled. 
A desirable property for the energy functional is the so-called $\sigma$-weak* closure property
\begin{align*}
\left.
\begin{aligned}
\xi^{\rev{n}}\in \partial \funccp(\x^{\rev{n}}),\\
\x^{\rev{n}} \stackrel{\sigma}{\to} \x\in \operatorname{dom}(\partial\funccp),\\
\xi^{\rev{n}} \rightharpoonup^* \xi
\end{aligned}\right\}
\Rightarrow
\left\{
\begin{aligned}
\xi \in \partial \funccp(\x),\\
\funccp(\x^{\rev{n}})\to \funccp(\x)
\end{aligned}\right.
\end{align*}
of its subdifferential, c.f. \cref{pro:C1:ii}. The $\sigma$-lower semicontinuity of the slope \cref{asm: 03b} and \labelcref{eq: slope} are almost immediate consequences of the closure property, as was shown in \cite[Lemma  2.3.6,Theorem  2.3.8]{ambrosio05}.

\begin{example}
As an application of \cref{cor: FinC1}, we consider the finite dimensional adversarial setting introduced in \cref{ch: intro}, i.e., we choose $\Ban=\R^\ddim$. Let 
\begin{align*}
\funccp(\x)\coloneqq \underbrace{-\loss(\hyp(\x),y)}_{=\funccpd} + \underbrace{\chara_{\cl{B_\budget}(\x^0)}}_{\funccpc} \nc,
\end{align*}
then by the chain rule $\funccpd\in C^1(\Ban)$, if $h\in C^1(\Ban;\Oup)$ and $\ell\in C^1(\Oup\times\Oup)$. We consider a neural network $\hyp = \phi^L\circ\ldots\circ\phi^1$ with the $l$th layer being given as
\begin{align*}
\phi^l:\R^{d^l}\to\R^{d^{l+1}},  \phi^l(\xx):= \actfun(W \xx + b),
\end{align*}
for a weight matrix $W\in\R^{d^{l+1}, d^l}$, bias $b\in\R^{d^{l+1}}$ and activation function $\actfun:\R\to\R$, which is applied entry-wise. Therefore, the network $\hyp$ is $C^1$ if its activation function is in $C^1(\R)$. Typical examples that fulfill this assumption are the Sigmoid function and smooth approximations to ReLU \cite{fukushima1980neocognitron}, such as GeLU \cite{hendrycks2023gaussian}, see also \cref{sec:num} for more details on such activation functions. Furthermore, many popular loss functions are in $C^1(\Oup\times\Oup)$,  like the \textit{Mean Squared Error} (MSE) or \textit{Cross-Entropy}  paired with a \textit{Softmax} layer \cite{boltzmann1868studien,good1952rational,cybenko1998mathematics}. On the other hand, the \textit{Root Mean Squared Error} (RMSE) is not differentiable whenever a component is $0$. 
\end{example}

\cref{lm: EqiSlope} provides an alternative characterization of the metric slope, employing a $\limsup$ formulation. The next two lemmas show that $C^1$-perturbations are regular enough, such that the limes superior can be replaced by a standard limit. This is used in \cref{lm:PotSlope}. The first lemma establishes the fact, that for convex functionals, there is a minimizing sequence for the value of $\funccp_\tau(\x)$ that lies on the boundary $\partial B_\tau(\x)$.
\rev{
\begin{remark}
Similar to \cite[Sec. 3.1]{ambrosio05}, we remark some properties of $\funccp_\tau(x)=\inf_{\xatt\in \cl{B_\tau}(x)}\funccp(\xatt)$ in the case, when $\funccp$ is convex. Since $\funccp_\tau(x)$ is defined via the infimal convolution, see \cref{rem:infconv}, we can directly infer convexity in the $x$ argument, if $\funccp$ was already convex. Furthermore, we also have convexity in $\tau$, which can be seen as follows. Let $\tau_1,\tau_2\geq 0$ be arbitrary, where we also allow them to attain $0$. For any $z_1\in \cl{B_{\tau_1}}(x), z_2\in \cl{B_{\tau_2}}(x)$ we have that $\lambda z_1 + (1-\lambda) z_2 \in \cl{B_{\tilde{\tau}}}(x)$ with $\tilde{\tau}=\lambda \tau_1 + (1-\lambda)\tau_2$ for any $\lambda\in [0,1]$. The definition of $\funccp_\tau$ and the convexity of $\funccp$ yields
\begin{align*}
\funccp_{\tilde{\tau}}(x)\leq
\funccp(\lambda z_1 + (1-\lambda) z_2) \leq \lambda\funccp(z_1) + (1-\lambda)\funccp(z_2)
\end{align*}
and since $z_1\in \cl{B_{\tau_1}}(x), z_2\in \cl{B_{\tau_2}}(x)$ were arbitrary, we obtain
\begin{align*}
\funccp_{\tilde{\tau}}(x)\leq \lambda \funccp_{\tau_1}(x) + (1-\lambda)\funccp_{\tau_2}(x).  
\end{align*}
If $x\in\dom(\funccp))$, we have that $\dom(\tau\mapsto \funccp_\tau(x))=[0,\infty)$ and thus $\tau\to\funccp_\tau$ is continous on $(0,\infty)$. If $\funccp$ is lower semicontinous, we also obtain continoutiy at $0$.
\end{remark}
}
\rev{\begin{lemma}\label{lm: conBound}
If $\funccp$ is a proper, convex, lower semicontinuous function then for all $x\in \dom(\funccp)$ with $|\partial\funccp(\x)|\not = 0$ there is an $\epsilon>0$ such that for all $0<\tau<\epsilon$ there exists a sequence $\seq{\x}$ with
\begin{align}\label{eq:approx}
\funccp(\x^{\rev{n}})\rightarrow \funccp_{\tau}(\x)\quad and \quad  \|\x-\x^{\rev{n}}\|=\tau \ \forall n\in\N.
\end{align}
If in addition the Banach space $\Ban$ is reflexive then there exists $\x_\tau\in\Ban$ with
\begin{align*}
\funccp(\x_\tau)=\funccp_{\tau}(\x)\quad and \quad  \|\x-\x_\tau\|=\tau.
\end{align*}
\end{lemma}%
}
\rev{%
\begin{proof}
Let $\x\in \dom(\funccp)$ with $|\partial\funccp(\x)|\not = 0$, then the mapping $\tau\mapsto\funccp_\tau(\x)$ is non-increasing and not constant. Therefore, we can find an $\epsilon> 0$ such that $\funccp_\tau(\x)>\funccp_{\epsilon}(\x)$ for all $0<\tau<\epsilon$. Let $\seq{\xatt}$ be a sequence such that
\begin{align*}
\lim_{n\to\infty} \funccp(\xatt_n) = \funccp_\tau(\x).
\end{align*}
Since $\funccp_\tau(\x)>\funccp_{\epsilon}(\x)$ we can find an element $\hat{\x}$ that fulfills 
\begin{align*}
\funccp(\xatt_n)>\funccp(\hat{\x})\qquad\text{for every } n\in\N\qquad\text{and}\qquad
\tau < \norm{\x- \hat{\x}} \leq {\epsilon}.
\end{align*}
Since $\hat{\x}\notin \cl{B_\tau}(\x)$ and $\xatt_n\in \cl{B_\tau}(\x)$, the line between each pair $(\hat{\x}, \xatt_n)$, 
\begin{align*}
c_n:t\in[0,1]\mapsto t\hat{\x}+(1-t)\xatt_n    
\end{align*}
has to intersect the sphere $\partial B_\tau(\x)$ at some point $t_n\in [0,1)$, where we define the intersection point as $\x_n=c_n(t_n)\in \partial B_\tau(\x)$. Due to convexity, we obtain 
\begin{align*}
\funccp_\tau(\x) \leq \funccp(\x_n) = \funccp(t_n\hat{\x}+(1-t_n)\xatt_n)\leq t_n \funccp(\hat{\x}) + (1-t_n)\funccp(\xatt_n) \leq  \funccp(\xatt_n).
\end{align*}
Note, that the last inequality would only be strict, if $t_n\neq 0$, however, since $\xatt_n$ might already be lying on the sphere, we only obtain the weak inequality. The sequence $\x_n$ now is the desired sequence in \labelcref{eq:approx}.\par

In the reflexive case, the weak compactness of the unit ball guarantees weak convergence of a subsequence of $\seq{\x}$ to some $\x_\tau\in \cl{B_\tau}(\x)$. Lower semicontinuity and convexity imply weak lower semicontinuity of $\funccp$ and thus 
\begin{align*}
\funccp_\tau(\x)\leq\funccp(\x_\tau) \leq \liminf_{n\to\infty} \funccp(\x_n) = \funccp_\tau(\x).
\end{align*}
As above, we can choose an element $\hat{x}$ with $\norm{\hat{x} - x}>\tau$ with $\funccp(x_\tau) > \funccp(\hat{x})$. Applying the same argument as above, there is some $t\in[0,1)$ such that $t\hat{x} + (1-t) x_\tau$ intersects $\partial B_\tau(x)$. As above, if $t\neq 0$, convexity yields
\begin{align}
\funccp_\tau(\x)\leq \funccp(t\hat{x} + (1-t) x_\tau) < 
\funccp(x_\tau),
\end{align}
which contradicts the fact that $\funccp_\tau(\x)=\funccp(x_\tau)$ and thus $x_\tau$ must have already been on the boundary.
\end{proof}
}

Using the previous lemma, we can now show that for $C^1$-perturbations of convex functions, we can replace the $\limsup$ in \cref{lm: EqiSlope} by a normal limit.
\begin{lemma}\label{lm:Lim}
Let $\funccp:\Ban\rightarrow (-\infty,+\infty]$ admit a decomposition $\funccp=\funccpc+\funccpd$, into a proper, lower semicontinuous convex function $\funccpc$ and a $C^1$-function $\funccpd$, then for all $x\in \dom(\funccp)$ we have
\begin{align}\label{eq: LimSlop1}
|\partial\funccp|(\x)=\lim_{\tau\rightarrow 0^+}\frac{\funccp(\x)-\funccp_\tau(\x)}{\tau}.
\end{align}
\end{lemma}
\begin{proof}\phantom{-}\\
\textbf{Step 1: The convex case.}\\
We first assume that $\funccp$ is convex. We choose $\tau$ small enough such that by \cref{lm: conBound} we obtain a sequence $\{\x_n\}_n$ with $\|\x-\x_n\|=\tau$ and $\lim_{n\to\infty} \funccp(\x_n) = \funccp_\tau(\x)$. For each $n\in\N$, we consider the line 
\begin{align*}
c_n(t):= t\, \x_n + (1-t)\, \x
\end{align*}
evaluated at $\tilde{t}=\tilde{\tau}/\tau$ for some  $0<\tilde{\tau}<\tau$, which yields $\norm{\x-c_n(\tilde{t})}=\tilde{\tau}/\tau\norm{\x-\x_n} = \tilde{\tau}$. Due to convexity we obtain
\begin{align*}
\funccp(c_n(\tilde{t}))
\leq \tilde{t}\, \funccp(\x_n)+\left(1-\tilde{t}\right)\, \funccp(\x)\qquad \Rightarrow\qquad
\funccp(\x) - \funccp(c_n(\tilde{t})) \geq \tilde{t}\, \left(\funccp(\x) - \funccp(\x_n)\right).
\end{align*}
Using the fact that $\funccp_{\tilde{\tau}}(\x) \leq \funccp(c_n(\tilde{t}))$ and dividing by $\tilde{\tau}$ in the above inequality, yields
\begin{align*}
\frac{\funccp(\x)-\funccp_{\Tilde{\tau}}(\x)}{\Tilde{\tau}}\geq \frac{\funccp(\x)-\funccp(c_n(\tilde{t}))}{\Tilde{\tau}}
\geq
\frac{\funccp(\x)-\funccp(\x_n)}{\tau}.
\end{align*}
Considering the limit $n\to\infty$ we obtain the following inequality,
\begin{align*}
\frac{\funccp(\x)-\funccp_{\Tilde{\tau}}(\x)}{\Tilde{\tau}}\geq \limsup_{n\rightarrow \infty}\frac{\funccp(\x)-\funccp(c_n(\tilde{t}))}{\Tilde{\tau}}\geq \lim_{n\rightarrow \infty}\frac{\funccp(\x)-\funccp(\x_n)}{\tau}=\frac{\funccp(\x)-\funccp_\tau(\x)}{\tau}.
\end{align*}
This shows that $\tau\mapsto Q(\tau) := \frac{\funccp(\x)-\funccp_\tau(\x)}{\tau}$ is decreasing in $\tau$ and therefore, for a null sequence $\tau_n\to 0$, $Q(\tau_n)$ is an increasing sequence. The monotone convergence theorem together with \cref{lm: EqiSlope} shows \cref{eq: LimSlop1}.\\
\textbf{Step 2: Extension to $C^1$-perturbations.}\\ 
We now assume that $\funccp$ is a $C^1$-perturbation of a convex function. By the definition of differentiability, we can write
\def\convexfun{\rev{F}}
\begin{align*}
\funccp(\xx)=\underbrace{\funccpc(\xx)+\funccpd(\x)-\langle D\funccpd(\x),\x-\xx\rangle }_{\eqqcolon  \convexfun(\x)}+R(\x,\x-\xx),
\end{align*}
with $R(\x,\x-\xx)\in o(|\x-\xx|)$ for every $\xx\in\dom(\funccp)$. We observe, that $\convexfun$ is again a convex function. Let $\epsilon>0$, then we denote by $\rev{\x}^{\funccp}_{\tau,\epsilon},\x^{\convexfun}_{\tau,\epsilon}\in \rev{\cl{B_\tau}(\x)}$ the quasi-minimizers that fulfill
\begin{align*}
\funccp(\x^{\funccp}_{\tau,\epsilon})-\funccp_{\tau}(\x)\leq \tau \epsilon\quad \text{and}
\quad \convexfun(\x^{\convexfun}_{\tau,\epsilon})-\convexfun_{\tau}(\x)\leq \tau \epsilon\quad\text{respectively}.
\end{align*}
We use the estimate
\begin{gather*}
\funccp_{\tau}(\x)-\convexfun_{\tau}(\x)\leq
\funccp_{\tau}(\x) - \convexfun(\x^{\convexfun}_{\tau,\epsilon}) + \tau\epsilon
\\
=\underbrace{\funccp_{\tau}(\x)-\funccp(\x^{\convexfun}_{\tau,\epsilon})}_{\leq 0} + R(\x,\x-\x^{\convexfun}_{\tau,\epsilon})+ \tau\epsilon\leq |R(\x,\x-\x^{\convexfun}_{\tau,\epsilon})|+\tau\epsilon
\end{gather*}
and analogously
\begin{gather*}
\convexfun_{\tau}(\x)-\funccp_{\tau}(\x)\leq 
\convexfun_{\tau}(\x)- \funccp(\x^{\funccp}_{\tau,\epsilon}) + \tau\epsilon
\\
=\underbrace{\convexfun_{\tau}(\x)-\convexfun(\x^{\funccp}_{\tau,\epsilon})}_{\leq 0}-R(\x,\x-\x^{\funccp}_{\tau,\epsilon})+\tau \epsilon \leq |R(\x,\x-\x^{\funccp}_{\tau,\epsilon})|+\tau\epsilon,
\end{gather*}
to obtain
\begin{align}\label{eq:ecineq}
|\funccp_{\tau}(\x)-\convexfun_{\tau}(\x)|\leq \max\left\{ |R(\x,\x-\x^{\funccp}_{\tau,\epsilon})|, |R(\x,\x-\x^{\convexfun}_{\tau,\epsilon})|\right\}+\tau \epsilon.
\end{align}
Using that $\funccp(\x) = \convexfun(\x)$ and dividing by $\tau$ in \labelcref{eq:ecineq} yields the inequality
\begin{gather}\label{eq:dif}
\abs{\frac{\funccp(\x)-\funccp_{\tau}(\x)}{\tau}-\frac{\convexfun(\x)-\convexfun_{\tau}(\x)}{\tau} }\leq 
\underbrace{\frac{\max\{ |R(\x,\x-\x^{\funccp}_{\tau}(\epsilon))|, |R(\x,\x-\x^{\convexfun}_{\tau}(\epsilon))|\}}{\tau}}_{:=r(\tau)}+\epsilon.
\end{gather}
\rev{Since $\abs{\x-\x^{\funccp}_\tau(\epsilon)}= \abs{\x-\x^{\convexfun}_\tau(\epsilon)}\leq\tau$ it holds $\lim_{\tau\to 0} r(\tau) = 0$. Taking the $\limsup$ of \eqref{eq:dif} and sending $\epsilon$ to zero then yields,} 
\begin{align*}
\lim_{\tau\to 0^+} \abs{\frac{\funccp(\x)-\funccp_{\tau}(\x)}{\tau}-\frac{\convexfun(\x)-\convexfun_{\tau}(\x)}{\tau} } =  0.
\end{align*}
Therefore, the limit in \labelcref{eq: LimSlop1} exists,
\begin{align*}
\lim_{\tau\to 0^+} \frac{\funccp(\x)-\funccp_{\tau}(\x)}{\tau} &= 
\lim_{\tau\to 0^+} \frac{\funccp(\x)-\funccp_{\tau}(\x)}{\tau}-\frac{\convexfun(\x)-\convexfun_{\tau}(\x)}{\tau}  + \frac{\convexfun(\x)-\convexfun_{\tau}(\x)}{\tau}\\ &= 
\lim_{\tau\to 0^+} \frac{\convexfun(\x)-\convexfun_{\tau}(\x)}{\tau}=\abs{\partial \convexfun}(\x),
\end{align*}
where in the last step we used that $\convexfun$ is convex together with \textbf{Step 1}.
\end{proof}
%
\subsection{Differential inclusions}\label{ch: DifIn}
Similar to \cite[Proposition 1.4.1]{ambrosio05} for finite $p$, we now give a characterization of $\infty$-curves of maximal slope via differential inclusions, whenever the slope of the energy $\funccp$ can be written as
\begin{align}\label{eq: slope}
    |\partial \funccp|(\x)=\min\{\|\xi\|_*\st\xi\in\partial\funccp(\x)\}=\|\partial^\circ \funccp(\x)\|_* \quad \forall \x\in \Ban.
\end{align}
By \cref{pro:C1} this is, e.g., the case for $C^1$-perturbations. Let us start by defining a degenerate duality mapping $\mathcal{J}_\infty: \Ban \rightarrow 2^{\Ban^*}$,

\begin{align*}
\mathcal{J}_\infty(\x)\coloneqq \begin{cases}
 \{\xi \in \Ban^*\st \langle \xi,u\rangle=\| \xi\|_*\} &\text{if } \|\x\|=1,\\
 \{0\} &\text{if } \|\x\|<1,\\
 \emptyset &\text{if } \|\x\|>1,
\end{cases}
\end{align*}
as the limit case of the classical $p$-duality mapping \cite[Definition 2.27]{SchusterKaltenbacherHofmannKazimierski+2012}
\begin{align*}
    \mathcal{J}_p(\x)\coloneqq \left\{ \zeta \in \Ban^*\st \  \langle \zeta , u \rangle =\|\x\| \|\zeta\|_*,\ \|\zeta\|_*=\|x\|^{p-1}\right\}.
\end{align*}
This definition allows us to extend the classical Asplund theorem \cite[Theorem 2.28]{SchusterKaltenbacherHofmannKazimierski+2012} to the limit case. 
\begin{theorem}[Asplund Theorem for $p=\infty$]\label{thm: Asp}
The following identity holds true, 
 \begin{align*}
\mathcal{J}_\infty=\partial \chara_{\cl{\unitb}}.
 \end{align*}  
 \begin{proof}
 For $\x\in \Ban$ with $\|\x\| \neq 1$ the equality holds trivially. Therefore, we consider $\|\x\|=1$.\\
 \textbf{Step 1:}
$\mathcal{J}_\infty(\x)\subset\partial \chara_{\cl{\unitb}}(\x)$.\\
Let $\xi\in \mathcal{J}_\infty(\x)$, which means $\langle \xi, \x\rangle = \norm{\xi}_*$, and consider an arbitrary $\xx\in \Ban$. If $\norm{\xx}\leq 1$ we obtain
 \begin{align*}
 \rev{
\chara_{\cl{\unitb}}(\xx)-\langle \xi, \xx-\x\rangle
=-\langle \xi, \xx\rangle +\|\xi\|_* 
\geq
\|\xi\|_*\,( 1-\|\xx\|)
\geq 0
=\chara_{\cl{\unitb}}(\x),}
 \end{align*}
while for $\|\xx\|> 1$ the inequality holds trivially, thus we have $\xi\in \partial\chara_{\cl{\unitb}}(\x)$.\\
\textbf{Step 2:} $\mathcal{J}_\infty(\x)\supset\partial\chara_{\cl{\unitb}}(\x)$.\\
Let $\xi \in \partial\chara_{\cl{\unitb}}(\x)$, then for all $\xx\in \cl{\unitb}$ we get
\begin{align*}
\underbrace{\partial\chara_{\cl{\unitb}}(\xx)}_{=0}\geq \underbrace{\partial\chara_{\cl{\unitb}}(\x)}_{=0}+\langle \xi,\xx-\x\rangle \Longleftrightarrow \langle \xi, \xx\rangle\leq \langle \xi, \x\rangle\leq \|\xi\|_*.
\end{align*}
Taking the supremum over all $\xx\in \cl{\unitb}$ yields the equality $\langle \xi, \x\rangle= \|\xi\|_*$ and thus $\xi\in\dual_\infty(\x)$.
\end{proof}
\end{theorem}

Next, we are interested in the behavior of the energy along curves of maximal slope. We 
derive a more general chain rule for 
subdifferentiable energies, that only requires differentiability along curves. 

\begin{lemma}[Chain rule]\label{rm:ChainR}
Let $u:[0,T] \rightarrow \operatorname{dom}(\funccp)$ be a curve, then at each point $t$ where $u$ and $\funccp \circ u$ are differentiable and $\partial \funccp(u(t)) \neq \emptyset$, we have
\begin{align}
\frac{d}{dt} \funccp(u(t))=\langle \xi, u'(t) \rangle \quad \forall \xi \in \partial \funccp(u(t)) \label{eq: Chain}.
\end{align}
\end{lemma}
\begin{proof}
Let $t\in[0,T]$ be a point, where $u$ and $\funccp\circ u$ are differentiable, then we use the definition of the derivative, to obtain
\begin{align*}
\frac{d}{dt}\funccp(u(t))-\langle \xi,u'(t)\rangle&=\lim_{n\to\infty} \frac{\funccp(u(t+h_n))-\funccp(u(t))-\langle \xi , u(t+h_n)-u(t)\rangle}{h_n} =: (\spadesuit),
\end{align*}
where $\{h_n\}_n$ is a null sequence. We first consider only positive null sequences $h_n>0$, where we want to ensure that $u(t+h_n)\neq u(t)$. If such a sequence does not exist, we infer that
$$\frac{d}{dt}\funccp(u(t))= 0 =u'(t)$$ and \labelcref{eq: Chain} holds. Now assuming that there exists a sequence with $u(t+h_n)\neq u(t)$ we continue,
\begin{align*}
(\spadesuit)=\lim_{n\rightarrow \infty} \underbrace{\frac{\funccp(u(t+h_n))-\funccp(u(t))-\langle \xi , u(t+h_n)-u(t)\rangle}{\|\rev{u}(t+h_n)-u(t)\|}}_{=:l_n}\cdot \underbrace{\frac{\|\rev{u}(t+h_n)-u(t)\|}{h_n}}_{r_n}.
\end{align*}
Note, that $r_n\geq 0$ for all $n\in\N$ since we only allowed positive null sequences. Since $u$ is differentiable and in particular continuous at $t$ and since $\xi\in\partial\funccp(u(t))$, \cref{eq: SubDif} yields
\begin{align*}
\liminf_{n\to\infty} l_n \geq 0,
\end{align*}
i.e., for every null sequence $\{h_n\}_n$ we can find a subsequence $\{h_n\}_n$ such that $l_n$ either converges to some limit $l\geq 0$ or diverges to $+\infty$. In the convergent case, we obtain
\begin{align*}
(\spadesuit) = l \cdot \|u'(t)\| \geq 0.
\end{align*}
In the divergent case we also have $(\spadesuit)\geq 0$, since we can find a $n_0$ such that $l_n$ is non-negative for all $n\geq N$. Using the same arguments as above, but only allowing negative null sequences $h_n < 0$, we instead obtain 
$(\spadesuit)\leq 0.$
This finally yields
\begin{align*}
\frac{d}{dt}\funccp(u(t))-\langle \xi,u'(t)\rangle =  0.
\end{align*}
\end{proof}

The chain rule from \cref{rm:ChainR}, together with the characterization of the metric slope \labelcref{eq: slope} enables us to show, that energy dissipation inequality \labelcref{eq:InfFlow} can be equivalently characterized via a differential inclusion.
\begin{theorem}\label{thm: GradFlowForm}
Let $\funccp: \Ban \rightarrow (-\infty,+\infty]$ satisfy \labelcref{eq: slope}  and $u:[0,1] \rightarrow \Ban$ be an a.e.~differentiable Lipschitz curve. Let further $\funccp\circ u$ be a.e. equal to a non-increasing function $\nimap$, then the following are equivalent:
\begin{enumerate}[label=(\roman*)]
\item     
$|u'|(t)\leq 1$ and $\nimap'(t)\leq -|\partial \funccp|(u(t))$ 
for a.e. $t\in[0,T]$,
\item
$\mathcal{J}_\infty(u'(t)) \supset -\partial^\circ \funccp(u(t)) \neq \emptyset$ for a.e. $t\in[0,1]$,
\item 
$u'(t) \in \partial \|\cdot\|_*(-\xi)\cap\Ban=-\argmax_{\x\in \cl{B_1}}\langle \xi, \x\rangle$ for all $\xi \in \partial^\circ \funccp(u(t)) \not= \emptyset,$ and a.e. $t\in (0,T)$.
\end{enumerate}
\end{theorem}
\begin{proof}\phantom{-}\\
\textbf{Step 1}: $(i)\Leftrightarrow (iii)$.\\
Since $\psi$ is a monotone function it is differentiable a.e., and thus we can find a Lebesgue null set $N\subset[0,T]$, such that $u$ and $\psi$ are differentiable and  $\funccp(u(t))=\psi(t)$ for every $t\in[0,T]\setminus N$. Using \cref{rm:ChainR} and \labelcref{eq: slope}, for $t\in[0,1]\setminus N$ we obtain, 
\begin{align*}
\begin{aligned}
\left.\begin{array}{cc}
\psi'(t) \leq -|\partial \funccp|(u(t))  \\
|u'|(t)\leq 1
\end{array}\right\}
&\Leftrightarrow
\left\{
\begin{array}{cc}
\langle \xi, u'(t)\rangle
=\psi'(t) \leq -\|\xi\|_*\quad\text{for all}\quad\xi\in\partial^\circ\funccp(u(t))\\
|u'|(t)\leq 1
\end{array} 
\right.\\[.5cm]
&\Leftrightarrow 
\langle \xi, u'(t)\rangle\leq - \|\xi\|_* - \chara_{\cl{\unitb}}(u'(t))\quad\text{for all}\quad\xi\in\partial^\circ\funccp(u(t)).
\end{aligned}
\end{align*}
For each $\xi\in\partial^\circ\funccp(u(t))$, the last statement is \cref{prop:youngii} with $f=\chara_{\cl{\unitb}}$ and $f^*=\|\cdot\|_*$, which is equivalent to \cref{prop:youngiv}, i.e.,
\begin{align}\label{eq: equival}%
\langle \xi, u'(t)\rangle\leq - \|\xi\|_* - \chara_{\cl{\unitb}}(u'(t))
\quad &\Leftrightarrow\quad 
u'(t) \in \partial \|\cdot \|_{*}(- \xi),
\end{align}
and thus we have shown $(i)\Leftrightarrow(iii)$. The set identity in $(iii)$,
\begin{align*}
\partial \norm{\cdot}_*(-\xi)\cap\Ban = -\argmax_{\x\in\cl{\unitb}} \langle \xi, \x\rangle,
\end{align*}
follows from \cref{cor:charainclusion}.\\
\textbf{Step 2}: $(i)\Leftrightarrow (ii)$.\\
Using the equivalence of \cref{prop:youngii} and \cref{prop:youngi} in \cref{eq: equival} we also obtain that for a.e. $t\in[0,T]$ and all $\xi\in\partial\funccp^\circ(u(t))$
\begin{align*}
(i) \quad&\Leftrightarrow \quad
-\xi \in \partial \chara_{\cl{\unitb}}(u'(t)).
\end{align*}
From Asplund's Theorem (\cref{thm: Asp}) we have that
\begin{align*}
-\xi \in \partial \chara_{\cl{\unitb}}(u'(t)) \Longleftrightarrow -\xi \in \mathcal{J}_\infty(u'(t))
\end{align*}
which thus implies $(i)\Leftrightarrow (ii)$.
\end{proof}

\subsection{Semi-implicit time stepping}\label{ch: LinTi}
The minimizing movement scheme in \labelcref{eq: MinMov} can be considered as an implicit time stepping scheme, which is often computationally intractable in practice. Therefore, one may want to instead employ an explicit scheme. In this regard, we are interested in minimizing movement schemes of the semi-implicit energy, which in many cases can be computed explicitly. We consider a Banach space $\Ban$ that fulfills \cref{asm: 01,asm: 02} and a $C^1$-perturbation of a convex function $\funccp=\funccpd+\funccpc$, fulfilling assumptions \cref{asm: lc,asm: 03b}. Furthermore, we assume:
\begin{mdframed}
\addtocounter{asscount}{1}
\setcounter{assume}{0}
\begin{assume}[Lipschitz continuous differentiability]\label{asm: LipDi}
The differentiable part $\funccpd$ has a Lipschitz continuous first derivative.
\end{assume}
\end{mdframed}

We can linearize the differentiable part of the energy around a point $z$  and define the linearized energy by
\begin{align*}
\funccps(\x; \xx)\coloneqq \funccpd(\xx)+\langle D \funccpd(\xx),\x-\xx\rangle+\funccpc(\x).
\end{align*}
To ensure that the minimizers in \labelcref{eq: SemiEx} are obtained, we assume:
\begin{mdframed}
\setcounter{assume}{1}
\begin{assume}[Lower semi-continuity]\label{asm: LoSe} The semi linearization 
$\x\mapsto\funccps(\x; \xx)$ is $\sigma$-lower semicontinuous for every $\xx\in\Ban$.
\end{assume}
\end{mdframed} 
\begin{remark}\label{rem:LoSe}
In reflexive spaces this is a very mild assumption, as the $\sigma$-topology is often chosen to be the weak topology. In this case, we only need an assumption on the convex part $\funccpc$, namely lower semicontinuity, which together with convexity implies weak lower semicontinuity. The linearized part $\x\mapsto\funccpd(\xx)+\langle D \funccpd(\xx),\x-\xx\rangle$ is even weakly continuous and therefore, we do not need additional assumptions.
\end{remark}

\begin{definition}[Semi-implicit Scheme]\label{def: SeLiSch}
For $\x^0\in\dom(\funccpc)$, we define the semi-implicit scheme as
\begin{align}\label{eq: SemiEx}
\xs_{\tau}^{k+1}\in\argmin_{\x\in\cl{B_\tau}(\xs_\tau^k\rev{)}} \funccps(\x;\xs_\tau^k),
\end{align}
for $k\in\N$ with $\xs^0_\tau=\x^0$.
We define the step function $\sxs_\tau$ by 
\begin{align*}
\sxs_\tau(0)=\x^0, \quad \sxs_\tau(t)=\xs^k_\tau \qquad\text{if}\qquad t\in (t_\tau^{k-1},t^k_\tau], k\geq1.
\end{align*}
Furthermore, we define 
\begin{align*}
|\xs^\prime_\tau|(t)\coloneqq \frac{d(\xs^k_\tau,\xs_\tau^{k-1})}{t_\tau^k-t_\tau^{k-1}} \text{ if } t \in(t_\tau^{k-1},t_\tau^k)
\end{align*}
as the metric derivative of the corresponding piecewise affine linear interpolation. 
\end{definition}
\rev{%
\begin{remark}
The above scheme can also be recovered via the theory of doubly non-linear equations developed in \cite{mielke2013nonsmooth}. Namely, by considering the state dependent dissipation potential
\begin{align*}
\Psi_z(v) := \chi_{\cl{B_1}}(v) + \funccpd(z) + \langle D\funccpd(z), v \rangle 
\end{align*}
the minimizing movement scheme defined in \cite[Eq. 4.9]{mielke2013nonsmooth} is given as
\begin{align*}
\xs_\tau^{k+1}\in \argmin_{x\in \Ban} \left\{\tau\Psi_{\xs^k_\tau}\left(\frac{x-\xs^k_\tau}{\tau}\right) + \funccpc(x)\right\}
\end{align*}
which exactly recovers the scheme defined in \cref{def: SeLiSch}. The authors show convergence of this scheme towards solution of the equation
\begin{align*}
\partial \Psi_{u(t)}(u'(t)) + \partial \funccpc(u(t)) \ni 0
\end{align*}
which corresponds to the inclusion derived in \cref{thm: GradFlowForm}. However, we cannot directly apply the results of \cite{mielke2013nonsmooth} since the choice of dissipation potential as above violates condition (2.$\Psi_1$), since $\dom(\Psi) \neq \Ban$, (2.$\Psi_2$) since in general $\Psi_u(0)\neq 0$ and the growth condition on the Fenchel conjugate $\Psi_z^*(\xi) = \norm{\xi - D\funccpd(z)}_* - \funccpd(z)$ is not fulfilled and also (2.$\Psi_3$). In fact, a more detailed study on how these assumptions could be relaxed would be very interesting, which we however leave for future work.
\end{remark}
}
An important special case, of the above scheme, is a reflexive Banach space $\Ban$ together with a $C^1$ energy $\funccp$, i.e., we can choose $\funccpc=0$. In this case, the scheme is fully explicit, as the following lemma shows.
\begin{lemma}\label{lem:semiexp}
If the Banach space $\Ban$ is reflexive, and $\funccp\in C^1(\Ban)$ then we can explicitly compute the iterates in \cref{def: SeLiSch} as
\begin{align*}
\xs_{\tau}^{k+1} \in  \xs_{\tau}^{k}-\tau\, \partial \|\cdot\|_*(D \funccp(\xs_{\tau}^{k})). 
\end{align*}
\end{lemma}
\begin{proof}
We compute
\begin{align*}
\xs_{\tau}^{k+1}&
\in\argmin_{\x: \|\x-\xs_{\tau}^{k}\|\leq \tau} \funccp(\xs_{\tau}^{k})+\langle D \funccp(\xs_{\tau}^{k}) , \x-\xs_{\tau}^{k} \rangle\\ 
&=\argmin_{\x: \|\x-\xs_{\tau}^{k}\|\leq \tau} \langle D \funccp(\xs_{\tau}^{k}) , \x\rangle\\
&=-\argmax_{\x: \|\x-\xs_{\tau}^{k}\|\leq \tau} \langle D \funccp(\xs_{\tau}^{k}) , \x\rangle\\
&=
\xs^k_\tau
-\tau \argmax_{\x\in \cl{B_1}} \langle D \funccp(\xs_{\tau}^{k}), \x\rangle\\
&=\xs_{\tau}^{k}-\tau\, \partial \|\cdot\|_*(D \funccp(\xs_{\tau}^{k})),
\end{align*}
where for the last identity, we used \ref{cor:charainclusion}.
\end{proof}
In \cref{ch: AdAt} we consider a case, where $\funccpc\neq 0$, but the scheme can still be computed explicitly. In fact, the iteration then coincides with \labelcref{eq: IFGSM}, which ultimately yields the desired convergence result. 

It is easy to see that the metric slope of $ \funccp$ and its semi linearization $ \funccps(\cdot;z)$ coincide in the point of linearization $\xx$, i.e. $|\partial \funccp|(\xx)=|\partial \funccps(\cdot;\xx)|(\xx)$. The next lemma estimates the difference of their slope when $u$ is not the point of linearization.   

\begin{lemma}\label{lm: SlLinEs}
Let $\funccp$ be a $C^1$-perturbation of a convex function satisfying \cref{asm: LipDi}, then for each $z,\rev{x}\in\Ban$ we have the following estimate
\begin{align}\label{eq: LinSlop}
\abs{|\partial \funccp|(\x)-|\partial \funccps(\cdot;z)|(\x)}\leq \Lip(D \funccpd) \|z-\x\|.
\end{align}
\end{lemma}
\begin{proof}

Let $z,\rev{x}\in\Ban$, from \cref{pro:C1:i} we know
\begin{align*}
\partial \funccp(\x) &= \partial\funccpc(\x) + D\funccpd(\x),\\
\partial \funccps(\x; z) &= \partial\funccpc(\x) + D\funccpd(\xx),
\end{align*}
and then \cref{pro:C1:iii} implies that there exists $\xi_1,\xi_2\in\partial\funccpc(\x)$ such that
\begin{align*}
\abs{\partial \funccp}(\x) &= \min\set{
\norm{\xi + D\funccpd(\x)}_*:\xi\in\partial\funccpc(\x)} = \norm{\xi_1 + D\funccpd(\x)}_*,\\
\abs{\partial\funccps(\cdot; z)}(\x) &= \min\set{
\norm{\xi + D\funccpd(\xx)}_*:\xi\in\partial\funccpc(\x)} = \norm{\xi_2 + D\funccpd(\xx)}_*.
\end{align*}
We can then estimate
\begin{align*}
\abs{\partial\funccp}(\x)
&\leq 
\norm{D \funccpd(\x)+ \xi_2}_*
\leq 
\|D \funccpd(\x)-D \funccpd(\xx)\|_*+\|D \funccpd(\xx)+\xi_2\|_*\\
&\leq \Lip(D \funccpd)\|\x-z\|+\abs{\partial\funccps(\cdot; z)}(\x),
\end{align*}
and therefore
\begin{align*}
\abs{\partial\funccp}(\x) -  \abs{\partial\funccps(\cdot; z)}(\x)\leq  \Lip(D \funccpd)\|\x-z\|.
\end{align*}
Analogously, we estimate 
\begin{align*}
\abs{\partial \funccps(\cdot,\xx)}(\x)
&\leq \| D \funccpd(\xx) +\xi_1\|\leq \| D \funccpd(\xx)-D \funccpd(\x)\|+\|D \funccp(\x)+\xi_1\|\\
&\leq \Lip(D \funccpd)\|\x-z\|+|\partial\funccp| (\x).
\end{align*}
and therefore
\begin{align*}
\abs{\partial\funccps(\cdot; z)}(\x)-\abs{\partial\funccp}(\x)\leq  \Lip(D \funccpd)\|\x-z\|.
\end{align*}
This concludes the proof.
\end{proof}
In the following, we want to define 
a variational interpolation similar to \cref{def: DeGiorgi}. Therefore, we consider 
\begin{align*}
\funccps_\tau(\x;\xx) = \min_{\xatt\in \cl{B_\tau}(\x)} \funccps(\xatt;\xx).
\end{align*}
For better readability, if $\xx$ and $\x$ coincide above, we set
\begin{align*}
\funccps_\tau(\x)\coloneqq\funccps_\tau(\x;\x)= \min_{\xatt\in \cl{B_\tau}(\x)} \funccps(\xatt; \x).
\end{align*}
\begin{definition}
[Semi-implicit variational interpolation]\label{df: SemiLin}
We denote by $\vxs_\tau:[0,T]\rightarrow \Ban$ any interpolation of the discrete values satisfying
\begin{align*}
\vxs_\tau(t)
\in \arg\min_{\x} 
\set{\funccps(\x;\xs^{k-1}_\tau)\st d(\x,\xs_\tau^{k-1})\leq t-t^{k-1}_\tau}
\end{align*}
if $t \in (t_\tau^{k-1},t^k_\tau]$ and $ k\geq1$. Furthermore, we define
\begin{align}
\mathcal{D}_\tau(t)\coloneqq 
\frac{d}{dt} \funccps_{(t-t_\tau^{k-1})}(\xs^{k-1}_\tau).
\end{align}
\end{definition}
The following Lemma shows that the variational interpolation of the semi-implicit minimizing movement scheme satisfies the same properties, \labelcref{eq: DDef} and \labelcref{eq: D}, as the \textit{De  Giorgi variational interpolation}, up to an error in $\mathcal{O}(\tau)$. 
\begin{lemma}\label{lm: SemiLin}
We have that
\begin{align}\label{eq: LinDDef}
\mathcal{D}_\tau(t)= 
\frac{d}{dt} \funccps_{(t-t_\tau^{k-1})}(\xs^{k-1}_\tau) 
\leq -|\partial \funccps(\cdot;\xs_\tau^{k-1})|(\vxs_\tau(t))=-|\partial \funccp |(\vxs_\tau(t))+\mathcal{O}(\tau) \text{ if } t\in(t_\tau^{k-1},t_\tau^k]
\end{align}
and
\begin{align}\label{eq: LinD}
\funccp(\vxs_\tau(s))+\int_s^t D_\tau(r) dr\geq \funccp(\vxs_\tau(t)) +\mathcal{O}(\tau)\quad\forall\  0\leq s\leq t\leq T.
\end{align}
\end{lemma}

\begin{proof}
For \labelcref{eq: LinDDef}, we apply \cref{lem: Dif} to the mapping 
$\x\mapsto\funccps(\x; \xs_\tau^{k-1})$ to obtain 
\begin{align*}
\frac{d}{dt} \funccps_{(t-t_\tau^{k-1})}(\x;\xs_\tau^{k-1}) \leq |\partial \funccps(\cdot,\xs_\tau^{k-1})|\left(\x_{\text{min},t-t^{k-1}_\tau}\right),
\end{align*}
where $\x_{\mathrm{min},t-t^{k-1}_\tau}\in \argmin_{\xatt} \{\funccps(\xatt; \xs_\tau^{k-1})\st \xatt\in \cl{B_\tau}(\x)\}$. Choosing $v=\xs_\tau^{k-1}$ then yields
\begin{align*}
\frac{d}{dt} \funccps_{(t-t_\tau^{k-1})}(\xs_\tau^{k-1}) \leq -|\partial \funccps(\cdot;\xs_\tau^{k-1})|(\vxs_\tau(t)).
\end{align*}
The last equality of \labelcref{eq: LinDDef}, follows by \cref{lm: SlLinEs}. To show \labelcref{eq: LinD} we again use \cref{lem: Dif} and get
\begin{align*}
\funccps(\vxs_\tau(s);\xs^k_\tau)+\int_{s}^t \mathcal{D}_\tau(r) dr\geq  \funccps(\vxs_\tau(t),\xs^k_\tau) \quad\text{for all}\quad t_\tau^k\leq s\leq t\leq t_\tau^{k+1}.
\end{align*}
Due to \cref{thm:Taylor}
\begin{align*}
\abs{\funccps(\vxs_\tau(s);\xs^k_\tau)-\funccp(\vxs_\tau(s))}&=\abs{
\funccpd(\xs^k_\tau) + \langle D\funccpd(\xs^k_\tau), \vxs_\tau(s) - \xs^k_\tau\rangle - \funccpd(\vxs_\tau(s))
}
\\
&\leq \bigg|\int_0^1 
\left\langle D 
\funccpd\left(\xs^k_\tau+r
(\vxs_\tau(s)-\xs^k_\tau)\right),\xs^k_\tau-\vxs_\tau(s) \right\rangle\\
&\quad-\langle D \funccpd(\xs^k_\tau),\vxs_\tau(t)-\xs^k_\tau \rangle dr\bigg|\\
&\leq \int_0^1 r \Lip(D \funccpd)\| \vxs_\tau(s)-\xs^k_\tau\|^2dr \\
&\leq \frac{1}{2} \Lip(D \funccpd)\| \vxs_\tau(s)-\xs^k_\tau\|^2 \leq \frac{1}{2} \Lip(D \funccpd) \tau^2
\end{align*}
and analogusly
$|\funccps(\vxs_\tau(t),\xs^{k}_\tau)-\funccp(\vxs_\tau(t))|\leq\frac{1}{2} \Lip(D \funccpd) \tau^2$. Therefore, for all $t_\tau^k\leq s\leq t\leq t_\tau^{k+1}$, we have that
\begin{align} \label{eq: PartEDI}
\begin{aligned}
\funccp(\vxs_\tau(s))+\int_{s}^t \mathcal{D}_\tau(r) dr &\geq 
\funccps(\vxs_\tau(s); \x_\tau^k)
+\int_{s}^t \mathcal{D}_\tau(r) dr 
-\frac{1}{2} \Lip(D \funccpd) \tau^2\\
&\geq 
\funccps(\vxs_\tau(t); \x_\tau^k)
-\frac{1}{2} \Lip(D \funccpd) \tau^2\\
&\geq  \funccp(\vxs_\tau(t)) - \Lip(D \funccpd) \tau^2.
\end{aligned}
\end{align}
Now for $s\in [t_\tau^m,t_\tau^{m+1}]$ and $t\in [t_\tau^k,t_\tau^{k+1}]$ with $m\leq k$ we add up \labelcref{eq: PartEDI} to obtain
\begin{align*}
\funccp(\vxs_\tau(s))&+ \int_s^{t_\tau^{m+1}} \mathcal{D}_\tau(r) dr +\sum_{i=m+1}^{k-1} \int_{t_\tau^{i}}^{t_\tau^{i+1}} \mathcal{D}_\tau(r) dr+\int_{t_\tau^k}^t \mathcal{D}_\tau(r) dr \\
&\geq \funccp(\vxs_\tau(t))-\sum_{i=m}^k \Lip(D \funccpd) \tau^2\\
&=
\funccp(\vxs_\tau(t))- (k-m) \Lip(D \funccpd) \tau^2
\\
&\geq
\funccp(\vxs_\tau(t))- T \Lip(D \funccpd) \tau
\end{align*}
such that we finally obtain \labelcref{eq: LinD}.
\end{proof}

As an immediate consequence of \cref{lm: SemiLin}, we can replace the minimizing movement scheme in the proof of \cref{thm: exist} by the semi-implicit scheme, as the error terms are of order $\mathcal{O}(\tau)$ and vanish during the limiting process $\tau \rightarrow 0$.
Then $\sxs_\tau$ $\sigma$-converges up to a subsequence to a $\infty$-curve of maximal slope.

\begin{theorem}\label{thm:semilincvg} Let $\funccp$ be a $C^1$-perturbation of a convex function. Under \cref{asm: 01,asm: 02,asm: lc,asm: 03b,asm: LipDi,asm: LoSe} there exists a $\infty$-curve of maximal slope $u(t)$, with respect to the energy $\funccp$ and its upper gradient $|\partial\funccp|$, and a
subsequence of $\tau_n=T/n$ such that
\begin{align*}
\sxs_{\tau_n}(t) \stackrel{\sigma}{\rightharpoonup}u(t) \text{ as } n\rightarrow \infty \quad \forall t\in [0,T].
\end{align*}
\end{theorem}
\begin{proof}
We simply replace the minimizing movement scheme in \cref{def: MinMov} and De Giorgis variational interpolation (see \cref{def: DeGiorgi}) by the semi-implicit scheme in \cref{def: SeLiSch} and its corresponding variational interpolation of \cref{lm: SemiLin}. Proceeding similarly as in the proof of \cref{thm: exist}, we use \cref{prop:AA} to show     
\begin{align*}
\sxs_{\tau_n}(t) \stackrel{\sigma}{\rightharpoonup}u(t) \text{ as } n\rightarrow \infty \quad \forall t\in [0,T]
\end{align*} 
for a subsequence $\tau_n$, where $u$ is a \rev{$1$-}Lipschitz curve with $u(0)=\x^0$. Then the same holds true for $\vxs_{\tau_n}(t)$. 
\rev{Taking for $\tau_n$ the limes inferior for $n\rightarrow \infty$ of \cref{eq: LinD} and using \cref{asm: lc}, \cref{asm: 03b} and \cref{eq: LinDDef} we again obtain
\begin{align*}
    \funccp(u(0))\geq \funccp(u(t))+\int_0^t |\partial\funccp|(u(r)) dr \quad \text{ for all }t\in[0,T].
\end{align*}
Since on the other hand $|\partial \funccp|$ is a strong upper gradient, equality in the above equation must hold.
}
\end{proof}

\begin{remark}
Let $\tau_n$ be any sequence such that $\tau_n \rightarrow 0$.
 If the $\infty$-curve of maximal slope $u$ is unique, we can apply \cref{thm:semilincvg} to every subsequence of $\tau_n$ and find a further subsequence $\Tilde{\tau}_n$ such that 
\begin{align*}
\sxs_{\Tilde{\tau}_n}(t) \stackrel{\sigma}{\rightharpoonup}u(t) \text{ as } n\rightarrow \infty \quad \forall t\in [0,T].
\end{align*}
This implies that already for $\tau_n$
\begin{align*}
\sxs_{\tau_n}(t) \stackrel{\sigma}{\rightharpoonup}u(t) \text{ as } n\rightarrow \infty \quad \forall t\in [0,T].
\end{align*}
and the semi-implicit scheme converges.     
\end{remark}

\section{Wasserstein infinity flows}  \label{sec: Wass}

The previous sections consider a \enquote{single particle}, $\x\in\Inp$, trying to minimize an energy $\func$, by following an $\infty$-curve of maximal slope. This single particle may be drawn from a probability distribution $\mu_0 \in \mathcal{P}(\WBan)$ which over time also minimizes an energy $\mathcal{E}$ defined on the space of probabilities. In this section, we choose the underlying metric space $\CMS$ to be the space of Borel probability measures with bounded support $\mathcal{P}_\infty(\WBan)$, and equip it with the $\infty$-Wasserstein distance. 
We show that for potential energies, $\infty$-curves of maximal slope can be expressed via a probability measures $\eta$ on the space $C(0,T;\WBan)$ which is concentrated on $\infty$-curves of maximal slope on the underlying Banach space $\WBan$. From $\eta$ we can then derive a corresponding continuity equation which those $\infty$-curves of maximal slope have to fulfill.
\rev{
This concept is commonly referred to as the \enquote{superposition principle}, where our approach directly follows the setup of \cite{ambrosio05,lisini2014absolutely,Lisini07}. We refer to \cite{stepanov2017three} for an overview of different works in this direction, as well as results that hold true in a much more general setting.
}
%

\subsection{Preliminaries on Wasserstein spaces}
We give a brief introduction to the basic properties of Wasserstein spaces. For more details, we refer to \cite{ambrosio05, givens1984class, villani2009optimal}. 
In the following $(\WBan, \|\cdot\|)$ is a separable Banach space. We denote by $\Prob(\WBan)$ the space of Borel probability measures on $\WBan$. For $1\leq p<\infty$ $\Prob_p(\WBan) \subset \Prob(\WBan)$ is the subset of measures with finite $p$-momentum, while $\Prob_\infty(\WBan)\subset \Prob(\WBan)$ is the subset of measures with bounded support. For $1\leq p<\infty$ and $\mu,\nu \in \Prob_p(\WBan)$ we define the $p$-Wasserstein distance as
\begin{align*}
W^p_p(\mu,\nu)\coloneqq \inf_{\gamma\in\Gamma(\mu,\nu)} \int \| \x-\xx\|^p d\gamma(\x,\xx).
\end{align*}
Here,
\begin{align}\label{eq: pWass}
\Gamma(\mu,\nu)\coloneqq\{\gamma\in\Prob(\WBan\times \WBan): \pi^1_\# \gamma=\mu, \pi^2_\# \gamma=\nu\},
\end{align}
is the set of admissible transport plans and $\pi^1(\x,\xx)=\x$, $\pi^2(\x,\xx)=\xx$ denote the projection on the first and second component.
For $\mu,\nu \in \Prob_\infty(\WBan)$ the $\infty$-Wasserstein distance is given by
\begin{align}\label{eq: IWass}
W_\infty(\mu,\nu)\coloneqq \inf_{\gamma\in\Gamma(\mu,\nu)} \gamma-\esssup \| \x-\xx\|.
\end{align}
In both cases the minimum of \labelcref{eq: pWass} and \labelcref{eq: IWass} is obtained (see, e.g., \cite{ambrosio05, villani2009optimal} and \cite[Proposition 1]{givens1984class} for the case $p=\infty$) and $\Gamma_0(\mu,\nu)$ denotes the set of optimal transport plans where the minimum is reached.

\begin{proposition}[{\cite[Proposition 6.]{givens1984class}}]
 For $p\in[1,\infty]$,  $\mathcal{W}_p=(\Prob_p(\WBan), W_p)$, i.e., $\Prob_p(\WBan)$ equipped with the $p$-Wasserstein distance, is a complete metric space. For $p<\infty$, $\mathcal{W}_p$ is separable.
\end{proposition}
The following lemma shows that Wasserstein distances are ordered in such a way that they get stronger by increasing $p$, see \cite[Proposition 3.]{givens1984class}.
\begin{lemma}[{\cite[Proposition 3.]{givens1984class}}]\label{lem:incrW}
    For $1\leq p\leq q \leq \infty$ and $\mu,\nu\in \Prob(\WBan)$
    \begin{align}\label{eq:WassOrder}
        W_p(\mu,\nu)\leq W_q(\mu,\nu)
    \end{align} 
    and in particular
    \begin{align*}
        W_\infty(\mu,\nu)=\sup_p W_p(\mu,\nu)=\lim_{p\rightarrow \infty}W_p(\mu,\nu).
    \end{align*}
\end{lemma}

Let now $\sigma$ denote the narrow topology, namely, $\mu^n\stackrel{\sigma}{\rightarrow}  \mu$ iff,
\begin{align}\label{eq: narrow}
\int_{\WBan} \varphi d\mu^n \rightarrow  \int_{\WBan} \varphi d\mu \quad \forall \varphi\in C_b(\WBan),
\end{align}
where $C_b(\WBan)$ denotes the space of bounded and continuous functions on $\WBan$. The next lemma is helpful, when we are considering limits in \labelcref{eq: narrow} with $\varphi$ being unbounded or only lower semicontinuous.
\begin{lemma}[{\cite[Lemma 5.1.7.]{ambrosio05}}]
Let $\seq{\mu}$ be a sequence in $\Prob(\WBan)$ narrowly converging to $\mu \in \Prob(\WBan)$. If $g: \WBan \rightarrow (-\infty,+\infty]$ is lower semicontinuous and \rev{its negative part $g^-=-\min\{g,0\}$} is uniformly integrable w.r.t. the set $\{\mu^n\}_{n\in \mathbb{N}}$, then
\begin{align*}
    \liminf_{n\rightarrow \infty}\int_\WBan g(x)d \mu^n(x)\geq \int_\WBan g(x) d\mu(x) >-\infty .
\end{align*}
\end{lemma}
When working with probability measures, Prokhorov's theorem (\cite[Theorem 5.1-5.2]{billingsley2013convergence}, repeated for convenience in the appendix, \cref{thm:prokh}) is useful, since it characterizes relatively compact sets with respect to the narrow topology.
In certain situations the assumption \labelcref{eq: tight} of this theorem, i.e.,
\begin{align*}
\forall \epsilon>0\quad \exists K_\epsilon \text{ compact in } \WBan \text{ such that } \mu(\WBan\setminus K_\epsilon)\leq \epsilon \quad \forall \mu\in \mathcal{K},
\end{align*}
can only be shown for bounded and not compact sets. There we use the observation in the following remark, to still obtain some sort of relative compactness. In the following, we denote by $\WBan_\omega$, the space $\WBan$ equipped with the weak topology $\sigma(\WBan,\WBan^*)$.
\begin{remark}\label{rm: renorm}
 If $\WBan$ is separable and reflexive, then so is its dual. For a countable dense subset $\{x^*_n\}_{n\in\N}$ of $\cl{B_1^{\WBan*}}$, we can define the norm
$$\| x \|_{\omega}=\sum_{n=1}^\infty \frac{1}{n^2} |\langle x^*_n,x\rangle|,$$
which induces the weak topology $\sigma(\WBan,\WBan^*)$ on bounded sets \cite[Lemma 3.2]{morrison2011functional}. This norm is a so-called Kadec norm. In particular, we have that the Borel sigma algebra $\mathcal{B}(\WBan)$, generated by the norm topology, and the one generated by the weak topology $\mathcal{B}(\WBan_{\omega})$ coincide and thus $\Prob(\WBan)=\Prob(\WBan_\omega)$, see \cite[Theorem 1.1]{Edg}.
Now, let us assume that for a set $\mathcal{K}\subset \Prob(\WBan)=\Prob(\WBan_\omega)$ we have that
\begin{align}
    \forall \epsilon>0\quad \exists K_\epsilon\  \|\cdot\|\text{-bounded in } \WBan, \text{ such that } \mu(\WBan\setminus K_\epsilon)\leq \epsilon \quad \forall \mu\in \mathcal{K}.
\end{align}  
Since bounded sets are subsets of $\cl{B_\epsilon}$ for $\epsilon$ large enough and $\cl{B_\epsilon}$ is compact in the weak topology $\sigma(\WBan,\WBan^*)$, Prokhorov's theorem can be applied for $\WBan_\omega$. We obtain that there exists a subsequence $\seq{\mu}\subset \mathcal{K}$ and a limit $\mu\in \Prob(\WBan)=\Prob(\WBan_\omega)$ such that 
\begin{align}
    \int_{\WBan} \varphi d\mu^n \rightarrow  \int_{\WBan} \varphi d\mu \quad \forall \varphi\in C^\omega_b(\WBan),
\end{align}
where $C^\omega_b(\WBan)$ now denotes the set of weakly continuous, bounded functions.
\end{remark}

The next lemma follows from \cite[Theorem  6.9, Corollary 6.11, Remark 6.12]{villani2009optimal} together \cref{lem:incrW}, i.e., \cite[Proposition 3.]{givens1984class}.
\begin{lemma}[Compatibility]\label{lem: Comp}
The \rev{narrow} topology $\sigma$ is weaker than the topology induced by $W_p(\cdot,\cdot)$ on $\Prob_p(\WBan)$, for every $p\in[1,\infty]$. Furthermore, $W_p$ is lower semicontinuous with respect to the narrow topology $\sigma$, i.e., for every $p\in[1,\infty]$:
\begin{align*}
\left. \begin{array}{ll}
  \mu^n\stackrel{\sigma}{\rightarrow}  \mu \\
  \nu^n\stackrel{\sigma}{\rightarrow}  \nu
\end{array}
\right\} \Longrightarrow W_p(\mu,\nu)\leq \liminf_{n\rightarrow \infty}W_p(\mu^n,\nu^n).
\end{align*}
\end{lemma}
\begin{remark}
    For $1\leq p< \infty$ convergence in $W_p$ is equivalent to narrow convergence and convergence of the $p$-th moment \cite[Theorem  6.9]{villani2009optimal}. This equality is lost for the $\infty$-Wasserstein distance, as convergence in the narrow topology ($\mu^n\stackrel{\sigma}{\rightarrow} \mu$) together with $\bigcup_n \supp(\mu^n)$ being bounded or relatively compact no longer \rev{guarantees} convergence in $W_\infty$, as \cref{ex:winfty} demonstrates.
\end{remark}

\begin{example}\label{ex:winfty}
We consider the sequence $$\mu^n=\frac{n-1}{n}\, \delta_0+\frac{1}{n}\, \delta_1,$$ where $\delta_t$ denotes the Dirac measure at $t\in\R$. Then we have that
\begin{align*}
\int_\R \varphi d\mu^n = \frac{n-1}{n}\, \varphi(0)+\frac{1}{n}\, \varphi(1) \xrightarrow[]{n\to\infty} \varphi(0) = \int_\R \varphi d\delta_0
\end{align*}
for every $\varphi\in C_b(\R)$ and thus $\mu^n\stackrel{\sigma}{\rightarrow} \delta_0$. However, we see that
\begin{align*}
W_\infty(\mu^n,\delta_0) = \min_{\gamma\in\Gamma(\mu^n,\delta_0)} \gamma-\esssup \abs{\x-\xx} = \abs{1 - 0} = 1, 
\end{align*}
for every $n\in\N$ and thus we have no convergence in $W_\infty$.
\end{example}
\subsection{Absolutely continuous curves in Wasserstein spaces \srev{and the superposition principle}}
In this section, we \rev{employ the superposition principle to obtain alternative characterizations of absolutely continuous curves in Wasserstein spaces.}
In \cite{Lisini07}, Lisini shows that for $p\in (1,\infty)$, $p$-absolutely continuous curves $\mu:[0,T]\to W_p$ can be written as a push forward of a Borel probability measure over the space of continuous curves. Using this statement, the author was able to derive a \rev{well-known} characterization of absolutely continuous curves via solutions of continuity equations, when the underlying space of $W_p$ is Banach. In \cite{lisini2014absolutely} the first result was extended to Wasserstein--Orlicz spaces, which also covers the $W_\infty$ case.
\rev{%
In \cite[Sec. 4]{stepanov2017three} the authors were able to derive a refined version of the result obtained in \cite{Lisini07} that also includes the case $p=\infty$. For completness, we state the corresponding theorems in this section and provide the proofs that specifically adapt the arguments of \cite{Lisini07} to our setting in \cref{app:help}.
}%
Connected to this, we also refer to the discussion in the book by \textcite[Ch. 5.5.1]{Santambrogio15} and the associated paper \cite{BrLoSa2011}, where this topic was discussed as the limit $p\to\infty$, for $\Inp=\R^d$. We further discuss difficulties arising when the norm of underlying Banach space is not strictly convex.\par 

Let $\mathcal{P}(C(0,T;\WBan))$ denote the space of Borel probability measures on the Banach space of continuous functions on the interval $[0,T]$. 
We define the evaluation map $e_t: C(0,T;\WBan)\rightarrow \WBan$ by 
$$ e_t(u)=u(t).$$
Then absolutely continuous curves in Wasserstein spaces can be represented by a Borel probability
measure on $C(0,T;\WBan)$ concentrated on the set of absolutely continuous curves in $\WBan$, as the following theorem from \cite{lisini2014absolutely} shows. Here, $\AC^p(0,T\rev{;}W_p)$ denotes the set of $p$-absolutely continuous curves $\mu:[0,T]\to W_p$.
\begin{theorem}[{\cite[Thm. 3.1]{lisini2014absolutely}}]\label{thm: curveDe} Let $\WBan$ be separable.
    For $p\in(1,\infty]$, if $\mu\in \AC^p(0,T;\rev{\mathcal{W}}_p)$, then there exists $\eta \in\mathcal{P}(C(0,T;\WBan))$ such that
    \begin{itemize}
        \item $\eta$ is concentrated on $\AC^p(0,T;\WBan)$,
        \item ${e_t}_\# \eta=\mu_t\quad \forall t\in[0,T] $,
        \item for a.e. $t\in [0,T]$ the metric derivative $|u'|(t)$ exists for $\eta$-a.e. $u\in C(0,T;\WBan)$ and it holds the equality 
        $$ |\mu'|(t)=\| |u'|(t)\|_{L_p(\eta)}.$$
    \end{itemize}
\end{theorem}
For a Banach space $(\WBan,\|\cdot\|)$ and a finite measure space $(\Omega, \mathcal{A},\mu)$ we denote for $1\leq p\leq \infty$ the Lebesgue--Bochner space by $L^p(\mu;\WBan)$. A function $f: \Omega \rightarrow \WBan$ belongs to $L^p(\mu;\WBan)$ if it is $\mu$-Bochner integrable and its norm 
\begin{align*}
\|f\|_{L_p(\mu;\WBan)}^p&\coloneqq \int_{\Omega} \|f\|^p d\mu \quad \text{for}\quad 1\leq p < \infty,\\ \|f\|_{L_\infty(\mu;\WBan)}&\coloneqq\mu-\esssup \|f\|\quad p =\infty,
\end{align*}
is finite, see \cite{diestelvector}.
For a narrowly continuous curve $\mu: [0,T]\rightarrow \Prob_p(\WBan)$, we define $\bar{\mu}\in \Prob([0,T]\times \WBan)$ by
\begin{align*}
    \int_{[0,T]\times \WBan } \varphi (t,x) d\Bar{\mu}\coloneqq \frac{1}{T}\int_{[0,T]} \int_{\WBan} \varphi (t,x) d\mu_t (x) dt
\end{align*}
for every bounded Borel function $\varphi: [0,T]\times \WBan\rightarrow \R$.
Let $\bm{v}:[0,T]\times \WBan \rightarrow \WBan$ be a time dependent velocity field belonging to $L^p(\Bar{\mu},\WBan)$, then we say $(\mu,\bm{v})$ satisfies the continuity equation 
\begin{align}\label{eq:ce}\tag{CE}
    \partial \mu_t +\div(\velo_t \mu_t)=0,
\end{align}
if the relation
\begin{align*}
    \frac{d}{dt} \int_\WBan \varphi d\mu_t =\int_\WBan\langle D\varphi,\velo_t\rangle d\mu_t \quad \forall \varphi \in C^1_b(\WBan)
\end{align*}
holds in the sense of distributions in $(0,T)$. Here $C^1_b(\WBan)$ denotes the space of bounded, Fréchet-differentiable  functions $\varphi:\WBan\rightarrow \mathbb{R}$, such that $D\varphi:\WBan\rightarrow \WBan^*$ is continuous and bounded. Using this notion, we define,
\begin{align*}
\EC^p(\WBan)\coloneqq 
\set{
(\mu,\bm{v})\st
\begin{aligned}
\mu:[0,T] \rightarrow \Prob_p(\WBan) \quad&\text{is narrowly continous}, \bm{v}\in L^p(\Bar{\mu};\WBan),\\%
(\mu,\bm{v})\quad&\text{satisfies the continuity equation}
\end{aligned}}.
\end{align*}
As the next theorem shows, the curves contained in  the support of $\eta$ in \cref{thm: curveDe} can be understood as the \enquote{characteristics} of a corresponding transport equation. \rev{The statement is an extension of \cite[Thm. 7]{Lisini07} to the case $p=\infty$. For completness we give an adapted proof in \cref{app:help}. Here we assume that the Banach space also has the Radon--Nikodým property, see, e.g., \cite[Ch. 5]{ryan2002introduction}, which we recall in the following. In particular, every reflexive Banach space has this property, see \cite[Corollary 5.45]{ryan2002introduction}.
\begin{definition}[Radon--Nikodým property]
We say that a Banach space $\WBan$ has the Radon--Nikodým property, if for every vector measure $\mu$ of bounded variation defined over a $\sigma$-algebra $\Sigma$ over $\WBan$, that is absolutely continuous with respect to a finite, positive measure $\lambda$, there exists a $\lambda$-Bochner integrable function $f$ such that $\mu(A) = \int_A f d\lambda$ for all $A\in\Sigma$. 
\end{definition}}
\begin{theorem}\label{thm: ECp} Let $\WBan$ be separable and satisfy the Radon--Nikodým property. 
If $\mu \in \AC^\rev{\infty}([0,T]\rev{;}\mathcal{W}_\rev{\infty})$ then there exists a vector field $\bm{v}: [0,T]\times \WBan\rightarrow \WBan$ such that $(\mu,\bm{v})\in \EC^\rev{\infty}(\WBan)$ and    
\begin{gather}
    \label{LVel} \|\velo_t\|_{L^\rev{\infty}(\mu_t;\WBan)}\leq |\mu'|(t) \quad \text{for a.e. }  t\in(0,T).
\end{gather} 
\end{theorem}
If in addition $\WBan$ satisfies the bounded approximation property (BAP) then the following \cref{thm: ECRp} acts as the counterpart of  \cref{thm: ECp} and states that solutions of the continuity equation are absolutely continuous curves. In particular, for a specific $\mu\in \AC^p([0,T]\rev{;}\rev{\mathcal{W}}_p)$ the velocity $\bm{v}$ field obtained in \cref{thm: ECp} is minimal in the sense that 
\begin{align*}
\|\bm{v}_t\|_{L^p(\mu;\WBan)} =|\mu'|(t)\leq \|\Tilde{\bm{v}}_t\|_{L^p(\mu;\WBan)}\quad \text{for a.e.} \  t\in (0,T)\text { and for all } \Tilde{\bm{v}}\text{ satisfying  } (\mu,\Tilde{\bm{v}})\in \EC^p(\WBan).
\end{align*}
\rev{
We briefly recall the (BAP) and then state \cref{thm: ECRp}, which is an extension of \cite[Thm. 8]{Lisini07} to $p=\infty$. For completness the proof (which again is a slight modification of \cite{Lisini07}) is provided in \cref{app:help}.
}
\begin{definition}[BAP]
A separable Banach space $\WBan$ satisfies the bounded approximation property (BAP), if there exists a sequence of finite rank linear operators $T_n:\WBan\rightarrow \WBan$ such that \begin{align*}
        \lim_{n\rightarrow \infty} \|T_n x-x\|=0.
    \end{align*}
\end{definition}
\rev{
In particular every Hilbert space and every Banach space with a Schauder basis fulfills this property, see \cite[Ch. 9]{Schaefer1999}.
}
\begin{theorem}\label{thm: ECRp}
Assume, that $\WBan$ is separable and satisfies the Radon--Nikodým property  as well as the bounded approximation property (BAP). If $(\mu,\bm{v}) \in \EC^\rev{\infty}(\WBan)$, then $\mu \in \AC^\rev{\infty}([0,T]\rev{;}\rev{\mathcal{W}}_\rev{\infty})$ and
\begin{align*}
|\mu'|(t)\leq \|\velo_t\|_{L^\rev{\infty}(\mu_t;\WBan)} \quad \text{for a.e. } t\in(0,T). 
\end{align*}
\end{theorem}

\begin{remark}[Uniqueness of the Velocity field]
As mentioned before, if $\WBan$ satisfies the bounded approximation property, the velocity field  obtained in \cref{thm: ECp} is minimal. For the case that $p\in(1,+\infty)$ and the norm of the underlying Banach space $\WBan$ is strictly convex, then $\| \cdot\|_{L^p(\mu_t;\WBan)}$ is also strictly convex. Then the uniqueness of the minimal velocity field follows. In the other cases, the uniqueness is lost.
\end{remark}

\begin{remark}
Whenever \cref{thm: ECRp} is applicable, $\|\velo_t\|_{L^\infty(\mu_t;\WBan)}=|\mu'|(t)$ for a.e. $t\in(0,T)$ and thus \cref{eq: JensEst} is actual an equality. For the Wasserstein spaces $p\rev{\in}(1,+\infty)$ we obtain 
\begin{align*}
  \int_\WBan \left\|\int_{C(0,T;\WBan)} u'(t) d\Bar{\eta}_{x,t}\right\|^p d\mu_t= \int_\WBan\int_{C(0,T;\WBan)} \left\|u'(t)\right\|^p d\Bar{\eta}_{x,t} d\mu_t\quad \text{for a.e. } t\in (0,T)
\end{align*}
or equivalently 
\begin{align*}
    \left\|\int_{C(0,T;\WBan)} u'(t) d\Bar{\eta}_{x,t}\right\|^p=\int_{C(0,T;\WBan)} \left\|u'(t)\right\|^p d\Bar{\eta}_{x,t} \quad \text{ for } \Bar{\mu}\text{-a.e. } (t,x)\in [0,T]\times\WBan.
\end{align*}
from corresponding calculations \cite[Theorem 7]{Lisini07}. Notice that this is the \rev{equality} case of Jensen's inequality.
For a strictly convex norm $\|\cdot\|$, this equality can only hold when $u'(t)$ is constant $\Bar{\eta}_{x,t}$-a.e. Thus, heuristically spoken, all curves passing through a point $x\in \WBan$ at time $t$ have the same derivative. This is in particular the reason why on an infinitesimal level optimal transport plans $\gamma_h\in \Gamma(\mu_t,\mu_{t+h})$ behave like classical optimal transport, i.e., for a.e. $t\in(0,T)$ (\rev{see \cite[Proposition 8.4.6]{ambrosio05}), }
 \begin{align*} 
        \lim_{h \rightarrow 0} (\pi^1,\frac{1}{h}(\pi^2-\pi^1))_{\#} \gamma_h=(\Id\times \velo_t)_{\#}\mu_t\quad \text{in } \mathcal{P}(\WBan\times \WBan).
\end{align*}
This argument fails in the case $W_\infty$ or when the norm $\|\cdot\|$ is not strictly convex.
\end{remark}

\subsection{Curves of maximal slope of potential energies}\label{sc: CurvPot}
In addition to being separable, we now assume $\WBan$ to be reflexive, and we need the following assumption on the potential $\pot$.
\addtocounter{asscount}{1}
\setcounter{assume}{0}
\begin{mdframed}%
\begin{assume}\label{assm: pota}%
Let $\pot:\WBan \rightarrow (-\infty,+\infty]$ be weakly continuous on its domain, which we assume to be closed and convex.
\end{assume}%
\end{mdframed}
 
The potential energy $\func: \Prob_\infty (\WBan) \rightarrow (-\infty,+\infty]$ is defined as
\begin{align*}
    \func(\mu)\coloneqq \int \pot(\x) d\mu(\x).
\end{align*}
As in \cref{ch: existence}, we consider a minimizing movement scheme, approximating curves of maximal slope, where in each step the following minimization problem arises, 
\begin{align} \label{eq:Min}
\argmin_{\muatt : W_\infty(\muatt,\mu)\leq \tau} \int \pot(x) d\muatt(x).
\end{align}
Notably the $\infty$-Wasserstein distance in \labelcref{eq:Min} restricts the movement of mass uniformly. Intuitively this means that for every point $x \in \WBan$ we need to solve the local problem
\begin{align} \label{eq:loMin}
\bm{r}_\tau(x)\coloneqq\argmin_{\xatt\in \cl{B_\tau(\x)}} \pot(\xatt),
\end{align}
where $\bm{r}_\tau(x): \WBan \mto \WBan $ is a possibly multivalued correspondence, see \cref{sec: corr}. Then a possible optimal transport plan between $\mu$ and a minimizer of \labelcref{eq:Min}, $\mu_{\min}$, should transport the mass from some point $x$ to a minimizing point in $\bm{r}_{\tau}(x)$. 
In this regard, we employ the measurable maximum theorem (\cite[Theorem 18.19]{guide2006infinite}, repeated for convenience in the appendix as \cref{thm:MeaMax})
. This theorem guarantees the measurability of the \enquote{argmin} correspondence in \labelcref{eq:loMin}.  
Definitions of (weak) measurability for correspondences are repeated in \cref{sec: corr}, where we refer to \cite{guide2006infinite} for a detailed overview over the topic. In order to apply the mentioned theorems to the problem in \labelcref{eq:loMin}, we need to check the underlying correspondence for weak measurability.
Let us  define $$\dom_\tau(E)\coloneqq\{x\in \WBan : \|x-z \|\leq\tau \text{ for a } z\in\dom(E)\}.$$
\begin{lemma}\label{lm:Cores}
For $\tau\geq 0 $ the correspondence $\varphi_\tau:(\WBan\cap\dom_\tau(E),\|\cdot\|)\mto (\WBan\cap\dom(E),\|\cdot\|_\omega)$ given by $\varphi_\tau:x \mapsto \cl{B_\tau}(x)$ is weakly measurable \rev{and has nonempty 
 weakly compact
values.}
\end{lemma}
\begin{proof}
Every weakly open set $G\subset\WBan\cap\dom(E)$ is strongly open as well. And for strongly open sets $G$, \rev{the lower inverse as defined in \cref{eq:lowinv} is given as
$$\varphi_\tau^l(G)=\{s\in \WBan |\ \exists\ x\in G \text{ with }\|s-x\|\leq \tau\}.$$
Since $G$ is strongly open, this set is again strongly open} and thus in $\Sigma=\mathcal{B}(\WBan)$\rev{, yielding weak measurability of $\varphi_ \tau$. To conclude we observe that $\overline{B_\tau}(x)$ is nonempty and weakly compact.} 
\end{proof}
The next corollary now follows immediately from the measurable maximum theorem.
\begin{corollary}\label{lem:LocMin} 
Let $\WBan$ be a reflexive, separable Banach space and let $\pot$ fulfill \cref{assm: pota}, then for $\tau \geq 0$ 
\begin{align}
\pot_\tau(x)\coloneqq \min_{\xatt\in\cl{B_\tau}(\x)} \pot(\xatt)
\end{align}
is $\mathcal{B}(\WBan)$-measurable.
The correspondence $\bm{r}_\tau : \WBan \mto \WBan$
\begin{align} \label{eq:loMinSel}
\bm{r}_\tau (x) \coloneqq \argmin_{\xatt\in\cl{B_\tau}(\x)} \pot(\xatt)
\end{align}
has nonempty and compact values, it is measurable and admits a $\mathcal{B}(\WBan)$-measurable selector.
\end{corollary}
\begin{proof}
As mentioned in remark \cref{rm: renorm}, $\mathcal{B}(\WBan)$ and $\mathcal{B}(\WBan_{\omega})$ coincide in this particular setting. We choose the correspondence $\varphi_\tau$ from \cref{lm:Cores} and set $f(s,x)=-\pot(x)$. \rev{Since \cref{assm: pota} guarantees that $f(s,x)=-E(x)$ is a Carathéodory function} the application of \cref{thm:MeaMax} yields this corollary, but only restricted to $\WBan\cap \dom_\tau(E)$. However, we can extend $\pot_\tau(x)$ and $\bm{r}_\tau$ measurably by setting them to $+\infty$ and $\cl{B}_\rev{\tau}(x)$ on $\dom_\tau(E)^c$ respectively. 
\end{proof}

\begin{theorem}\label{thm:GloLo}
Let $\WBan$ be a reflexive, separable Banach space and let $\pot$ fulfill \cref{assm: pota}, then 
\begin{align*}
\mu_\tau \coloneqq (r_\tau)_\# \mu \in \argmin_{\muatt : W_\infty(\muatt,\mu)\leq \tau} \int \pot(x) d\muatt(x)
\end{align*}
for every measurable selection $r_\tau$ of $\bm{r}_\tau$ from \labelcref{eq:loMinSel}\rev{.} 
\end{theorem}
\begin{proof}
\cref{lem:LocMin} ensures the existence of measurable selectors of \labelcref{eq:loMinSel}. We take $\muatt$, such that $W_\infty(\mu,\muatt)\leq \tau$ and $\gamma \in \Gamma_0(\mu,\muatt)$, then by disintegration we get
\begin{align*}
\int \pot(x) d\muatt(x)=\int \pot(\x)d \gamma(\xx,\x)= \int \int \pot(\x) d \rho_{\xx} (\x)\, d \mu(\xx)
\end{align*}
with a Borel family of probability measures $\{ \rho_{\xx}\}_{\xx \in \WBan_1} \subset \Prob(\WBan)$ and $\supp(\rho_{\xx}) \subset \cl{B_\tau} (\xx) $. We further estimate, 
\begin{align*}
\int \int \pot (\x) d \rho_{\xx} (\x) d \mu(\xx) \geq \int \pot({r}_\tau(\xx)) d \mu(\xx)= \int \pot(\xx)\, d ({r}_\tau)_\# \mu(\xx),
\end{align*}
and since $\muatt$ was arbitrary, this concludes the proof.
\end{proof}
In order to proceed with the following lemma, we also need the assumption, that the potential is a $C^1$-perturbation of a convex function and is Lipschitz continuous.

\begin{mdframed}
\begin{assume}\label{assm: pot}
Let $\pot:\WBan \rightarrow (-\infty,+\infty]$ be a $C^1$-perturbation of a proper, convex lower semicontinuous function. Further, let the differentiable part $\pot^\text{d}$ be globally Lipschitz.
\end{assume}
\end{mdframed}
Then the relation between the slope of $\func$ and the slope of the potential $\pot$ is stated in the following theorem.

\begin{lemma}\label{lm:PotSlope}
Let $\pot:\WBan \rightarrow (-\infty,+\infty]$ fulfill \cref{assm: pota,assm: pot}. Then
\begin{align}\label{eq: slopW}
|\partial \func|(\mu)=\int_{\WBan} |\partial \pot|(x) d\mu(x)
\end{align}
\rev{and $|\partial \func|(\mu)$ is a strong upper gradient of $\func$.}
\end{lemma}

\begin{proof}

\rev{ Since $\frac{\pot(x)-\pot_\tau(x)}{\tau}\geq 0$ we can use Fatou's lemma to show}
\begin{align*}
\int_\WBan |\partial \pot|(x) d\mu(x)&=\int_\WBan \lim_{\tau \rightarrow 0} 
\frac{\pot(x)-\rev{\pot_\tau(x)}}{\tau} d\mu(x)%
\\&\leq
\liminf_{\tau \rightarrow 0} \int_\WBan \frac{\pot(x)-\rev{\pot_\tau(x)}}{\tau} d\mu(x)\\& \leq
\limsup_{\tau \rightarrow 0} \int_\WBan \frac{\pot(x)-\pot({r}_\tau(x))}{\tau} d\mu(x)
\\&
= \limsup_{\tau \rightarrow 0} \frac{\func(\mu) - \func(\mu_\tau)}{\tau}
= |\partial \func|(\mu),
\end{align*}
where in the last step, we employ \rev{\cref{lm: EqiSlope}}. 
This implies that when $\int_\WBan |\partial \pot|(x) d\mu(x)=+\infty$ then $|\partial \func|(\mu)=+\infty$. In the case $\int_\WBan |\partial \pot|(x) d\mu(x)<+\infty$, we use \cref{lm: EqiSlope} (for $\func$) and \cref{lm:Lim} (for $\pot$)  to calculate
\begin{align*}
    \int_\WBan |\partial \pot|(x) d\mu (x)&=\int_\WBan \lim_{\tau \rightarrow 0} \frac{\pot(x)-\rev{\pot_\tau(x)}}{\tau}d\mu(x) \\
    &=\lim_{\tau \rightarrow 0} \frac{\int_\WBan \pot(x)-\pot({r}_\tau(x)) d\mu(x)}{\tau}\\
    &=\lim_{\tau \rightarrow 0} \frac{\func(\mu)-\func(\mu_\tau)}{\tau}=|\partial \func|(\mu),
\end{align*}
where the dominated convergence theorem was used to draw the limit into the integral. For the upper bound we observe
\begin{align*}
|\partial \pot|(x)=&\limsup_{z\rightarrow x}\frac{(\potc(x)-\potc(z)+\potd(x)-\potd(z))^+}{\|x-z\|}\\ \rev{\geq}
    &\limsup_{z\rightarrow x} \frac{(\potc(x)-\potc(z))^+}{\|x-z\|} -\frac{|\potd(x)-\potd(z)|}{\|x-z\|}\\ \rev{\geq}
    &\limsup_{z\rightarrow x} \frac{(\potc(x)-\potc(z))^+}{\|x-z\|}-\Lip(\potd)=|\partial\potc|(x)-\Lip(\potd)
\end{align*}
and by \cite[Theorem 2.4.9]{ambrosio05}
$$ \sup_{\xx\neq \x}\frac{(\potc(\x)-\potc(\xx))^+}{\|\x-\xx\|} =|\rev{\partial}\potc|(\x).$$
Then we can give an upper bound by 
\begin{align*}
\frac{\pot(x)-\pot(r_\tau(x))}{\tau}&\leq \frac{\potc(x)-\potc(r_\tau(x))}{\|x-r_\tau(x)\|}+\frac{\potd(x)-\potd(r_\tau(x))}{\|x-r_\tau(x)\|} \\
&\leq \sup_{\xx\neq \x}\frac{(\potc(x)-\potc(\xx))^+}{\|x-\xx\|}  +\Lip(\potd)=|\partial \potc|(x)+\Lip(\potd)
\\&\leq |\partial \pot|(x)+2\Lip(\potd).
\end{align*}
\rev{
To prove that $|\partial \func|$ is a strong upper gradient let $\mu_t$ be an absolutely continuous curve in $\mathcal{W}_\infty(\WBan)$.
Since $\abs{\partial\func}(\mu)= \int_\WBan |\partial \pot|(x) d\mu(\x)$ and by \cref{pro:C1:iv} the slope $|\partial \pot|(x)$ is lower semicontinuous, it follows from \cite[Lemma 5.1.7.]{ambrosio05} that  $\abs{\partial\func}(\mu)$ is lower semicontinuous w.r.t. narrow convergence and in particular $t\mapsto|\partial \func|(\mu_t)$ is lower semicontinuous and thus Borel.
Assume that $\int_s^t \int_\WBan |\partial \pot|(x) d\mu_r(\x) |\mu'|(r) dr
= \int_s^t \abs{\partial\func}(\mu_r) |\mu'|(r)dr<+\infty$, otherwise \cref{eq: StUpGrad} holds trivially. 
We can estimate
\begin{align}\label{eq: UPGradE}\begin{split}
|\func(\mu_t)-\func(\mu_s)|&=\left|\int_\WBan \pot(x)d\mu_t(x)-\int_\WBan \pot(x)d\mu_s(x)\right|\\
&=\left|\int_\WBan \pot(\x)d{e_t}_\# \eta(\x)-\int_\WBan \pot(\x)d{e_s}_\# \eta(\x)\right|\\
&=\left|\int_{C(0,T;\WBan)} \pot(u(t))d\eta(u)-\int_{C(0,T;\WBan)} \pot(u(s))d \eta(u)\right|
\\&\leq \int_{C(0,T;\WBan)} 
|\pot(u(t))-\pot(u(s))| d\eta(u)
\\
&\overset{(i)}{\leq}\int_{C(0,T;\WBan)} \int_s^t | \partial \pot|(u(r))\, |u'|(r)dr d\eta(u)\\
&\overset{(ii)}{=}\int_s^t\int_{C(0,T;\WBan)}  | \partial \pot|(u(r))\, |u'|(r) d\eta(u)dr\\
&\overset{(iii)}{\leq} \int_s^t\int_{C(0,T;\WBan)}  | \partial \pot|(u(r)) d\eta(u) |\mu'|(r)dr\\
&= \int_s^t\int_{\WBan}  | \partial \pot|(x) d\mu_r |\mu'|(r)dr<+\infty,
\end{split}
\end{align}
where $(t,u)\mapsto | \partial\pot|(u(t))$ is $\Bar{\eta}$-measurable since it is lower semicontinuous on $[0,T]\times C(0,T,\Inp)$ and measurability of $\abs{u'}$ follows as in the proof of \cite[Thm. 7]{Lisini07} and \cref{thm: ECp}. For $(ii)$ we use the theorem of Fubini--Tonelli, while for $(i)$ we observe that $\eta$ from \cref{thm: curveDe} is concentrated on $\AC^\infty(0,T;\WBan)$ and $|\partial \pot|$ is a strong upper gradient (c.f. \cref{def: strGrad}) and for (iii) we use $|\mu'|(t)=\| |u'|(t)\|_{L_p(\eta)}$.
}
\end{proof}

The main result of this section now states, that $\infty$-curves of maximal slope on $W_\infty(\WBan)$ can be equivalently characterized, by the property that $\eta$-a.e. curve fulfills the differential inclusion w.r.t. the potential $E$ on the Banach space $\WBan$.
\rev{
\begin{theorem} \label{thm: WInfCurve}
Let $\func: W_\infty(\WBan)\rightarrow (-\infty,+\infty]$ be a potential energy with the potential $\pot$ satisfying \cref{assm: pota,assm: pot}, $\mu_t \in \dom(\func) $ for all $t\in[0,T]$ and $\mu \in \AC^\infty(0,T; \mathcal{W}_\infty)$ with $\eta$ from \cref{thm: curveDe}. Let further $\func\circ \mu$ be for a.e. $t\in[0,T]$ equal to a non-increasing map $\psi:[0,T] \rightarrow \R$.
Then the following statements are equivalent:
\begin{enumerate}[label=(\roman*)]
\item%
$
    |\mu'|(t)\leq 1 \text{ and } \psi'(t)\leq -|\partial \func|(\mu(t)) \text{ for a.e. } t\in (0,T).
$
\item For $\eta$-a.e. curve $u\in C(0,T;\WBan)$ it holds, that $E\circ u$ is for a.e. $t\in(0,T)$ equal to a non-increasing map $\psi_u: [0,T]\rightarrow \R$ and
$$ u'(t)\ \rev{\in}\ \partial \|\cdot\|_*(-\xi) \quad \forall \xi \in \partial^\circ \pot(u(t)) \not= \emptyset, \quad \text{for a.e. } t\in(0,T).$$
\end{enumerate} 
\end{theorem}

}

\begin{proof}
\rev{
\textbf{Step 1} $(i)\Longrightarrow (ii)$.\\
Because of \cref{rm:DisE} we know that $\mu_t$ satisfies the energy dissipation equality \cref{eq: EnergyDissip}. 
Making a similar estimate as in \cref{eq: UPGradE} we obtain
\begin{align}\label{eq: Equal}
\begin{split}
    \func(\mu_0)-\func(\mu_T)&=  \int_{C(0,T;\WBan)} E(u(0))-E(u(T)) d\eta(u)\\
   &\leq \int_{C(0,T;\WBan)}  \int_0^T |\partial E|(u(r)) |u'|(t) dr d\eta(u)\\
   &\leq \int_{C(0,T;\WBan)}  \int_0^T |\partial E|(u(r)) dr d\eta(u)\\
   &=\int_0^T |\partial\func|(\mu(r))dr=\func(\mu_0)-\func(\mu_T).
\end{split}
\end{align}
This implies that $|\partial E|(u(r))|u'|(t) \in L^1(0,T)$ for $\eta$-a.e. $u$.
Since $|\partial E|$ is a strong upper gradient $\psi_u:t\mapsto E(u(t))$ has to be absolutely continuous for $\eta$-a.e. $u$ and
\begin{align*}
    E(u(s))-E(u(t))\leq \int_s^t |\partial E|(u(r)) |u'|(r) dr \leq \int_s^t |\partial E|(u(r)) dr \quad \text{for }\eta\text{-a.e. }u  \text{ for all } 0 \leq s < t \leq T. 
\end{align*}
Equality in \cref{eq: Equal} can then only hold if for all $0\leq s < t \leq T$ we have
\begin{align*}
     E(u(s))-E(u(t))= \int_s^t |\partial E|(u(r)) |u'|(t) dr = \int_s^t |\partial E|(u(r)) dr \quad \text{for }\eta\text{-a.e. }u
\end{align*}
and thus $\psi_u \coloneqq E\circ u$ is a non-increasing map for $\eta$-a.e. $u$. \cref{rm:ChainR} and \cref{pro:C1:iii} imply that for every $\xi\in\partial^\circ\pot(u(t))$ we obtain
\begin{align*}
\langle \xi, u'(t)\rangle = 
(\pot\circ u)'(t) = -|\partial \pot|(u(t)) =-\norm{\xi}_*-\chara_{\cl{\unitb}}(u'(t)),
\end{align*}
where we use \cref{lem:timeesssup} and $ \||u'|(t)\|_{L^\infty(\eta)}=|\mu'|(t)\leq 1$ for a.e. $t\in(0,T)$ to infer that $\abs{u'}(t)\leq 1$ for a.e. $t\in(0,T)$. Using the equivalence of \cref{prop:youngiii} and \cref{prop:youngiv} yields
\begin{align*}
u'(t) \in \partial \norm{\cdot}_*(-\xi)
\end{align*}
for a.e. $t\in(0,T)$ and $\eta$-a.e. curve $u$.\\[1em]
\noindent\textbf{Step 2} $(ii)\Longrightarrow (i)$.\\
Due to \cref{rm:DisE}  $E\circ u$ is for $\eta$-a.e. curve $u\in C(0,T;\WBan)$ an absolutely continuous curve, and it satisfies the energy dissipation equality
\begin{align*}
E(u(t))-E(u(s))=\int_s^t -|\partial E|(u(r)) dr \quad \text{for }0\leq s\leq t\leq T.
\end{align*}
Therefore we obtain
\begin{align*}
\func(\mu_t) - \func(\mu_s) &= 
\int_{C(0,T;\WBan)} E(u(t))-E(u(s)) d\eta(u) = 
\int_{C(0,T;\WBan)}  \int_s^t -|\partial E|(u(r)) dr d\eta(u)\\ &= 
\int_s^t \int_{C(0,T;\WBan)}   -|\partial E|(u(r))  d\eta(u)dr = 
\int_s^t -\abs{\partial\func}(\mu_r) dr\leq 0,
\end{align*}
where the application of Fubini--Tonelli is justified due to the assumption $\mu_t\in\dom(\func)$ for all $t\in[0,T]$, which yields that $\abs{\func(\mu_t) - \func(\mu_s)} < \infty$ for all $s,t\in[0,T]$. By the Lebesgue differentiation theorem, we obtain
\begin{align*}
(\func\circ \mu)'(t) = -\abs{\partial\func}(\mu_t)
\end{align*}
for almost every $t\in(0,T)$. Furthermore, $\func \circ \mu$ is a non-increasing map and \cref{thm: curveDe} yields that
\begin{align*}
\abs{\mu'}(t) = \eta(u) -\esssup \abs{u'}(t) \leq 1
\end{align*}
for a.e. $t\in (0,T)$ and since $\mu\in\AC^\infty$ this yields that $\abs{\mu'}(t)\leq 1$ for a.e. $t\in(0,T)$.
}    
\end{proof}

\begin{remark}
In particular, those curves of maximal slope satisfy the continuity equation for the velocity field
$$ \velo_t(x)\coloneqq \int_{C(0,T;\WBan)} u'(t) d \Bar{\eta}_{x,t} \quad\text{for } \Bar{\mu} \text{-a.e. } (t,x) \in (0,T)\times\WBan.$$
If $\partial^\circ \pot(x)$ is unique, i.e., if $\|\cdot\|$ is strictly convex or $E(x)\in C^1(\WBan)$, then for $\Bar{\eta}_{x,t}$-a.e. $u\in C(0,T;\WBan)$ the derivatives $u'(t)$ lie in the closed and convex set $\partial \|\cdot\|_*(-\partial^\circ E(x))$. Thus $$\velo_t(x)\in \partial \|\cdot\|_*(-\partial^\circ E(x))\quad\text{for } \Bar{\mu} \text{-a.e. } (t,x) \in (0,T)\times\WBan.$$ 
\end{remark}

As the last \rev{result in this} section, we give an explicit setting, where the existence of curves of maximum slope is ensured. Here, we restrict ourselves to finite dimensions, mimicking \cref{cor: FinC1}.

\begin{corollary}[Existence in finite dimensions] Let $\WBan=(\mathbb{R}^d, \|\cdot\|)$ and $\pot:\R^d \to (-\infty,\infty]$ be a $C^1$-perturbation of a proper, lower semicontinuous, convex function. For every $\mu^0\in \dom(\func)$, there exists at least \rev{one} curve of maximal slope in the sense of \cref{def: CMS} with $\mu_0=\mu^0$. \rev{Further this curve satisfies the energy dissipation equality \cref{eq: EnergyDissip}.}
\end{corollary}
\begin{proof}
We simply check the conditions of \cref{thm: exist}. Choosing $\sigma$ to be the narrow topology, \cref{lem: Comp} \rev{guarantees} \cref{asm: 01}. To check \cref{asm: 02}, we know that for any sequence $\seq{\mu}$, $W_\infty(\mu^k,\mu^m)<\infty \quad \forall k,m\in\mathbb{N}$ implies that $\cup_{n} \supp(\mu^n)$ is bounded. Since we are in the finite dimensional case, we can now apply Prokhorov's Theorem to obtain relative compactness of the sequence.
\rev{We are left to check \cref{asm: lc,asm: 03b} for $\mathcal{E}$}:\\[5pt]
\textbf{\cref{asm: lc}}:\\[3pt]
Let $\mu^n\in\dom(\func)$ be a sequence converging in $W_\infty$ to $\mu$. This sequence has to be bounded in $W_\infty$ such that $\overline{\cup_n \supp(\mu^n)}$ is bounded.
Since $E$ is lower-semicontinuous 
$\dom(E)$ is closed and thus $\overline{\cup_n \supp(\mu^n)}\cap \dom(\pot)$ is compact and we obtain due to lower-semicontinuity 
\begin{align*}
\min_{x\in\overline{\cup_n \supp(\mu^n)}\cap \dom(\pot)}\pot(x)>-\infty
\end{align*}
\rev{Thus the negative part of $\pot(x)$ denoted by $\pot^-(x)$ is uniformly integrable with respect to $\{\mu^n\}_{n\in \mathbb{N}}$ and we can apply \cite[Lemma 5.1.7.]{ambrosio05} to obtain
$$\liminf_{n\rightarrow \infty} \int \pot(x)d\mu^n(x)\geq \int \pot(x) d\mu. $$}\\
\textbf{\cref{asm: 03b}:}\\[3pt]
Since $\mu$ has bounded domain the differentiable part $E^d$ satisfies a Lischitz condition and thus by \cref{lm:PotSlope}
$$|\partial \func|(\mu)=\int |\partial \pot|(x) d\mu $$
and by \cref{pro:C1} $|\partial \pot|(x)$ is lower semicontinuous and non-negative. Thus $|\partial \pot|(x)$ is uniformly integral, and we can apply \cite[Lemma 5.1.7.]{ambrosio05} to obtain
$$\liminf_{n\rightarrow \infty} \int |\partial \pot|(x)d\mu^n \geq \int |\partial \pot|(x)d\mu$$
for all $\mu^n$ converging narrowly to $\mu$.
\end{proof}

\section{Relation to adversarial attacks}
\label{ch: AdAt}

This section, explores the connection of the previous results to our initial motivation, adversarial attacks. As mentioned before, we now consider an energy defined as
\begin{align*}
\funccp(\x) := -\loss(\hyp(\x), y)
\end{align*}
for a classifier $\hyp$ and $\x\in\Inp, y\in\Oup$. The goal in \labelcref{eq: adver} is to maximize this function on the set $\cl{B_\budget}(\x^0)$, where $\x^0\in\Inp$ is the initial input. Roughly following the idea in the original paper proposing \labelcref{eq: FGSM}, we derive the scheme, via linearizing $\funccp$ around $\x^0$ and consider the linearized minimizing movement scheme in \cref{def: SeLiSch}. Assuming, that $\ell(\hyp(\cdot), y)$ is $C^1$, we consider
\begin{align*}
\funccps(\x; \xx)= -\loss(\hyp(\xx),y)-
\langle \nabla_\x \loss(\hyp(\xx),y),\x-\xx\rangle,
\end{align*}
where $\xx$ denotes the point of linearization. \cref{lem:semiexp} yields that the semi-implicit minimizing movement scheme in \cref{def: SeLiSch} can be expressed as 
\begin{align*}
\xs_{\tau}^{k+1} \in  \xs_{\tau}^{k}-\tau\partial \norm{\cdot}_*(D \funccp(\xs_{\tau}^{k})). 
\end{align*}
We note that this scheme can be understood as an explicit Euler discretization \cite{euler1794institutiones} of the differential inclusion in \cref{thm: GradFlowForm},
\begin{align}\label{eq: diffInc}
u'(t)\in \argmax_{\x\in\cl{B_1}} \langle \x, -D\funccp(u(t))\rangle = \partial \norm{\cdot}_*(D\funccp(u(t))),
\end{align}
which in turn is an equivalent characterization of $\infty$-curves of maximal slope. In this section, we consider the finite dimensional adversarial setting, i.e., the Banach space $(\Ban, \norm{\cdot}) = (\R^d, \norm{\cdot}_p)$. 
\begin{corollary}\label{cor:fgsm}
Given $\x^0\in\R^d$, the iteration
\begin{align*}
\xs^{k+1} =
\xs^k_\tau + \tau\
\sign(\nabla_x \funccp(\xs^k_\tau))\cdot \left(\frac{\abs{\nabla_x \funccp(\xs^k_\tau)}}{\norm{\nabla_x \funccp(\xs^k_\tau)}_q}\right)^{q-1}, \qquad\xs^0 = \x^0
\end{align*}
fulfills the semi-implicit minimizing movement scheme in \cref{def: SeLiSch} in the space $(\R^d, \norm{\cdot}_p)$ with $1/p + 1/q=1$. In this sense, \labelcref{eq: FGSM} is a one-step explicit Euler discretization of the differential inclusion \cref{eq: diffInc} with step size $\budget$.
\end{corollary}
\begin{remark}
We note that for $p\in\{1,\infty\}$ the expression in \cref{cor:fgsm} is to be understood in the sense of subdifferentials, as the following proof shows. However, the elements of the subdifferential we choose, can be understood as the limit cases of $p\to 1$ and $p\to\infty$, respectively.
\end{remark}
\begin{proof}
We choose $\Ban=\R^d$ with $\norm{\cdot} = \norm{\cdot}_{p}$. For $p=\infty$ we have that
\begin{align*}
\sign(\xi) \in  \partial \norm{\cdot}_{1}(\xi) =
\partial (\norm{\cdot}_{\infty})_*(\xi),
\end{align*}
for all $\xi\in\R^d$ and therefore, the following iteration fulfills the semi-implicit minimizing movement scheme,
\begin{align*}
\xs^{k+1}_\tau = \xs^{k}_\tau - \tau\sign(\nabla_\x \funccp(\xs^{k}_\tau)) = 
\xs^{k}_\tau + \budget\sign(\nabla_\x \loss(\hyp(\xs^{k}_\tau), y))\rev{,}
\end{align*}
and for $\budget=\tau$ the statement follows.
For $p=1$ we choose the following element of the subdifferential $g(\xi)$, with
\begin{align*}
g(\xi)_i := 
\#\{j:\abs{\xi_j} = \norm{\xi}_\infty\}^{-1} \cdot
\begin{cases}
\sign(\xi_i) &\text{if}  \abs{\xi_i}= \norm{\xi}_\infty,\\
0&\text{else},
\end{cases}
\end{align*}
and proceed as before. If we instead choose a finite $p\in(1,\infty)$, we obtain for $1/p + 1/q=1$,
\begin{align*}
\partial (\norm{\cdot}_p)_*(\xi) = \partial \norm{\cdot}_q(\xi) = 
\norm{\xi}_q^{1-q}
(\xi_1 \abs{\xi_1}^{q-2},\ldots, \xi_d\abs{\xi_d}^{q-2}) = 
\sign(\xi)\cdot \left(\frac{\abs{\xi}}{\norm{\xi}_q}\right)^{q-1},
\end{align*}
where the absolute value and the multiplication is to be understood entrywise. As above, this yields the statement also for $p\in(1,\infty)$. 
\end{proof}

\subsection{Convergence of IFGSM to curves of maximal slope}
Our main goal is to derive a convergence result of \labelcref{eq: IFGSM} for $\tau\to 0$. As mentioned before, \cref{lem:semiexp} yields an iteration, which can be expressed as normalized gradient descent in the finite dimensional case. The main obstacle that prohibits us from directly applying the convergence result for semi-implicit schemes (see \cref{thm:semilincvg}), is the budget constraint, $u'(t)\in \cl{B_\budget^p}(\x^0)$ for all $t$. Here and in the following, we now assume that the norm exponent of the underlying space and of the budget constraint norm are the same. In \labelcref{eq: IFGSM} this is enforced via a projection onto this set in each iteration. An easy way to circumvent this issue, is to only consider the iteration up to the step, where it would leave the constraint set. In this case, the projection never has any effect and we essentially consider signed gradient descent. Intuitively, the Lipschitz condition $\norm{u'(t)}\leq 1$ allows us to control how far $u(t)$ is away from $\x^0$. This mimicked in the discrete scheme, where we know that
\begin{align*}
\norm{\xs^{i}_\tau - \x^0}\leq \sum_{k=0}^{n-1} \norm{\xs^{k+1}_\tau - \xs^{k}_\tau}\leq n\tau = T,
\end{align*}
for every $i=0,\ldots, n$. Therefore, we can choose $T=\budget$ to ensure that $\xs^{i}_\tau\in \cl{B_\budget^p}(\x^0)$ for every $i=0,\ldots,n$. This yields the following result.
\begin{corollary}\label{cor:budgettime}
We consider the space $(\Ban, \norm{\cdot}) = (\R^d, \norm{\cdot}_{p})$ for $p\in[1,\infty]$ and $\funccp:\R^d\to\R$, a continuously differentiable energy, with a Lipschitz continuous gradient. Then for $T=\budget$, there exists a $\infty$-curve of maximal slope $u:[0,T]\to \R^d$, with respect to $\funccp$, and a subsequence of $\tau_n:=T/n$ such that
\begin{align*}
\norm{u_{\IFG, \tau_{n_i}}^{\lceil t/\tau_{n_i} \rceil} - u(t)}\xrightarrow{i\to\infty} 0\qquad\text{ for all } t\in[0,T].
\end{align*}
\end{corollary}
\begin{proof}
From \cref{lem:semiexp} and the calculation in the proof of \cref{cor:fgsm} we know that the iterates of \labelcref{eq: IFGSM} fulfill the linearized minimizing movement scheme in \cref{def: SeLiSch}. Here, we used that for $T=\budget$, the iterates do not leave the set $\cl{B_\budget^p}(\x^0)$ and therefore the projection has no effect.
\cref{asm: LipDi} is stated as an assumption of this corollary and \rev{\cref{rem:LoSe}} yields that \cref{asm: LoSe} holds true. Furthermore, using \rev{\cref{pro:C1}} , we know that \cref{asm: 01,asm: 02,asm: lc,asm: 03b} are fulfilled and therefore, we can apply \cref{thm:semilincvg} to obtain the desired result.
\end{proof}
Above, we only consider convergence up to a subsequence. While the convergence of the whole sequence for \labelcref{eq: IFGSM}, is left unanswered in this work, we note that at least for $p\in\{1,\infty\}$, this cannot be expected, since in this case $\infty$-curves of maximal slope lack uniqueness, even in the simple finite dimensional case, as the following example shows. 
\begin{example}[Non uniqueness for $p\in\{1,\infty\}$]
Let $(\Ban,\norm{\cdot})=(\R^2,\norm{\cdot}_\infty)$ and consider the energy be given by 
$$\funccp:(x_1,x_2)\in \R^2\mapsto x_1\in \R $$
then both $u_1(t)=(-t,0)$ and $u_2(t)=(-t,-t)$ are $\infty$-curves of maximal slope on $[0,T],T>0$, with $u_1(0)=u_2(0)$ since
$$u_1'(t)=(-1,0)\in -\partial\norm{\cdot}_1(1,0)=-\partial\norm{\cdot}_1(\nabla \funccp(u_1(t)))$$
and
$$u_2'(t)=(-1,-1)\in-\partial\norm{\cdot}_1(1,0)= -\partial\norm{\cdot}_1(\nabla \funccp(u_2(t))).$$
In two dimensions for $p=1$, we can simply rotate the above setup to deduce the same non-uniqueness. Namely for $\funccp(\x_1,\x_2)=\x_1+\x_2$, we have that $u_1(t) = (-t,0)$ fulfills
$$u_1'(t)=(-1,0)\in -\partial\norm{\cdot}_\infty(1,1)=-\partial\norm{\cdot}_\infty(\nabla \funccp(u_1(t))) $$
and also $u_2(t) = \frac{1}{2}(-t,-t)$ fulfills
$$u_2'(t)=\frac{1}{2}(-1,-1)\in -\partial\norm{\cdot}_\infty(1,1)=-\partial\norm{\cdot}_\infty(\nabla \funccp(u_2(t))). $$
\end{example}

In \cref{cor:budgettime}, we only allow the iteration to run until it hits the boundary. However, in practice, it is more common to also iterate beyond the time $\budget$. In order to incorporate the budget constraint in this case, we modify the energy to
\begin{align*}
\funccp(\x) := -\loss(\hyp(\x), y) +
\chara_{\cl{B_\budget^p}(\x^0)}(\x),
\end{align*}
which yields the semi-implicit energy
\begin{align*}
\funccps(\x;\xx) = -\loss(\hyp(\xx), y) -\langle \nabla_x\loss(\hyp(\rev{\xx}),y), \x-\xx\rangle + \chara_{\cl{B_\budget^p}(\x^0)}(\x). 
\end{align*}
In order to show that \labelcref{eq: IFGSM} corresponds to the minimizing movement scheme, we need to 
show that first minimizing on $\cl{B_\tau^p}(\x)$ and then projecting to $\cl{B_\budget^p}(\x^0)$ is equivalent to directly minimizing on $\cl{B_\budget^p}(\x^0)\cap \cl{B_\tau^p}(\x)$. Here, we restrict ourselves to the case $p=\infty$, which corresponds to the standard case of \labelcref{eq: IFGSM} as proposed in \cite{goodfellow2014explaining}. For $p\neq \infty$ a more refined analysis would be required, c.f. \cref{fig:projtrick}. In the following lemma, we use the projection defined componentwise as
\begin{align*}
\rev{\Clip_{\x^0,\budget} (\x)_j}:=
\Pi_{\cl{B_\budget^\infty}(\x^0)}(\x)_j= \rev{\x^0_j + }
\max\{\min\{\x_j-\rev{\x^0_j},\budget\}, -\budget\}.
\end{align*}
The proof relies on the basic intuition in the original paper \cite{goodfellow2014explaining} that maximizing the linearized energy on a hyper-cube, is a linear program \cite{sierksma2015linear,fourier1827histoire} with a solution being attained in a corner. We also note that this does not directly work for other choices of budget constraints, see \cref{fig:projtrick}
\begin{lemma}\label{lem:projtrick}
For $\x\in \cl{B^\infty_\budget}(\x^0)$ and $\tau>0$ it holds that
\begin{align*}
\Clip_{\x^0,\budget}(\x+\tau \sign(\nabla_\x \ell(\hyp(\x),y)) )\in \argmin_{\xatt\in \cl{B_\tau^\infty}(\x)} \funccps(\xatt; \x).
\end{align*}
\end{lemma}
\begin{proof}
Without loss of generality we assume that $x^0=0$. Let $\xi:= -\nabla_\x\loss(\hyp(\x), y)$, then we know that $\x^\text{d}=\x-\tau \sign(\xi)$ is a minimizer of $\xatt\mapsto \langle \xi,\xatt\rangle$ on $\cl{B^\infty_\tau}(\x)$. Furthermore, we define $\delta\in\R^n$ as
\begin{align*}
\delta_i := -\sign(\x^\text{d}_i) \max\{\abs{\x^\text{d}_i} - \budget, 0\}, 
\end{align*}
i.e., we have that $\Clip_{0,\budget}(\x^\text{d})=\x^\text{d} +\delta$. The important fact, where the choice of budget constraint matters, is that $\xatt -\delta\in \cl{B_\tau^\infty}(\x)$ for all $\xatt\in \cl{B_\tau^\infty}(\x)\cap \cl{B_\budget^\infty}(0)$, since we have
\begin{gather*}
\max\{-\budget, \x_i-\tau\} \leq \xatt_i \leq \min\{\budget, \x_i+\tau\}\\
\Rightarrow 
\begin{cases}
\delta_i = 0&:\qquad \abs{\xatt_i -\delta_i - \x_i} = \abs{\xatt_i - \x_i}\leq \tau\\
\delta_i <0&:\qquad  
\x_i\leq \budget < \x_i^\text{d} \leq \x_i+\tau \\
&\phantom{:}\qquad\Rightarrow 
\abs{\xatt_i - \delta_i-\x_i} \leq  \abs{\budget + \x^\text{d}_i - \budget - \x_i} \leq \tau\\
\delta_i >0&:\qquad \abs{\xatt_i -\delta_i - \x_i} \leq \tau,\quad \text{ analogously  to the case above.}
\end{cases}
\end{gather*}
Now assume, that there exists $\xatt\in \cl{B_\budget^\infty}(0)\cap \cl{B_\tau^\infty}(\x)$ such that $\langle \xi, \xatt\rangle  < \langle \xi, \x^\text{d} + \delta\rangle$. Then we infer that
\begin{align*}
\langle \xi, \xatt - \delta\rangle < \langle \xi, \x^\text{d} + \delta\rangle - \langle \xi, \delta\rangle = \langle \xi, \x^\text{d}\rangle
\end{align*}
and therefore $\x^\text{d}$ is not a minimizer on $\cl{B_\tau^\infty}(\x)$, which is a contradiction. Therefore, we have that
\begin{align*}
\x^\text{d} + \delta = \Clip_{0,\budget}(\x^{\text{d}})=
\Clip_{0,\budget}(\x + \tau\, \sign(\xi)) \in \argmin_{\xatt\rev{\in} \cl{B_\tau^\infty}(\x)\cap \cl{B_\budget^\infty}(0)} \langle \xi,\xatt\rangle = 
\argmin_{\xatt\rev{\in} \cl{B_\tau^\infty}(\x)} \funccps(\xatt;\x)\rev{.}
\end{align*}
\end{proof}
\begin{figure}
\begin{subfigure}[t]{.3\textwidth}
\centering%
\includegraphics[width=\textwidth]{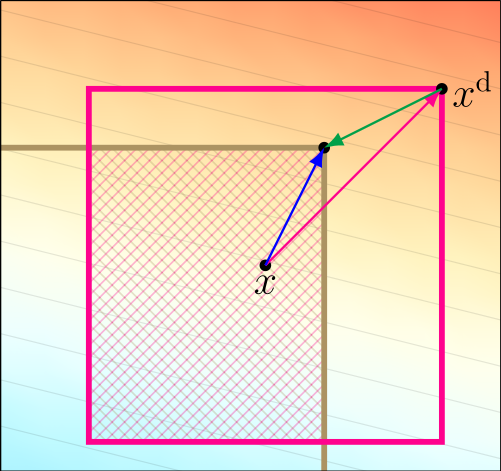}
\captionsetup{width=.9\linewidth}
\caption{Underlying space and budget constraint with $p=\infty$.}
\end{subfigure}
\hfill%
\begin{subfigure}[t]{.3\textwidth}
\centering%
\includegraphics[width=\textwidth]{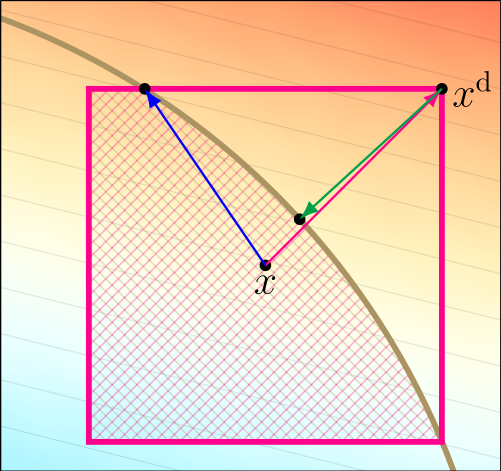}
\captionsetup{width=.9\linewidth}
\caption{Norm mismatch: underlying space with $p=\infty$ and budget constraint with $q=2$.}
\end{subfigure}
\hfill%
\begin{subfigure}[t]{.3\textwidth}
\centering%
\includegraphics[width=\textwidth]{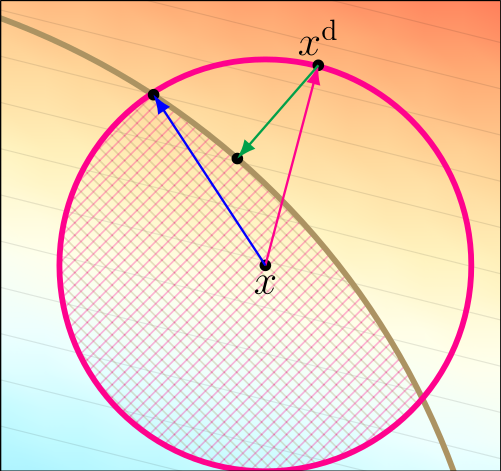}
\captionsetup{width=.9\linewidth}
\caption{Underlying space and budget constraint with $p=2$.}
\end{subfigure}
\caption{Visualization of one \labelcref{eq: IFGSM} step, employing different norm constraints and underlying norms. The beige line marks the boundary of $B_\budget^p(\x^0)$, the pink line the boundary of $B_\tau^q(\x)$  and the intersection $\cl{B_\budget^p}(\x^0) \cap \cl{B_\tau^q}(\x)$ is hatched. For the case $p=q=\infty$ minimizing a linear function on the intersection (blue arrow) is equivalent to first minimizing on $\cl{B_\tau^\infty}(\x)$ (pink arrow) and then projecting back to the intersection (green arrow). This is not true for $p=2$. Therefore, we need to choose the appropriate projection in \cref{lem:projtrick}.}\label{fig:projtrick}
\end{figure}

This result shows, that when we choose $p=\infty$ for the budget constraint \labelcref{eq: IFGSM} again fulfills the semi-implicit minimizing movement scheme, beyond the time restriction in \cref{cor:budgettime}.

\begin{theorem}\label{thm:ifgsm}
We consider the space $(\Ban, \norm{\cdot}) = (\R^d, \norm{\cdot}_{\infty})$, the energy \rev{$\funccp = \funccpd + \chara_{\cl{B_\budget^\infty}(\x^0)}$, with a continuously differentiable part $\funccpd$, which has a Lipschitz continuous gradient}. Then for $T>0$, there exists a $\infty$-curve of maximal slope $u:[0,T]\to \R^d$, with respect to $\funccp$, and a subsequence of $\tau_n:=T/n$ such that
\begin{align*}
\norm{\xifg_{ \tau_{n_i}}^{\lceil t/\tau_{n_i} \rceil} - u(t)}\xrightarrow{i\to\infty} 0\qquad\text{ for all } t\in[0,T].
\end{align*}
\end{theorem}

\begin{proof}
Since \cref{lem:projtrick} yields that \labelcref{eq: IFGSM} fulfills the semi-implicit minimizing movement scheme, we can proceed similarly as in the proof of \cref{cor:budgettime}. We note the all the necessary assumptions are fulfilled, since the \rev{indicator} function $\chara_{\cl{B_\budget^\infty}(\x^0)}$ is lower semicontinuous.
\end{proof}

\subsection{Adversarial training and distributional adversaries}\label{ch: DisAd}
As before, we assume that the underlying spaces are finite dimensional, i.e., $\Inp=\R^d, \Oup=\R^m$ with norms $\norm{\cdot}_\Inp, \norm{\cdot}_\Oup$ and $\mathcal{P}(\WBan\times \Oup)$ denotes the space of Borel \rev{p}robability measures.
We consider the adversarial training task, as proposed in \cite{kurakin2016adversarial,goodfellow2014explaining},
\begin{align}\label{eq: AdTrain}
\inf_{\hyp\in\Hyp}\int \sup_{\xatt\in\cl{B_\budget}(\x^0)} \loss(\hyp(\xatt),y)d\mu(x,y),
\end{align}
where $\mu \in \mathcal{P}(\Inp\times \Oup)$ denotes the data distribution and $\rev{\ell(h(\cdot),y)}\in C^1(\WBan\times \Oup)$. This interpretation of adversarial learning in the distributional setting has sparked a lot of interest in recent years, see e.g., \cite{staib2017distributionally, zheng2019distributionally, pydi2020adversarial, bungert2023begins,bungert2023geometry,pydi2021many,sinha2017certifying,mehrabi2021fundamental}.
In order to rewrite this task as a distributionally robust optimization problem, we equip $\mathcal{P}(\Inp\times \Oup)$ with a suitable optimal transport distance 
\begin{align*}
D(\mu,\Tilde{\mu})\coloneqq \inf_{\gamma\in \Gamma(\mu,\Tilde{\mu} )} \gamma-\esssup\ c(x,y,\xatt,\Tilde{y}),
\end{align*}
where
\begin{align}\label{eq:extdist}
c(x,y,\xatt,\Tilde{y})\coloneqq \begin{cases}
\|x-\xatt\|_\WBan &\text{if } y=\Tilde{y},\\
+\infty &\text{if } y\neq \Tilde{y},
\end{cases}
\end{align}
and $\Gamma(\mu,\Tilde{\mu})$ denotes the set of transport plans between $\mu$ and $\Tilde{\mu}$. Notably, the extended distance $c$ is not the one naturally generated by the norms of the underlying Banach spaces $\WBan$ and $\Oup$. Nonetheless, $c$ is compatible with respect to $\|\cdot\|_\WBan+\|\cdot\|_\Oup$ in the sense that 
\begin{gather*}
    \liminf_{n\rightarrow\infty} c(x_n,y_n,\tilde{x}_n,\tilde{y}_n)\geq c(x,y,\tilde{x},\tilde{y}), 
    \\ \forall (x,y),(\tilde{x},\tilde{y})\in (\WBan\times \Oup) : (x_n,y_n)\rightarrow (y,x) ,
    (\tilde{x}_n,\tilde{y}_n)\rightarrow (\tilde{x},\tilde{y}) \text{ w.r.t. }  \|\cdot\|_\WBan+\|\cdot\|_\Oup,
\end{gather*}
compare \cite[Eq. (1)]{lisini2014absolutely}. This ensures that, as we equip $\mathcal{P}(\WBan\times \Oup)$ with $D$, it is a well-defined extended distance, see \cite[Section 2.6.]{lisini2014absolutely}. The cost functional $c$ was similarly employed in \cite{bungert2023geometry,bui2022unified}, furthermore a similar setup was considered in \cite{staib2017distributionally}.

\begin{remark}\label{rem:marginal}
Assume that $\gamma\in\Gamma(\mu,\muatt)$ is a coupling, i.e., $\gamma\in \mathcal{P}(\mathcal{Z}\times\mathcal{Z})$, where $\mathcal{Z}=\Inp\times\Oup$, with $\gamma-\esssup c(\x,y,\xatt,\tilde{y}) < \infty$. Then we have that for every measurable set $A\subset\Oup$,
\begin{align*}
\gamma(\Inp\times A\times\mathcal{Z}) = 
\gamma(\Inp\times A\times\Inp\times A) =
\gamma(\mathcal{Z}\times\Inp\times A),
\end{align*}
which we see by contradiction: assume there exists a measurable set $A\subset\Oup$ s.t., for $B:=\Inp\times A\times\Inp\times(\Oup\setminus A)$ we have $\gamma(B) > 0$. Then we know that $c(\x,y,\xatt,\tilde{y}) = +\infty$ for all $(\x,y,\xatt,\tilde{y})\in B$ and since $\gamma(B) >0$ this yields that 
\begin{align*}
\gamma-\esssup c(\x,y,\xatt,\tilde{y}) \geq \gamma-\esssup_{B} c(\x,y,\xatt,\tilde{y}) = +\infty.
\end{align*}
The other identity can be proven analogously.
Therefore, if $D(\mu,\muatt)<\infty$ we know that there exists a coupling $\gamma$ fulfilling the above assumption and thus for every measurable set $A\subset\Oup$ we obtain
\begin{align*}
\mu(\Inp\times A) = 
\int_{\Inp\times A \times \mathcal{Z}} d\gamma = 
\int_{\mathcal{Z}\times \Inp\times A} d\gamma
= 
\muatt(\Inp\times A).
\end{align*}
If we now consider a disintegration of $\mu$ and $\muatt$ along the $\Inp$-axis, i.e., we obtain $d\mu = d\mu_y d\nu(y), d\muatt = d\muatt_y d\tilde{\nu}(y)$, with 
\begin{align*}
\nu(A) = \mu((\pi^y)^{-1}(A)) = \mu(\Inp \times A) = \muatt(\Inp \times A) =  \muatt((\pi^y)^{-1}(A)) = \tilde{\nu}(A)
\end{align*}
for every measurable $A\subset\Oup$, where $\pi^y(x,y):= y$ is the projection onto the $\Oup$-component.
\end{remark}

The transport distance $D$ behaves like the $\infty$-Wasserstein distance in the $\Inp$-direction (compare section \cref{sec: Wass}) and \rev{penalizes} movement of mass into the $\Oup$-direction, such that no movement in $\Oup$ can occur when $D(\mu,\Tilde{\mu})$ is finite (see \cref{rem:marginal}). Thus, all calculations done in \cref{sec: Wass} apply with minor adaptation to this case. We only state corresponding lemmas and theorems, while adapted proofs can be found in \cref{sc:ApDiAd}. The first property we prove in this section, is that the adversarial training problem \eqref{eq: AdTrain} is equivalent to the distributional robust optimization problem, \labelcref{eq:DRO}.
Note, that now we need to consider a potential defined on the space $\Inp\times\Oup$, namely $E(\x,y)\coloneqq-\loss(\hyp(\x), y)$, where the label $y$ is now also a variable argument.
\begin{corollary}\label{cor:maxsup}
It holds that
\begin{align}\label{eq: DisMax}
\int\max_{\xatt\in \cl{B_\budget}(\xatt)}  \loss(\hyp(\x),y) d\mu(x,y)= \max_{\tilde{\mu}:D(\tilde{\mu},\mu)\leq \budget}\int \loss(\hyp(\x),y) d \tilde{\mu}(\x,y)
\end{align}
where the maximizing argument is given by $\mu_{\max}=(r_\budget)_{\#}\mu$, with $r_\budget:\Inp\times\Oup\to\Inp$ being a $\mathcal{B}(\Inp\times\Oup)$-measurable selector from \cref{lm: DMin}
\end{corollary}
\begin{proof}
We employ the $\mathcal{B}(\Inp\times\Oup)$-measurable selector $r_\budget$, from \cref{lm: DMin} and compute
\begin{align*}
\int \max_{\xatt\in\cl{B_\budget}(\x)}  \loss(\hyp(\xatt),y)\, d\mu(\x,y)
&=
\int
\max_{(\xatt, \tilde{y}): c(\x,y,\xatt,\tilde{y}) \leq \budget}
-\pot(\xatt,\tilde{y})\, d\mu(\x,y)
=
\int -\pot(r_\budget(x,y))\, d\mu(\x,y)\\ 
&= 
-\int \pot(x,y) d(r_{\budget})_\#\mu(\x, y)\overset{(i)}{=} -\min_{\muatt:D(\mu,\muatt)\leq\budget} \int \pot(x,y)\, d\muatt(\x,y)\\
&=
\max_{\muatt:D(\muatt,\mu)\leq \budget}\int \loss(\hyp(\x),y) d \muatt(\x,y),
\end{align*}
where in $(i)$ we employ \labelcref{eq:PushMax}.
\end{proof}
\rev{%
\begin{remark}
In other works considering distributional adversarial attacks, for example \cite{pydi2020adversarial,pydi2021many} the well-definedness of the expressions in \cref{cor:maxsup} is not always ensured. In \cite{bungert2023geometry} this was resolved by considering open balls for the budget constraint. However, due to our assumption that $\ell(h(\cdot), y)\in C^1(\Inp\times\Oup)$, we do not encounter similar measurability issues, as shown in \cite{meunier2021mixed}.
\end{remark}
}
For the main result in this section, we now consider the energy defined via the potential defined on $\Inp\times\Oup$, i.e.,
\begin{align*}
\ffunc(\mu)\coloneqq
\int \pot(\x,y) d\mu(\x,y) = \int -\loss(\hyp(x),y) d\mu(x,y),
\end{align*}
where the underlying \rev{extended} metric space is chosen as \rev{$\mathcal{D}=(\mathcal{P}_\infty(\Inp\times\Oup), c)$}, with \rev{$\mathcal{P}_\infty(\Inp\times\Oup)$} denoting the subset of Borel probability measures with bounded support in $\WBan$- and $\Oup$-direction. 
\begin{remark}
    \cref{thm: curveDe} also holds for extended distances, i.e., distances which take values in $[0,+\infty]$,  compare \cite[Thm. 3.1]{lisini2014absolutely}. The distance $c(\cdot,\cdot)$ introduced in \cref{eq:extdist} is such an extended distance. 
For this particular choice of extended distance, the measure $$\eta\in\mathcal{P}(C(0,T; \left((\Inp\times \Oup),\|\cdot\|_\WBan+\|\cdot\|_\Oup\right))$$ is concentrated on $ AC^\infty\left(0,T;\left((\WBan\times \Oup),c\right)\right)=AC^\infty(0,T;(\WBan,\|\cdot\|_\WBan)\times \Oup\rev{)}$. Notice that the continuous curves are continuous w.r.t. $\|\cdot\|_\WBan+\|\cdot\|_\Oup$ while absolute continuity is w.r.t. $c(\cdot,\cdot)$(compare \cite[Section 2.3.]{lisini2014absolutely}).
\end{remark}

The theorem below is a variant of \cref{thm: WInfCurve} for the adversarial setting. Namely, we show, that $\infty$-curves of maximal slope that are used to solve \labelcref{eq:DRO} can be characterized by employing a representing measure $\eta$ on $C(0,T;\Inp\times\Oup)$, where  $\eta$-a.e. curve fulfills the differential inclusion w.r.t. the potential $E$. Here, we enforce the condition $D(\mu,\muatt)\leq \budget$, by only considering the evolution until time $T=\epsilon$.

\begin{theorem} \label{thm: WInfCurveAdv}
For $T=\budget$, let $\mu \in \AC^\infty(0,T\rev{;} \mathcal{D})$ with $\eta$ from \cref{thm: curveDe}.  \rev{Let further $\ffunc\circ \mu$ be for a.e. $t\in[0,T]$ equal to a non-increasing map $\psi:[0,T] \rightarrow \R$.}
Then the following statements are equivalent:
\begin{enumerate}[label=(\roman*)]
\item%
\rev{$
    |\mu'|(t)\leq 1 \text{ and } \psi'(t)\leq -|\partial \ffunc|(u(t)) \text{ for a.e. } t\in (0,T).
$}
\item For $\eta$-a.e. curve $u\in C(0,T;\WBan\times \Oup)$ it holds, that $E\circ u$ is for a.e. $t\in(0,T)$ equal to a non-increasing map $\psi_u:\rev{[0,T]\rightarrow \R}$ and
$$ u'(t)\ \rev{\in} \  (\partial \|\cdot\|_{\WBan^*}(-\nabla_x E(u(t))),0) , \quad \text{for a.e. } t\in(0,T).$$
\end{enumerate} 
\end{theorem}

\begin{proof}\phantom{.}\\
\rev{
\textbf{Step 1} $(i)\Longrightarrow (ii)$.\\
By \cref{lm: GSlop} we know that $|\partial\ffunc|$ is a strong upper gradient such that by \cref{rm:DisE} $\mu_t$ satisfies the energy dissipation equality \cref{eq: EnergyDissip}. Similar to \cref{thm: WInfCurve} we estimate
\begin{align}
\begin{split}
    \ffunc(\mu_0)-\ffunc(\mu_T)&=  \int_{C(0,T;\WBan)} E(u(0))-E(u(T)) d\eta(u)\\
   &\leq \int_{C(0,T;\WBan)}  \int_0^T \|\nabla_x E(u(r))\|_{\WBan^*} |u'|(t) dr d\eta(u)\\
   &\leq \int_{C(0,T;\WBan)}  \int_0^T \|\nabla_x E(u(r))\|_{\WBan^*} dr d\eta(u)\\
   &=\int_0^T |\partial\ffunc|(\mu(r))dr=\ffunc(\mu_0)-\ffunc(\mu_T).
\end{split}
\end{align}
and observe that this equality can only hold if for $\eta$.a.e. $u$ 
\begin{align*}
     E(u(s))-E(u(t))= \int_s^t \|\nabla_x E(u(r))\|_{\WBan^*} |u'|(t) dr = \int_s^t \|\nabla_x E(u(r))\|_{\WBan^*} dr \quad \text{ for all } 0 \leq s < t \leq T.
\end{align*}
and thus $\psi_u \coloneqq E\circ u$ is a non-increasing absolutely continuous map for $\eta$-a.e. $u$. 
}
We use \cref{lem:timeesssup} and $ \||u'|(t)\|_{L^\infty(\eta)}=|\mu'|(t)\leq 1$ for a.e. $t\in(0,T)$ to infer that for $\eta$-a.e. curve $u\in C(0,T\rev{;}\WBan\times \Oup)$ it holds, 
$$ \abs{u'}(t)\leq 1 \quad \text{for a.e. } t\in(0,T).$$
Denoting by $(u'(t))_x$ and $(u'(t))_y$ the to $\WBan$ and $\Oup$ corresponding parts of the derivative $u'(t)$ \rev{and keeping \cref{rm:ChainR} in mind} we obtain for $\eta$-a.e. curve $u\in C(0,T\rev{;}\WBan\times \Oup)$
\begin{align*}
\langle \nabla_x E(x,y),(u'(t))_x\rangle&=\langle \nabla E(x,y), u'(t)\rangle\\
&= (\pot\circ u)'(t) \\
&= -\|\nabla_x E(u(r))\|_{\WBan^*} =-\norm{\nabla_x E(u(r))}_{\WBan^*}-\chara_{\cl{\unitb}}((u'(t))_x)
\end{align*}
for a.e. $t\in(0,T)$.
 Using the equivalence of \cref{prop:youngiii} and \cref{prop:youngiv}
we obtain
\begin{align*}
u'(t)\ \rev{\in} \  (\partial \|\cdot\|_{\WBan^*}(-\nabla_x E(x,y)),0) \quad \text{for a.e. } t\in(0,T)
\end{align*}
for $\eta$-a.e. curve $u$.\\[1em]
\noindent%
\textbf{Step 2} $(ii)\Longrightarrow (i)$.\\
For $\eta$-a.e. $u\in C(0,T;\WBan \times \Oup)$
we know by \cref{rm:DisE} that  \rev{that the energy dissipation equality 
\begin{align*}
    E(u(s))-E(u(t)) = \int_s^t \|\nabla_x E(u(r))\|_{\WBan^*} dr \quad \text{ for all } 0 \leq s < t \leq T.
\end{align*}
holds.} In particular, it is absolutely continuous such that \cref{rm:RadDer} applies.
 We calculate,
\begin{align*}
\ffunc(\mu_t) &- \ffunc(\mu_s) = 
\int_{C(0,T\rev{;}\WBan)} E(u(t))-E(u(s)) d\eta(u) = 
\int_{C(0,T\rev{;}\WBan)}  \int_s^t (E\circ u)'(r) dr d\eta(u)\\ &= 
\int_{C(0,T\rev{;}\WBan)}  \int_s^t \langle \nabla_x E(u(r)),(u'(r))_x\rangle dr d\eta(u)\overset{(i)}{=}\int_{C(0,T\rev{;}\WBan)}  \int_s^t -\|\nabla_xE(u(r))\|_{\WBan^*} dr d\eta(u)\\
&=\int_{C(0,T\rev{;}\WBan)}  \int_s^t -\|\nabla_xE(u(r))\|_{\WBan^*}  d\eta(u)dr\overset{(ii)}{=}\int_s^t -|\partial \ffunc|(\mu_r) dr,
\end{align*}
Where for (i) we use the equivalence of \cref{prop:youngiii} and \cref{prop:youngiv}, while for (ii) we use \cref{lm: GSlop}. This implies $\ffunc \circ \mu_t$ is monotone non-increasing and $(\ffunc \circ \mu)'(t)\leq -|\partial \ffunc|(\mu_t)$ for a.e. $t\in(0,T)$. Further, by \cref{thm: curveDe} we have
\begin{align*}
|\mu'|(t)=\eta(u)-\esssup |u'|(t) =\eta(u)-\esssup \|u'(t)\|_\WBan \leq 1,   
\end{align*}
since all elements in $\partial \|\cdot\|_{\WBan^*}(-\nabla_x E(x,y))$ have norm smaller than $1$.

\end{proof}
\section{Conclusion and outlook}
In this work, we considered the limit case $p\to\infty$ of the \rev{well-known} $p$-curves of maximum slope, which yield a versatile gradient flow framework in metric spaces, \cite{ambrosio05}. In the abstract setting, we proved existence by employing the minimizing movement scheme, adapted to the case $p=\infty$. Assuming that the underlying space is Banach, we were able to characterize $\infty$-curves of maximum slope via differential inclusions. Furthermore, we also demonstrated the convergence of a semi-implicit scheme to the continuum flow. This insight constitutes the interface to the field of adversarial attacks. Namely, we showed that the \rev{well-known} fast gradient sign method, and its iterative variant, correspond to the semi-implicit scheme and therefore converge to the flow, when sending the step size to zero. More generally, this result holds true for a whole class of normalized gradient descent algorithms. Furthermore, we also considered Wasserstein gradient flows, where we first used the theory developed in \cite{lisini2014absolutely} to derive an alternative characterization of absolutely continuous curves via \rev{the continuity equation}. As our main result in this section, we prove that being an $\infty$-curve of maximal slope is equivalent to the existence of a representing measure on the space of continuous curves, where almost every curve, fulfills a differential inclusion on the underlying Banach space. This finally allowed us to generate distributional adversaries, in an adapted $\infty$-Wasserstein distance, via curves of maximum slope.
Similar to \cref{ch: AdAt} we could also consider the energy
\begin{align*}
    \ffunc(\mu)\coloneqq\int_{\WBan\times \Oup} E(x,y) d\mu+\chi_{B^D_\epsilon(\mu_0)}(\mu) 
\end{align*}
to generate distributional adversarial attacks.
We strongly suspect that corresponding $\infty$-curves of maximal slope in $\mathcal{D}$ would take the following form:
Let $\mu\in AC^\infty(0,T;\WBan)$ be a $\infty$-curve of maximal slope and $\eta$ its corresponding probability measure over the space $C(0,T;\WBan\times \Oup)$, then for $\eta$-a.e. $u\in C(0,T;\WBan\times \Oup)$
\begin{align*}
    u'(t)\ \rev{\in}\  (\partial \|\cdot\|_{\WBan^*}(-\nabla_x E_{u_0}(u(t))),0) , \quad \text{for a.e. } t\in(0,T),
\end{align*}
where 
\begin{align*}
    E_{u_0}(x,y)=E(x,y)+\chi_{B_\epsilon((u_0)_x)}(x).
\end{align*}
In \cite{wong2020fast} the authors suggested to combine \textit{FGSM} with stochastic elements. They proposed to use a single step
\begin{gather*}
    \sigma\sim \text{Uniform}\left(\cl{B_\epsilon^\infty}(x_0)\right),\\
    x_\frac{1}{2}=x_0+\sigma,\\
    x_1=\Clip_{0,\budget}\left(x_\frac{1}{2}+\sign(\nabla \loss(\hyp(x_\frac{1}{2}), y))\right).
\end{gather*}
This is reminiscent of the classical Langevin algorithm, therefore it would be interesting if this stochasticity could be incorporated into our framework.

\section*{Acknowledgements}
MB, TR and LW acknowledge support from DESY (Hamburg, Germany), a member of the Helmholtz Association HGF. This research was supported in part through the Maxwell computational resources operated at Deutsches Elektronen-Synchrotron DESY, Hamburg, Germany. MB and LW acknowledge support  from the German Research Foundation, project BU 2327/20-1. \rev{MB and
TR acknowledge funding by the German Ministry of Science and Technology (BMBF) under grant agreement
No. 01IS24072A (COMFORT). TR further wants to thank Samira Kabri for many insightful discussions.} Parts of this study were carried out while LW and TR were affiliated with the Friedrich-Alexander-Universität Erlangen-Nürnberg.\\

Competing interests: The authors declare none.
\appendix
\section{Convex analysis}
This section gives an overview over well-known definitions and statements in convex analysis. In the following $\Inp$ denotes a Banach space and $\Inp^*$ its dual.
\begin{definition}[Subdifferential]
For a convex function $f:\Inp\to (-\infty, \infty]$, we denote by
\begin{align*}
\partial f(\x) := \{\xi\in \Inp^*: f(\xx) - f(\x) \geq
\langle \xi, \xx - \x \rangle\quad\forall \xx\in \Inp\} 
\subset \Inp^* 
\end{align*}
the subdifferential of $f$ at $x\in \Inp$.
\end{definition}
If $f(\cdot)=\|\cdot\|$ then the subdifferential is given by 
\begin{align}\label{eq: SubN}
\partial \|\cdot\|(x)=\{ \xi\in \Inp^*|\langle \xi, x\rangle=\|x\|,\|\xi\|_*\leq 1\}
\end{align}

\begin{definition}[Fenchel conjugate]
For a function $f:\Inp \rightarrow [-\infty,+\infty]$, we denote by $f^*: \Inp^* \rightarrow [-\infty,+\infty]$,
\begin{align*}
f^*(\xi)\coloneqq \sup_{x\in \Inp} \langle \xi, \x\rangle -f(\x) \quad \text{for } \xi \in \Inp^*
\end{align*}
 the \emph{Fenchel conjugate} of $f$.
\end{definition}
A direct consequence of this definition is the so called Fenchel--Young inequality 
\begin{align}\label{eq:Young}
    \langle \xi,\x\rangle \leq f(\x)+f^*(\xi).
\end{align}
The next proposition yields the conditions under which the equality in \eqref{eq:Young} is obtained.
\begin{proposition}[{\cite[Prop. 2.33]{barbu2012convexity}}] \label{pro:Young}
Let $f:\Inp\rightarrow ]-\infty,+\infty]$ be a proper convex function. Then for $x\in \Inp$, the following three properties are equivalent:
\begin{propenum}
\item\label{prop:youngi} $\xi\in \partial f(x)$.
\item\label{prop:youngii} $f(x)+f^*(\xi)\leq\langle \xi,x\rangle$.
\item\label{prop:youngiii} $f(x)+f^*(\xi)=\langle \xi,x\rangle$.
\end{propenum}
If, in addition, $f$ is lower-semicontinuous, then all of these properties are equivalent to the following one.
\begin{propenum}[resume]
\item\label{prop:youngiv} $x\in\partial f^*(\xi)$.
\end{propenum}
\end{proposition}
\begin{remark}
In \cref{prop:youngiv}, we use the canonical embedding to obtain the subspace relation $\Inp\subset\Inp^{**}$. Following \cite[Rem. 2.35]{barbu2012convexity}, if $\Inp$ is reflexive, i.e. $\Inp^{**}=\Inp$, then it follows from Proposition \ref{pro:Young} 
that
\begin{align*}
    x\in \partial f^*(\xi) \Longleftrightarrow \xi\in \partial f(x),
\end{align*}
which yields
\begin{align*}
(\partial f)^{-1}(\xi) = 
\{x\in\Inp\st \xi\in \partial f(x)\} = 
\{x\in\Inp\st x\in \partial f^*(\xi)\} = 
\partial f^*(\xi)
\end{align*}
In the non-reflexive case one can not argue as above, and we do not obtain the simple relation between $\partial f^*$ and $\partial f$, see, e.g., \cite{rockafellar1970maximal}.
\end{remark}
An important corollary of \cref{pro:Young} is its application to the indicator function of the closed unit ball $f=\chara_{\cl{\unitb}}$, where its convex conjugate for $\xi\in\Ban^*$ is given by
\begin{align*}
\chara_{\cl{\unitb}}^*(\xi)=
\sup_{\x\in\Ban}\langle \xi, \x\rangle-\chara_{\cl{\unitb}}(\x)=
\sup_{\x\in \cl{\unitb}}\langle \xi, \x\rangle=\|\xi\|_*.
\end{align*}

\begin{corollary}\label{cor:charainclusion}
For a Banach space $\Ban$, and $\xi\in\Ban^*$ we have that
\begin{align}\label{eq:indicator}
\partial \norm{\cdot}_*(\xi)\cap \WBan = \argmax_{\x\in\cl{\unitb}} \langle \xi, \x\rangle.
\end{align}

\end{corollary}
\begin{proof}
Since $\chara_{\cl{B_1}}$ is lower semicontinuous, we can 
use the equivalence of \cref{prop:youngiii} and \cref{prop:youngiv}, to infer
\begin{align*}
\x\in\partial \norm{\cdot}_*(\xi)
\quad\Leftrightarrow\quad%
\norm{\xi}_* = \langle\xi,\x\rangle - \chara_{\cl{\unitb}}(\x).
\end{align*}
In the second statement, using the definition of $\norm{\xi}_*$ as in \cref{eq:indicator}, therefore yields that each $\x$ above realizes the supremum, which concludes the proof.
\end{proof}

\section{Refined version of Ascoli--Arzelà}

\begin{proposition}[{\cite[Prop. 3.3.1]{ambrosio05}}]\label{prop:AA}
Let $u^n:[0,T]\to\CMS$ be a sequence of curves, that fulfills the following conditions:
\begin{enumerate}[label=(AA-\roman*), align=left,]
\item\label{def:AA1} There is a $\sigma$-sequentially compact set $K\subset\CMS$, such that 
 \begin{align*}
 u^n(t)\in K\quad \text{for every}\quad t\in[0,T]\quad \text{and every}\quad n\in\N.
 \end{align*}
\item\label{def:AA2} There is a symmetric function $\omega:[0,T]\times[0,T]\to [0,+\infty)$ with $\lim_{(s,t)\to(r,r)} \omega(s,t) = 0$ for all $r\in [0,T]\setminus C$, where $C$ is an at most countable set, such that
\begin{align*}
\limsup_{n\to\infty} d(u^n(s), u^n(t))\leq \omega(s,t)\quad\text{for all}\quad s,t\in[0,T].
\end{align*}
\end{enumerate}%
Then there exists a subsequence $u^{n_k}$ and a limit curve $u:[0,T]\to\CMS$, which is $d$-continuous in $[0,T]\setminus C$, such that
\begin{align*}
u^{n_k}(t) \overset{\sigma}{\rightharpoonup} u(t)\quad\text{for all}\quad t\in [0,T].
\end{align*}
\end{proposition}

\section{Taylor's formula in Banach spaces}
\begin{theorem}\label{thm:Taylor}
Suppose $E$, $F$ are real Banach spaces, $U\subset E$ an open and nonempty subset, and $f\in C^n(U,F)$. Given $x_0\in U$ choose $r>0$ such that $x_0+B_r \subset U$, where $B_r$ is the open ball in $E$ with center $0$ and radius $r$. Then for all $h\in B_r$ we have, using the abbreviation $h^k=(h,...,h)$, $k$ terms,
\begin{align*}
    f(x_0 +h)=\sum_{k=0}^n \frac{1}{k!} f(k) (x_0)(h)^k +R_n(x_0,h),
\end{align*}
 where the remainder $R_n$ is of form
 \begin{align*}
     R_n(x_0,h)= \frac{1}{(n-1)!} \int_0^1 (1-t)^{n-1} [f^{(n)}(x_0+th)-f^{(n)}(x_0)] (h)^n dt.
 \end{align*}
\end{theorem}
\begin{proof}
A proof for this statement can be found, e.g., in \cite[Thm. 30.1.3]{blanchard2015mathematical}.
\end{proof}

\section{Prokhorov's theorem}

\begin{theorem}[Prokhorov {\cite[Theorem 5.1-5.2]{billingsley2013convergence}}]\label{thm:prokh} If a set $\mathcal{K}\subset \Prob(\WBan)$ is tight, i.e.,
\begin{align}\label{eq: tight}
    \forall \epsilon>0\quad \exists K_\epsilon \text{ compact in } \WBan \text{ such that } \mu(\WBan\setminus K_\epsilon)\leq \epsilon \quad \forall \mu\in \mathcal{K},
\end{align}
then $\mathcal{K}$  is relatively compact in $\Prob(\WBan)$. Conversely, if $\WBan$ is a Polish space, every relatively compact subset of $\Prob(\WBan)$ is tight.
\end{theorem}

\section{Helpful lemmas and supplementary proofs}\label{app:help}
\rev{%
In the following, we provide the proof of \cref{lm: EqiSlope}, which is a particular case of \cite[Lem. Lemma 3.1.5]{ambrosio05}. For completeness, we provide a version of the proof that is specifically adapted to the case $p=\infty$.

\begin{proof}[Proof of \cref{lm: EqiSlope}]
Let us suppose that for all $\tau>0$, $\func_\tau(\x)< \func(\x)$ else $|\partial\func|(\x)=0$ and equality \labelcref{eq:AltSlop} holds trivially.
We calculate
\begin{align*}
    \limsup_{\tau \rightarrow 0^+} \frac{\func(\x)-\func_\tau (\x)}{\tau} &= \limsup_{\tau \rightarrow 0^+}\sup_{\xx:0<d(\x,\xx)\leq \tau} \frac{\func(\x)-\func(\xx)}{\tau}\\
    &= \inf_{\epsilon>0} \sup_{0<\tau\leq \epsilon} \sup_{\xx:0<d(\x,\xx)\leq \tau} \frac{\func(\x)-\func(\xx)}{\tau}\\
    &= \inf_{\epsilon>0} \sup_{\xx,\tau:0<d(\x,\xx)\leq \tau\leq \epsilon} \frac{\func(\x)-\func(\xx)}{\tau}\\
    &=\inf_{\epsilon>0} \sup_{\xx:0<d(\x,\xx)\leq \epsilon} \sup_{\tau:d(\x,\xx)\leq \tau \leq \epsilon} \frac{\func(\x)-\func(\xx)}{\tau}\\
    &\rev{=\inf_{\epsilon>0} \sup_{\xx:0<d(\x,\xx)\leq \epsilon} \frac{(\mathcal{E}(x)-\mathcal{E}(z))^+}{d(x,z)}-\frac{(\mathcal{E}(x)-\mathcal{E}(z))^-}{\epsilon}}\\
    &\stackrel{(*)}{=}\inf_{\epsilon>0} \sup_{\xx:0<d(\x,\xx) \leq \epsilon} \frac{\func(\x)-\func(\xx)}{d(\x,\xx)}\\
    &=\limsup_{\xx \rightarrow \x} \frac{\func(\x)-\func(\xx)}{d(\x,\xx)}\\
    &= |\partial \func|(\x).
\end{align*}
\rev{Equality  $(*)$ can be verified by the observation that $\func_\tau(\x)< \func(\x)$ for all $\tau>0$ ensures the existence of at least one $\xx$ with $d(\x,\xx)\leq \epsilon$ such that $\func(\x)-\func(\xx)\geq 0$.}
\end{proof}%
}
\begin{lemma}\label{lm: help1}
Let $\phi:[0,T]\rightarrow \R$ be continuous and $\nimap:[0,T]\rightarrow \R$ be non-increasing. If $\phi(t)=\nimap(t)$ for a.e. $t\in [0,T]$, then $\phi(t)=\nimap(t)$ for all $t\in (0,T)$.
\end{lemma}
\begin{proof}
Assume there is a $t\in(0,T)$ such that $\phi(t)\not=\nimap(t)$. Without loss of generality $\phi(t)>\nimap(t)$. Then we can take a sequence $t_n$ with $t_n\rightarrow t$ and $\phi(t_n)=\nimap(t_n)$ and $t_n>t$. The continuity of $\phi$ implies that for any $\epsilon>0$ we can choose $t_n$ small enough, such that $|\phi(t)-\phi(t_n)|<\epsilon$. This contradicts the monotonicity of $\nimap$, since if we choose $\epsilon < \phi(t)-\nimap(t)$, we obtain a $t_n>t$ with $\nimap(t)<\nimap(t_n)=\phi(t_n)$. 
In the case $\phi(t)<\nimap(t)$, we can make the same argument with sequences $t_n<t$.
\end{proof}
\rev{%
\rev{In the following we show that the arguments of \cite[Thm. 7]{Lisini07} can indeed be adapted to the case $p=\infty$. We closely follow the arguments in \cite[Thm. 7]{Lisini07}, where it was proven for $p\in(1,\infty)$. For convenience, we copy the relevant steps and show how to adapt them to the case $p=\infty$. 
\begin{proof}[Proof of \cref{thm: ECp}]
Let $\mathcal{L}^1_{(0,T)}$ denote the Lebesgue measure on $(0,T)$, then for $\eta$ from \cref{thm: curveDe}, we define $\Bar{\eta}\coloneqq \frac{1}{T} \mathcal{L}^1_{(0,T)}\otimes \eta$ and the evaluation map $e: [0,T]\times C(0,T;\WBan) \rightarrow [0,T]\times \WBan$ by 
$$ e(t,u)=(t,e_t(u))=(t,u(t)).$$ 
We observe that $ e_{\#}\eta=\Bar{\mu}$ and denote by $\Bar{\eta}_{x,t}$ the Borel family of probability measures on $C(0,T;\WBan)$ obtained by disintegration of $\Bar{\eta}$ with respect to $e$, such that $d\Bar{\eta}_{x,t}(u) d \Bar{\mu}(t,x) =d\bar{\eta}(t,u) $. Notably $\Bar{\eta}_{x,t}$ is concentrated on $\{u: e_t(u)=x\} \subset C(0,T;\WBan)$. 
Since $\WBan$ is assumed to satisfy the Radon--Nikodým property the pointwise derivative $u'(t)=\lim_{h\rightarrow 0}\frac{u(t+h)-u(t)}{h}$ is defined a.e. for an absolutely continuous curve $u$.
We now show that 
$$ A\coloneqq \{(t,u)\in [0,T]\times C(0,T;\WBan): u'(t) \text{ exists}\}$$
is a Borel set and $\Bar{\eta}(A^c)=0$.
For every $h\not=0$, we define the continuous function $g_h: [0,T]\times C(0,T;\WBan) \rightarrow \WBan$ by $g_h(t,u)=\frac{u(t+h)-u(t)}{h}$, where we extend the function $u$ outside of $[0,T]$ by $u(s)=u(0)$ for $s<0$ and $u(s)=u(T)$ for $s>T$.
By completeness of $\WBan$ 
$$ A^c \coloneqq \{ (t,u):\limsup_{(h,k)\rightarrow (0,0)}\| g_h(t,u)-g_k(t,u)\|>0 \} $$
and because of the continuity of the function $(t,u) \mapsto \| g_h(t,u)-g_k(t,u)\|$, $A^c$ and $A$ are Borel sets. Since $\Bar{\eta}$ is concentrated on $[0,T]\times \AC^\infty (0,T; \WBan)$ and $u'(t)$ exists a.e. for an absolutely continuous curve $u$, by Fubini's theorem $\Bar{\eta}(A^c)=0$.
Thus, for $\bar{\eta}$-a.e. $(t,u)$ the map 
$$\psi(t,u)=u'(t)$$
is well-defined. For every $x^*\in \WBan^*$, we define $\psi_{x^*}(t,h)\coloneqq \langle x^*, u'(t)\rangle$ on $(t,u)\in A$. As a limit of continuous functions $\psi_{x^*}$ is a Borel function on $A$ and thus $\Bar{\eta}$ measurable. Since $\WBan$ is separable, Pettis theorem ensures that $\psi$ is a $\Bar{\eta}$-measurable function. Now we can define the vector field 
$$ \velo_t(x)\coloneqq \int_{C(0,T;\WBan)} u'(t) d\Bar{\eta}_{x,t} \quad \text{ for 
 } \Bar{\mu}-a.e. (t,x)\in (0,T)\times \WBan.$$
\rev{For clarity, we now indicate the varibles over which the $\esssup$ is taken in brackets after the respective measure. Using this notation we estimate}
\begin{align*}
\Bar{\mu}-\esssup  \| \bm{v}\| &= \Bar{\mu}(x,t)-\esssup  \left\| \int_{C(0,T;\WBan)} u'(t) d\Bar{\eta}_{x,t}\rev{(u)}\right\|\\
&\leq
\Bar{\mu}(x,t)-\esssup   \int_{C(0,T;\WBan)} \left\|u'(t)\right\| d\Bar{\eta}_{x,t}\rev{(u)}\\
&\leq
\rev{
\Bar{\mu}(x,t)-\esssup\,  \big(\Bar{\eta}_{x,t}\rev{(u)}-\esssup \left\|u'(t)\right\|\big) 
}\\
&\leq \Bar{\eta}(u,t)-\esssup \left\|u'(t)\right\| <+\infty,
\end{align*}
and thus $\bm{v}\in L^\infty (\Bar{\mu};\WBan)$, where the last inequality follows from \cref{lem:disintesssup}. By Jensen's inequality we have for every $[a,b]\subset [0,T]$,
\begin{align}\label{eq: JensEst}
\begin{aligned}
\int_a^b \|\velo_t\|_{\rev{L^\infty}(\mu_t;\WBan)}dt&=\int_a^b \mu_t(\x)-\esssup \|\velo_t(x)\| dt\\&=\int_a^b \mu_t(\x)-\esssup \left\|\int_{C(0,T;\WBan)} u'(t) d\Bar{\eta}_{x,t}\right\| dt\\
&\leq\int_a^b \mu_t(\x)-\esssup \int_{C(0,T;\WBan)} \left\|u'(t)\right\| d\Bar{\eta}_{x,t} dt\\
&\leq \int_a^b \mu_t(\x)-\esssup\,  \Bar{\eta}_{x,t}(u)-\esssup\left\|u'(t)\right\|  dt
\\
&=\int_a^b \eta(u)-\esssup \| u'(t)\| dt=\int_a^b |\mu'|(t)dt.
\end{aligned}
\end{align}
such that $\|\velo_t\|_{L^p(\mu_t;\WBan)}\leq |\mu|'(t)$ for a.e.  $t\in(0,T)$. In the last inequality we used the fact, that $d\eta =  d\bar{\eta}_{x,t} d\mu_t$ holds for a.e. $t\in(0,T)$ together with \cref{lem:disintesssup}.
For more rigorous justifications regarding measurably and  integrability of all involved quantities, we refer to \cite[Theorem 7 ]{Lisini07}.
To show that $(\mu,\bm{v})\in \EC^\infty(\WBan)$  we take $\varphi\in C^1_b(\WBan)$ and observe that  $t\rightarrow \int_{\WBan} \varphi(x) d\mu_t(x)$ is absolutely continuous, since for $\gamma\in\Gamma_0(\mu_t,\mu_s)$
\begin{align*}
  \left|\int_\WBan \varphi d\mu_t -\int_\WBan \varphi d\mu_s\right|\leq \int_{\WBan\times\WBan} |\varphi(x)-\varphi(\Tilde{x})|d\gamma\leq\\
  \sup_{x\in\WBan} \|D\varphi(x)\|  \int_{\WBan\times\WBan} \|x-\Tilde{x}\|d\gamma\leq \sup_{x\in\WBan} \|D\varphi(x)\| W_\infty(\mu_t,\mu_s).  
\end{align*}
Further
\begin{align*}
\int_\WBan \varphi d\mu_t&-\int_\WBan \varphi  d\mu_s=\int_{C(0,T;\WBan)} \varphi(u(t))-\varphi(u(s))d\eta(u)\\&=
\int_{C(0,T;\WBan)} \langle D\varphi(u(s)),u(t)-u(s)\rangle d\eta(u)+\int_{C(0,T;\WBan)} \|u(t)-u(s)\| \omega_{u(s)}(u(t)) d\eta(u)\\&=
\int_{C(0,T;\WBan)} \langle D\varphi(u(s)),\int_s^t u'(r)dr\rangle d\eta(u)+\int_{C(0,T;\WBan)} \|u(t)-u(s)\| \omega_{u(s)}(u(t)) d\eta(u)
\end{align*}
where 
$$ \omega_x(y)=\frac{\varphi(y)-\varphi(x)-\langle D\varphi(u(x)),y-x\rangle}{\|y-x\|}.$$
We observe
$$\frac{1}{t-s}\langle D\varphi(u(s)),\int_s^t u'(r)dr\rangle \rightarrow \langle D\varphi(u(s)), u'(s)\rangle \quad \text{for } \eta \text{-a.e. } u$$ 
and 
$$
\frac{\|u(t)-u(s)\|}{t-s} \omega_{u(s)}(u(t)) \rightarrow 0 \quad \text{for } \eta \text{-a.e. } u
$$
and have for $\eta$-a.e. $u$ the upper bounds
\begin{align*}
\frac{1}{|t-s|}|\langle D\varphi(u(s)),\int_s^t u'(r)dr\rangle| &\leq \sup_{x\in \WBan} \|D\varphi(x)\|_*  \frac{\left\|\int_s^t u'(r)dr\right\|}{|s-t|}\\&\leq  \sup_{x\in \WBan} \|D\varphi(x)\|_* \esssup_{r\in[0,T]} |\mu'|(r)<+\infty  
\end{align*}
and
\begin{align*}
\frac{\|u(t)-u(s)\|}{|t-s|} |\omega_{u(s)}(u(t))|&\leq \esssup_{r\in[0,T]}|\mu'|(r) \left(\frac{|\varphi(u(t))-\varphi(u(s))|}{\|u(t)-u(s)\|}+\frac{|\langle D  \varphi(u(s)),u(t)-u(s)\rangle|}{\|u(t)-u(s)\|}\right)\\&\leq \esssup_{r\in[0,T]}|\mu'|(r)\,  2\Lip(\varphi)<+\infty.
\end{align*}
Dividing by $t-s$ and passing to the limit $t\rightarrow s$  by using Lebesgue theorem, we obtain
\begin{align*}
\frac{d}{ds} \int_\WBan \varphi d\mu_s=\int_{C(0,T;\WBan)} \langle D \varphi(u(s)),u'(s)\rangle d\eta(u)=\int \langle D \varphi, \velo_t\rangle d\mu_t \quad \text{for a.e. } s\in(0,T).
\end{align*}
This pointwise derivative corresponds to the distributional derivative and we obtain $(\mu,\bm{v})\in \EC^\infty(\WBan)$. 
\end{proof}
}%
Similarly, we can adapt \cite[Thm. 8]{Lisini07} to the case $p=\infty$, which we again show by reusing most of the arguments from the corresponding proof in \cite{Lisini07}.
\begin{proof}[Proof of \cref{thm: ECRp}]
This theorem was proven in \cite[Theorem 8]{Lisini07} for $p\in (1,+\infty)$ and can easily be extended to the case $p=+\infty$. Let $(\mu_t)_{t\in[0,T]}$ be 
a family of measures in $\Prob_\infty(\WBan)$ and for each $t$ we have a velocity field $\velo_t\in L^\infty(\mu_t;\mathbb{R}^d)$ with $ \esssup \|\velo_t\|_{L^\infty(\mu_t)}  <\infty$ , solving the continuity equation in the sense of distributions. Since 
\begin{align*}
\|\velo\|_{L^p(\Bar{\mu};\WBan)} \leq T^{1/p} \esssup \|\velo_t\|_{L^\infty(\mu_t)}  <\infty
\end{align*}
we can apply \cite[Theorem 8]{Lisini07} (i.e., the statement of \cref{thm: ECRp}) for all $p\in (1,\infty)$ and get
\begin{align*}
|\mu'|_{(p)}(t)\leq\|\velo_t\|_{L^p(\mu_t;\WBan)}   \text{ for  a.e.} \ t\in (0,T) \text{ and all } p\in(1,\infty).
\end{align*}
Therefore 
\begin{align*}
W_p(\mu_t, \mu_s)\leq \int_t^s |\mu'|_{(p)}(\Tilde{t}) d\Tilde{t}\leq \int_t^s \|\velo_t\|_{L^p(\mu_{\Tilde{t}};\WBan)} d\Tilde{t}\leq \int_t^s \|\velo_t\|_{L^\infty(\mu_{\Tilde{t}};\WBan)} d\Tilde{t}
\end{align*}
for all $t,s\in [0,T]$ with $t\leq s$ and $p\in(1,\infty)$, where $|\mu'|_{(p)}$ denotes the metric derivative of $\mu$ in $W_p$. Taking the limit $p\to\infty$, we get
\begin{align*}
W_\infty(\mu_t, \mu_s)=\lim_{p\rightarrow \infty}W_p(\mu_t, \mu_s)\leq \int_t^s \|\velo_t\|_{L^\infty(\mu_{\Tilde{t}};\WBan)} d\Tilde{t}
\end{align*}
 for all $t,s\in [0,T]$ with $t\leq s$ and thus by the minimality of the metric derivative, see \cref{rem:minmetric}, 
\begin{align*}
    |\mu'|_{(\infty)}(t)\leq \|\velo_t\|_{L^\infty(\mu_t;\WBan)} \text{ for a.e. }t \in (0,T).
\end{align*}
\end{proof}
}%
\begin{lemma} \label{lm: MeasHelp}
    Let $\mu$ be a Borel probability measure on $\WBan$ and $v: \WBan \rightarrow \WBan$, $\tilde{v}:\WBan\rightarrow \WBan$ be two $\mu$-measurable functions with
    \begin{align*}
        \int \langle D\varphi(x),v(x)\rangle d\mu(x)=\int \langle D\varphi(x),\tilde{v}(x)\rangle d\mu(x) \quad \forall \varphi\in C^1_b(\WBan)
    \end{align*}
    then
    \begin{align}\label{eq: DMeas}
        \int \langle \xi, v(x)\rangle d\mu(x) =\int \langle \xi, \tilde{v}(x)\rangle d\mu(x) \quad \forall \xi\in \WBan^*.
    \end{align}
\end{lemma}
\begin{proof}
    Let $g_n:\mathbb{R}\rightarrow \mathbb{R}$ be the function with $g(0)=0$ and
    \begin{align*}
        g'(x)=\begin{cases}
            0 & \text{for } |x|>n+1,\\
            1 & \text{for } |x|<n ,\\
            n+1-x & \text{for } x\in[n,n+1],\\
            n+1+x& \text{for } x \in [-(n+1),-n].
        \end{cases}
    \end{align*}
Then for each $\xi\in \WBan^*$ we get $G_n:x\mapsto g_n(\langle \xi, x\rangle)\in C^1_b(\WBan)$ with $DG_n(x)=g_n'(\langle \xi,x\rangle)\  \xi$ and
\begin{align*}
\int g_n'(\langle \xi,x\rangle)\langle \xi,v(x)\rangle d\mu
&=\int \langle D G_n(x),v(x)\rangle d\mu= \int \langle D G_n(x),\tilde{v}(x)\rangle d\mu\\&=\int g'_n(\langle \xi,x\rangle)\langle \xi,\tilde{v}(x)\rangle d\mu
\end{align*}
Since for $n\rightarrow \infty$ we have $g'(\langle \xi,x\rangle)\rightarrow 1$ pointwise we can apply the Lebesgue dominated convergence theorem (with the functions $|\langle \xi, v(x)\rangle|$ and $|\langle \xi, \tilde{v}(x)\rangle|$ as bound) to obtain \eqref{eq: DMeas}.
\end{proof}

The following lemma, shows that the disintegration property can be transferred to an inequality for essential suprema. For more details on disintegration, we refer to \cite[Ch. 5.3]{ambrosio05} and \cite[Ch. III-70]{dellacherie1978probabilities}. The proof strategy is taken from \cite[Lem. 2]{roith2023continuum} and amounts to controlling the null sets of the measures involved.
\begin{lemma}\label{lem:disintesssup}
Given $\Inp, \Inpp$ Radon separable metric spaces, a measure $\mu\in\mathcal{P}(\Inp)$, a Borel map $\pi:\Inp\to\Inpp$ and a disintegration $d\mu = d\mu_\xx d\nu$, with $\nu=\pi_\#\mu$ and $\{\mu_\xx\}_{\xx\in\Inpp}\subset \mathcal{P}(\Inp)$ being a family of probability measures, then we have that
\begin{align*}
\mu(\x)-\esssup f(x) \geq \nu(\xx)-\esssup\, \mu_\xx(\x)-\esssup f(\x)
\end{align*}
for every Borel map $f:\Inp\to[0,\infty]$.
\end{lemma}
\begin{proof}
Using the disintegration property, for every Borel set $A$ we obtain 
\begin{align*}
\mu(A) = 0 \quad\Leftrightarrow\quad \mu_\xx(A) \text{ for } \nu-\text{a.e. }  \xx\in\Inpp.
\end{align*}
Now assume that $\mu(A) = 0$, then we know that there exists a Borel set $B\subset\Inpp$ with $\nu(B)=0$ and $\mu_\xx(A) =0$ for all $\xx\in \Inpp\setminus B$. 
Therefore,
\begin{align*}
\sup_{x\in \Inp\setminus A} f(\x) &\geq \inf_{\tilde{A}:\mu_\xx(\tilde{A})=0} \sup_{x\in \Inp\setminus \tilde{A}}  f(\x) = 
\mu_\xx(\x)-\esssup f(\x)
\qquad\text{ for all } \xx\in\Inpp\setminus B\\
\Rightarrow
\sup_{x\in \Inp\setminus A} f(\x) &\geq 
\sup_{\xx\in\Inpp\setminus B}\, 
\mu_\xx(\x)-\esssup f(\x)\\
&\geq 
\inf_{\tilde{B}:\nu(\tilde{B})=0}\sup_{\xx\in\Inpp\setminus \tilde{B}}
\mu_\xx(\x)-\esssup f(\x)\\ 
&= 
\nu(\xx)-\esssup\, 
\mu_\xx(\x)-\esssup f(\x)
\end{align*}
and since this holds for every $\mu$-null set $A$, we can take the infimum to obtain
\begin{align*}
\mu(\x)-\esssup f(\x) = \inf_{A:\mu(A)=0}\sup_{x\in \Inp\setminus A} f(\x) \geq \nu(\xx)-\esssup\, 
\mu_\xx(\x)-\esssup f(\x).
\end{align*}
\end{proof}

\begin{lemma}\label{lem:timeesssup}
Let $\eta\in\mathcal{P}(C(0,T;\Inp))$, then we have that 
\begin{align*}
\eta(u)-\esssup \abs{u'}(t) \leq 1\quad &\text{for a.e. }t\in(0,T)\Longleftrightarrow
\esssup_{t\in(0,T)}\, \abs{u'}(t)\leq 1\quad &\text{for } \eta \text{ a.e. } u \in C(0,T;\Inp).
\end{align*}
\end{lemma}

\begin{proof}
Choosing $\psi=\chara_{[0,1]}$, and observing that $\psi(|u'|(t))$ is $\bar{\eta}$-measurable (see \cite[Eq. 55]{lisini2014absolutely}) implies

\begin{align*}
\eta(u)-\esssup \abs{u'}(t) \leq 1\quad &\text{for a.e. }t\in(0,T)
\\&\Longleftrightarrow
\int_{C(0,T,\WBan)} \psi(\abs{u'}(t)) d\eta(u) = 0\quad \text{for a.e. }t\in(0,T)\\
\Longleftrightarrow
\int_0^T \int_{C(0,T,\WBan)} \psi(\abs{u'}(t)) d\eta(u)dt = 0
&\Longleftrightarrow
\int_{C(0,T,\WBan)} \int_0^T  \psi(\abs{u'}(t)) d\eta(u)dt = 0
\\ \Longleftrightarrow
\int_{C(0,T,\WBan)} \int_0^T  \psi(\abs{u'}(t))dt d\eta(u) = 0
&\Longleftrightarrow\int_0^T  \psi(\abs{u'}(t))dt = 0\quad \text{for }\eta \text{ a.e. } u \in C(0,T;\Inp)\\
\Longleftrightarrow
\esssup_{t\in(0,T)}\, \abs{u'}(t)\leq 1\quad &\text{for } \eta \text{ a.e. } u \in C(0,T;\Inp),
\end{align*}
where we use Fubini--Tonelli theorem to change the order of integration.
\end{proof}

\section{Multivalued correspondences}
\label{sec: corr}
For multivalued correspondences, generalizations of continuity and  measurability can be defined. We use the definitions from \cite{guide2006infinite}. In the following, we write $\varphi:\Inp  \mto \Inpp$ to denote a mapping $\varphi:\Inp\to2^\Inpp$.

\begin{definition}[Weak measurability]\label{def:weakmeas}
Let $(S,\Sigma)$ be a measurable space and $\Inp$ be a topological space. We say that a correspondence $\varphi: S \mto \Inp$ is weakly measurable, if 
$$ \varphi^l(G)\in \Sigma \text{ for all open sets } G \text{ of }\Inp,$$
where
\rev{%
\begin{equation}\label{eq:lowinv}
\varphi^l(G)\coloneqq \{s\in S| \varphi(s) \cap G\neq \emptyset\} 
\end{equation}}
is the so-called lower inverse.
\end{definition}

\begin{definition}[Measurability] 
Let $(S,\Sigma)$ be a measurable space and $\Inp$ a topological space. We say that a correspondence $\varphi: S \mto \Inp$ is measurable, if 
$$ \varphi^l(F)\in \Sigma \text{ for all closed sets } F \text{ of }\Inp.$$
\end{definition}

The next theorem is known as the measurable maximum theorem, where we refer to \cite[Thm. 18.19]{guide2006infinite} for the proof of this statement.
\begin{theorem}[Measurable Maximum Theorem]\label{thm:MeaMax}
Let $\WBan$ be a separable metrizable space
and $(S,\Sigma)$ a measurable space. Let $\varphi: S \mto \WBan$ be a weakly measurable correspondence with nonempty compact values, and suppose $f:S\times \WBan\rightarrow \mathbb{R}$ is a
Carathéodory function. Define the value function $m:S\rightarrow \mathbb{R}$ by
$$m(s)=\max_{x\in \varphi(s)}f(s,x),$$
and the correspondence $\mu: S\mto \WBan$ of maximizers by
$$\mu(s)=\{x\in \varphi(s): f(s,x)=m(s)\}.$$
Then
\begin{itemize}
    \item The value function $m$ is measurable.
    \item The \enquote{argmax} correspondence $\mu$ has nonempty and compact values.
    \item The \enquote{argmax} correspondence $\mu$ is measurable and admits a measurable selector.
\end{itemize} 
\end{theorem}

\section{Calculations for distributional adversaries }\label{sc:ApDiAd}
For completeness, we state all lemmas used in \cref{ch: DisAd} here. Those lemmas correspond to a lemma proven in section \cref{sc: CurvPot} and are only adapted to the setting of the transport distance $D$.

\begin{lemma}\label{lm: DMin}
Let $\Inp\times \Oup = \R^d\times\R^m$, then the correspondence 
\begin{align*}
\bm{r}_\budget(x,y)=\argmin_{(\xatt, \tilde{y}):c(x,y,\xatt,\Tilde{y})\leq \budget} E(\Tilde{x},\Tilde{y})=\left(\argmin_{\xatt\in\cl{B_\budget}(\x)}E(\Tilde{x},y),y\right)
\end{align*}
is measurable and admits a $\mathcal{B}(\WBan\times\Oup)$-measurable selector.
Further, for each measurable selector $r_\budget:\Inp\times\Oup\to \Inp$ we have the following, 
\begin{align}\label{eq:PushMax}
     (r_\budget)_\#(\mu)\in \argmin_{D(\mu,\Tilde{\mu})\leq \budget} \int E(x,y) d\mu(x,y).
\end{align}
\end{lemma}
\begin{proof}
We consider the correspondence $\varphi:\Inp\times\Oup \mto \Inp \times \Oup$ given by
\begin{align*}
(x,y)\mapsto \left(\cl{B_\budget}(x), y\right)
\end{align*}
where on the input space we use the topology induced by $\|\cdot\|_{\WBan}+\|\cdot\|_{\Oup}$ and the output space is interpreted as the standard Euclidean space. Then we have that for every open set $G\in\WBan$ that
\begin{align*}
\varphi^l(G)=\{ (x,y)\in \WBan\times \Oup \st \left(\cl{B_\budget}(x),y\right)\cap G\neq \emptyset\} \subset \mathcal{B}(\WBan\times \Oup),
\end{align*}
is open, which implies weak measurability, according to \cref{def:weakmeas}. Furthermore, we define the map $f((\x,y), (\Tilde{x},\Tilde{y})):= -\pot(\Tilde{x}, \Tilde{y})$
which is a Carathéodory function, since $\pot$ is continuous, with 
\begin{align*}
\max_{(\Tilde{x},\Tilde{y}) \in \varphi(x,y)} f((\x,y),(\Tilde{x},\Tilde{y})) =
\min_{(\Tilde{x},\Tilde{y}):c(x,y, \xatt,\Tilde{y})\leq \epsilon} -\pot(\Tilde{x},\Tilde{y})
\end{align*}
Then \cref{thm:MeaMax} ensures the existence of a measurable selector.
To prove \cref{eq:PushMax} we observe that if $D(\mu,\tilde{\mu})\leq \budget$, then for an optimal transport plan $\gamma\in\Gamma_0(\mu,\Tilde{\mu})$, we know that $y=\Tilde{y}$ and $\|\x-\xatt\|\leq \budget$, $\gamma$-a.e. Thus, using the disintegration $d\gamma(x,y, \xatt,\tilde{y}) = d\psi_{x,y}(\xatt,\tilde{y}) d\mu(x,y)$, for every $\tilde{\mu}$ with $D(\Tilde{\mu},\mu)\leq \budget $, we calculate
\begin{align*}
\int E(\Tilde{x},\Tilde{y}) d\Tilde{\mu}(\Tilde{x},\Tilde{y})=\int E(\Tilde{x},\Tilde{y}) d\gamma(x,y,\Tilde{x},\Tilde{y})=\int \int E(\Tilde{x},\Tilde{y}) d\psi_{x,y}(\Tilde{x},\Tilde{y}) d\mu(x,y)\\
=\int \int E(\Tilde{x},y) d\psi_{x,y}(\Tilde{x},\Tilde{y}) d\mu(x,y)\geq \int \int E(r_\budget(x,y)) d\psi_{x,y}(\Tilde{x},\Tilde{y}) d\mu(x,y)=\int E(r_\budget(x,y)) d\mu(x,y),
\end{align*}
and \cref{eq:PushMax} follows.
\end{proof}
\begin{lemma} \label{lm: GSlop} 
Let $\Inp\times \Oup = \R^d\times\R^m$, $E: \Inp\times \Oup \rightarrow \R$ be in $C^1(\Inp\times \Oup )$ and $\mu\in \mathcal{P}_\infty(\Inp\times \Oup)$ , then the metric slope with respect to $D$ is given by
\begin{align*}
|\partial \ffunc|(\mu) = \int \| \nabla_x E(x,y) \|_{\WBan^*} d\mu.
\end{align*}
\rev{and $|\partial\ffunc|$ is a strong upper gardient.}
\end{lemma}
\begin{proof}
We follow the arguments of \cref{lm:PotSlope}.
    By \cref{lm: DMin} we obtain
    \begin{align*}
    |\partial \ffunc|(\mu)=\limsup_{\tau \rightarrow 0} \frac{\ffunc(\mu)-\ffunc_\tau(\mu)}{\tau}=\limsup_{\tau \rightarrow 0} \int \frac{E(x,y)-E(r_\tau(x,y))}{\tau} d\mu \\
    =\int \lim_{\tau \rightarrow 0} \frac{E(x,y)-E(r_\tau(x,y))}{\tau} d\mu=\int \|\nabla_x E(x,y)\|_{\WBan^*} d\mu,
    \end{align*}
    where dominated convergence together with \cref{lm:Lim} and \cref{pro:C1:iii} was used to draw the limit into the integral.
\rev{
    To prove that $|\partial \ffunc|$ is a strong upper gradient we observe that $\| \nabla_x E(x,y) \|_{\WBan^*}$ is continuous and in particular lower semicontinuous such that we can use \cite[Lemma 5.1.7.]{ambrosio05} to prove that the map $t\mapsto|\partial \ffunc|(\mu_t)$ is lower semicontinuous
    and thus Borel for every absolutely continuous curve $\mu_t$.  
Assume that $\int_s^t \int_{\Inp\times\Oup} \| \nabla_x E(x,y) \|_{\WBan^*} d\mu_r(\x,y) |\mu'|(r) dr
= \int_s^t \abs{\partial\ffunc}(\mu_r) |\mu'|(r)dr<+\infty$, otherwise \cref{eq: StUpGrad} holds trivially. 
By \cref{thm: curveDe} we can estimate
    \begin{align*}
        |\ffunc(\mu_t)-\ffunc(\mu_s)|&=\left|\int_{\Inp\times\Oup} \pot(x,y)d\mu_t(x,y)-\int_{\Inp\times\Oup} \pot(x,y)d\mu_s(x,y)\right|\\
    &=\left|\int_{C(0,T;\Inp\times\Oup)} \pot(u(t))d\eta(u)-\int_{C(0,T;\WBan)} \pot(u(s))d \eta(u)\right|\\
        &\leq \int_{C(0,T;\Inp\times\Oup)}\left| \pot(u(t))- \pot(u(s))\right|d \eta(u)\\
        &\leq \int_{C(0,T;\Inp\times\Oup)} \int_s^t \| \nabla_x E(u(r)) \|_{\WBan^*} \, |u'|(r)dr d \eta(u)\\
        &\leq \int_s^t\int_{\Inp\times \Oup}   \| \nabla_x E(x,y) \|_{\WBan^*} d\mu_r(x,y) |\mu'|(r)dr<+\infty.
    \end{align*}
    Here we use that $\eta$ is concentrated on $AC^\infty(0,T;\Inp\times\Oup)$ and by the definition of the the extended distance $c(x,\tilde{x},y,\tilde{y})$ on $\Inp\times\Oup$ a curve $u(t)\in AC^\infty(0,T;\Inp\times\Oup)$ only moves in $\Inp$-direction and for those curves  $\| \nabla_x E(u(r)) \|_{\WBan^*}$ acts like a strong upper gradient. 
}
\end{proof}

\section{Details on numerical examples}\label{sec:num}

Here, we give some details on the experiment that produces \cref{fig:IFGSM}, the source code is provided at \href{https://github.com/TimRoith/AdversarialFlows}{github.com/TimRoith/AdversarialFlows}.

\subsection{Training the neural network}
We sample $K=1000$ labeled data points $((\x_1,y_1),\ldots, (\x_K,y_K))$, with $\x_k\in\R^2, y_k\in\{0,1\}$, from the two moons data set using the \texttt{sci-kit} package \cite{scikit-learn}, see \cref{fig:twomoon}.
\begin{figure}
\centering
\begin{subfigure}[t]{.32\textwidth}
\includegraphics[width=\textwidth]{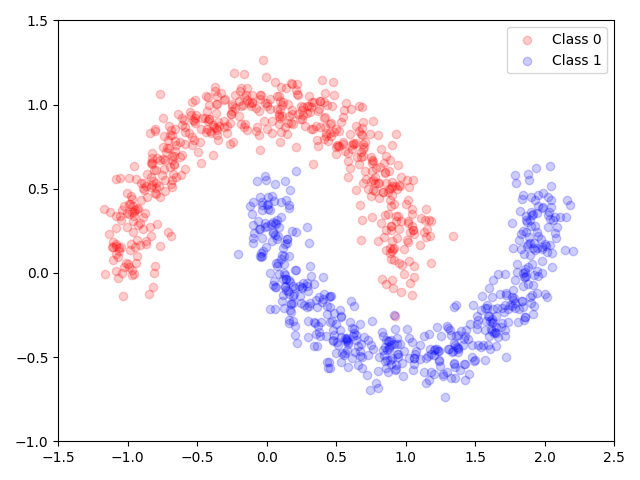}
\captionsetup{width=.9\linewidth}
\caption{The dataset used to train a binary classifier.}\label{fig:twomoon}  
\end{subfigure}
\begin{subfigure}[t]{.32\textwidth}
\captionsetup{width=.9\linewidth}
\includegraphics[width=\textwidth]{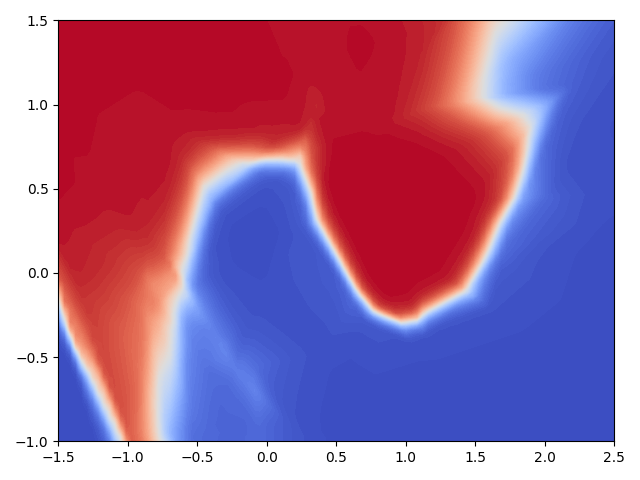}
\caption{The trained classifier, with ReLU activation, evaluated on the domain $[-1.5, 2.5]\times [-1., 1.5]$.}\label{fig:netvis}  
\end{subfigure}
\begin{subfigure}[t]{.32\textwidth}
\captionsetup{width=.9\linewidth}
\includegraphics[width=\textwidth]{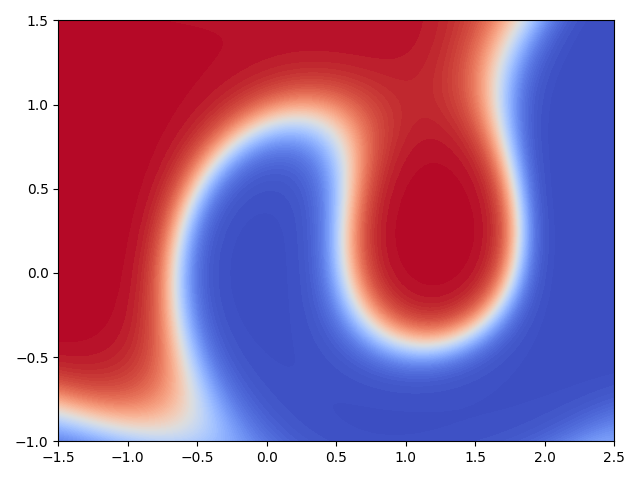}
\caption{The trained classifier, with GeLU activation, evaluated on the domain $[-1.5, 2.5]\times [-1., 1.5]$.}\label{fig:netvis_gelu}  
\end{subfigure}
\end{figure}
Using \texttt{PyTorch} \cite{paszke2017automatic}, we then train a neural network using the architecture displayed in \cref{fig:net} as proposed in \cite{Howard2022}, to obtain a mapping $h_\theta:\R^2\to [0,1]$, parametrized by $\theta$. Here \enquote{Linear $d^l\to d^{l+1}$} in the $l$th layer, denotes an \textit{affine} linear mapping \cite{rosenblatt1958perceptron} given by
\begin{align*}
z\mapsto W z + b, \qquad \text{with learnable parameters} \quad W\in\R^{d^{l+1}\times d^{l}}, b \in\R^{d^{l+1}}
\end{align*}
the activation functions \enquote{ReLU} \cite{fukushima1980neocognitron}, \enquote{GeLU} \cite{hendrycks2023gaussian} and \enquote{Sigmoid} are defined entry-wise for $i=1,\ldots, n$, as
\begin{align*}
\text{ReLU}(z_i) := \max\{0, z_i\},\qquad
\text{GeLU}(z_i) := z_i\cdot \Phi(z_i),\qquad
\text{Sigmoid}(z_i) := \frac{1}{1 + \exp(-z_i)},
\end{align*}
where $\Phi$ denotes the cumulative distribution function of the standard normal distribution. Here, we included both ReLU and GeLU (as a smooth approximation) to have an activation function, typically used in practice and a differentiable approximation fitting into the framework of \cref{ch: LinTi}. 
During training, we process batches of inputs $\vec{z} = (z^1,\ldots, z^B)$, with $z_i\in\R^{d^l}$, where \enquote{Batch Norm ($B$)}, as proposed in \cite{ioffe2015batch}, uses the entry-wise mean $\mu(\vec{z})_i := \frac{1}{B}\sum_{b=1}^B z_i^b$ and variance $\sigma(\vec{z})_i := \frac{1}{B}\sum_{b=1}^B (z_i^b - \mu(\vec z))^2$ and is defined as
\begin{align*}
z_i^b \mapsto \frac{z_i^b - \mu(\vec z)_i}{\sqrt{\sigma(\vec{z})_i^2 + \epsilon}} \cdot \gamma_i + \beta_k, \qquad\text{with learnable parameters}\quad \gamma, \beta \in\R^{d^l},
\end{align*}
where $\epsilon=10^{-5}$ is a small constant, added for numerical stability. During inference, the mean and variance are replaced by an estimate over the whole dataset, such that the output does not depend on the batch it is given.
In total, $\theta$ denotes the collection of weights $W$, biases $b$ and batch norm parameters $\gamma, \beta$ of all layers. For training, we consider the loss function
\begin{align}\label{eq:dataloss}
\mathcal{L}(\theta)= \frac{1}{2K} \sum_{k=1}^K 
\abs{h_\theta(\x_k) - y_k}^2,
\end{align}
where we employ the Adam optimizer \cite{kingma2017adam}, with standard learning rates, to approximate a minimizer. In each step, we employ a batched version of the function in \cref{eq:dataloss}, i.e., instead of using all data points at once, in each so-called \textit{epoch}, we randomly sample disjoint subsets of $\{1,\ldots, K\}$, of size $B=100$ and only sum over these points. We run this training process for a total of $100$ epochs, to obtain a set of parameters $\theta^*$, with a train loss of approximately $\mathcal{L}(\theta^*)\approx 0.002$ for ReLU and $\mathcal{L}(\theta^*)\approx 0.009$ for GeLU. The trained mappings $h_\theta$ are visualized in \cref{fig:netvis,fig:netvis_gelu}.

\begin{figure}
\centering%
\begin{tikzpicture}[block/.style={draw,minimum height=2.5em,rotate=90,minimum width=9em},>=stealth]
\node[block,fill=blue!20](L){Linear $2\to20$};
\node[block,fill=red!20] at ($(L)+(0:2.5em)$) (R){ReLU/GeLU};
\node[block,fill=green!20]at ($(R)+(0:2.5em)$) (B) {Batch Norm ($20$)};
\node[block,fill=blue!20]at ($(B)+(0:5em)$) (L){Linear $20\to20$};
\draw[->,thick] (B) -- (L);
\node[block,fill=red!20] at ($(L)+(0:2.5em)$) (R){ReLU/GeLU};
\node[block,fill=green!20]at ($(R)+(0:2.5em)$) (B) {Batch Norm ($20$)};
\node[block,fill=blue!20]at ($(B)+(0:5em)$) (L){Linear $20\to20$};
\draw[->,thick] (B) -- (L);
\node[block,fill=red!20] at ($(L)+(0:2.5em)$) (R){ReLU/GeLU};
\node[block,fill=green!20]at ($(R)+(0:2.5em)$) (B){Batch Norm ($20$)};
\node[block,fill=blue!20]at ($(B)+(0:5em)$) (L){Linear $20\to 1$};
\draw[->,thick] (B) -- (L);
\node[block,fill=red!20] at ($(L)+(0:2.5em)$) (R){Sigmoid};
\end{tikzpicture}
\caption{The network architecture used in the examples.}\label{fig:net}
\end{figure}
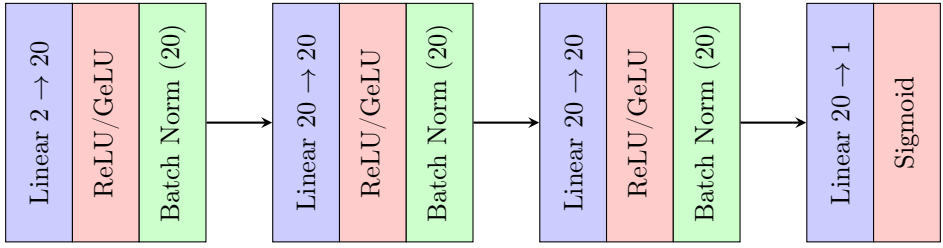

\subsection{Computing IGFSM and the minimizing movement scheme}

We now detail the iteration as displayed in \cref{fig:IFGSM}, first for ReLU. Here, we choose the initial value $\x^0 = (0.1, 0.55)$, as it is close to the decision boundary, with $h_{\theta^*}(\x^0)\approx 0.97$, an adversarial budget of $\budget=0.2$ and the energy
\begin{align*}
\funccp(\x):= \abs{h_{\theta^*}(\x) - 1}^2  + \chara_{B_\budget^\infty(\x^0)}.
\end{align*}
\ref{eq: IFGSM} is an explicit iteration and can therefore be implemented directly, where the gradient is computed with the automatic differentiation tools of \texttt{PyTorch}. For the minimizing movement scheme \ref{eq: MinMov}, we need to solve the problem
\begin{align*}
\x_\tau^{k+1}\in \argmin_{\x\in \cl{B_\tau^\infty}(\x_\tau^k)\cap \cl{B_\budget^\infty}(\x^0)} \funccp(\x) ,
\end{align*}
in each step. In order to avoid local minima, we do not employ a gradient based method here, but rather a particle based method, which allows exploring the full rectangle $\cl{B_\tau^\infty}(\x_\tau^k)$. We use consensus based optimization (CBO) as proposed in \cite{pinnau2017}, using the \texttt{CBXPy} package \cite{bailo2024cbx}. Concerning the hyperparameters, we choose $N=30$ particles, a noise scaling of $\sigma=2$, with standard isotropic noise, a time discretization parameter $dt=0.01$, $\alpha=10^8$ and perform $30$ update steps in each inner iteration. In order to ensure the budget constraint and the local restriction given by the step size $\tau$, we project the ensemble of the CBO iteration to the set
\begin{align*}
\cl{B^\infty_\budget}(\x^0)\cap  \cl{B^\infty_\tau}(\x^k_\tau)
\end{align*}
using the $\ell^\infty$ projection, i.e., a clipping operation. \rev{We refer to \cite{bungert2025mirrorcbo} for a more detailed numerical study considering projections in CBO schemes, which also suggests the validity of our method here.} We repeat the experiment for GeLU with a different initial value $\x^0=(0.45,0.3), \hyp_{\theta^*}(\x^0)\approx 0.74$ and budget $\budget=0.25$, which is displayed in \cref{fig:IFGSMgelu}.

\begin{figure}[t]
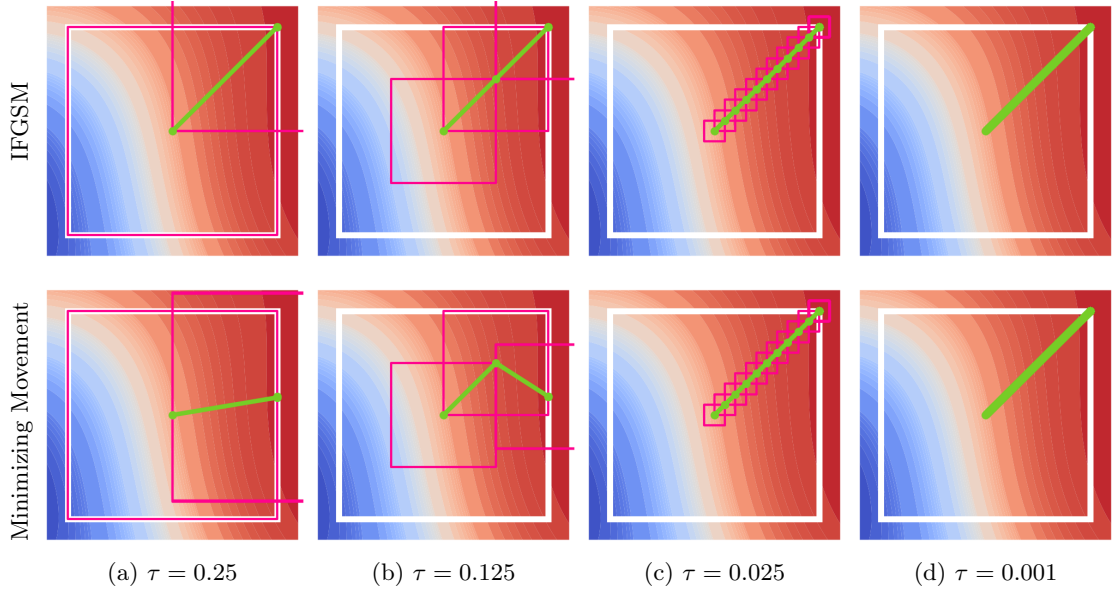

\centering
\rotatebox[origin=c]{90}{\makebox[4cm]{\small \phantom{g} IFGSM}}
\foreach \n in {0.25,0.125,0.025,0.001}{%
\begin{subfigure}[c]{.23\textwidth}
\includegraphics[width=\textwidth, trim={2cm 0cm 2cm 0cm},clip]{results/Flow_GeLU_fgsm\n.png}%
\end{subfigure}
}%

\rotatebox[origin=c]{90}{\makebox[4cm]{\small \phantom{---} Minimizing Movement}}
\foreach \n in {0.25,0.125,0.025,0.001}{%
\begin{subfigure}[c]{.23\textwidth}
\includegraphics[width=\textwidth, trim={2cm 0cm 2cm 0cm},clip]{results/Flow_GeLU_MinMove\n.png}%
\caption{$\tau=\n$}%
\end{subfigure}
}%
\caption{The same experiment as in \cref{fig:IFGSM}, but using a net employing the GeLU activation function.}
\label{fig:IFGSMgelu}
\end{figure}

\subsection{Convergence of the standard and semi-implicit scheme}

In this section, we consider the error between the standard and the semi-implicit minimizing movement, which serves as a very basic validation of the numerical schemes. Our theoretical framework shows that both iterations converge to a $\infty$-curve of maximum slope, which however is not available numerically. Instead, for $n\in\N$ and $k\leq n$, we can consider
\begin{align*}
\norm{\xs^k_{\tau_n} - \x^k_{\tau_n}}_\infty \leq 
\norm{\xs^k_{\tau_n} - u(k\cdot \tau_n)}_\infty +
\norm{\x^k_{\tau_n} - u(k\cdot \tau_n)}_\infty,
\end{align*}
where $\x^k_{\tau_n}$ fulfills the standard minimizing movement scheme and $\xs^k_{\tau_n}=\xifg^{k}_\tau$ is given by \labelcref{eq: IFGSM}, i.e., fulfills the semi-implicit scheme. Although, our theory does not provide concrete estimates or rates of the error between IFGSM and the minimizing movement scheme, we perform a small numerical experiment using the setup from above. For each choice of $\tau$ we sample $S=50$ different initial values $\x^{0,s}$ and compute the iterates $\xifg^{k,s}$ and $\x^{k,s}_{\tau}$ for $k\in\{1,\ldots, \lfloor 1/\tau\rfloor\}$ and compute the averaged maximal distance 
\begin{align}\label{eq:error}
e_\tau := 
\frac{1}{S}\sum_{s=1}^S \max_{k}\norm{\xifg^{k,s} - \x^{k,s}_\tau}_\infty.
\end{align}
The errors are plotted in \cref{fig:rates}. In both cases the errors converge to zero, however we observe that the order of convergence is higher for the GeLU function. We note that our theoretical results only provide a convergence statement for the differentiable case, therefore these results are in line with the analysis. In particular \cref{lem:semiexp} requires a Lipschitz differentiable gradient.
However, we hypothesize that the slower convergence in the ReLU case, actually comes from the non-implicit error as visualized in \cref{fig:relu}. There we mimic a situation enforced by the ReLU activation function. For $\tau>0.1$, the minimizing movement scheme always \enquote{jumps} across the non-differentiable line $x_1=0.1$, to the corner where the minimum on $\cl{B_\tau^\infty}(\x^0)$ is attained, which leads the following iterates to follow the gradient into the direction $(1,1)$. However, in this case the actual flow is given as $u(t):= (t,0)$, which, in this case, is more accurately prescribed by \labelcref{eq: FGSM}. In this regard, a more exhaustive study, both empirically and theoretically is required, which is left for future work.

\begin{figure}
\begin{subfigure}[t]{.5\textwidth}
\centering
\begin{tikzpicture}[scale=.8]
\begin{axis}[xmode=log,ymode=log, x dir=reverse, legend pos=south west, 
xlabel=Stepsize $\tau$,ylabel=Error $e_\tau$]
\addplot[color=apple, line width=3pt] table {
0.20000 0.16122
0.10000 0.07620
0.05000 0.06092
0.02500 0.03220
0.01250 0.02559
0.00625 0.02148
0.00313 0.01895
0.00156 0.01777
0.00078 0.01728
0.00039 0.01708
};
\addplot[color=sky, line width=3pt] table {
0.20000 0.15146
0.10000 0.07174
0.05000 0.05489
0.02500 0.03695
0.01250 0.02354
0.00625 0.01947
0.00313 0.00227
0.00156 0.00118
0.00078 0.00067
0.00039 0.00046
};
\legend{ReLU,GeLU}
\end{axis}
\end{tikzpicture}
\captionsetup{width=.9\linewidth}
\caption{Difference between IFGSM and the minimizing movement scheme, as defined in \cref{eq:error}, for different values of $\tau$, using ReLU and GeLU activation functions in the network architecture.}
\label{fig:rates}
\end{subfigure}
\begin{subfigure}[t]{.46\textwidth}
\centering%
\raisebox{.5cm}{%
\includegraphics[width=.73\textwidth]{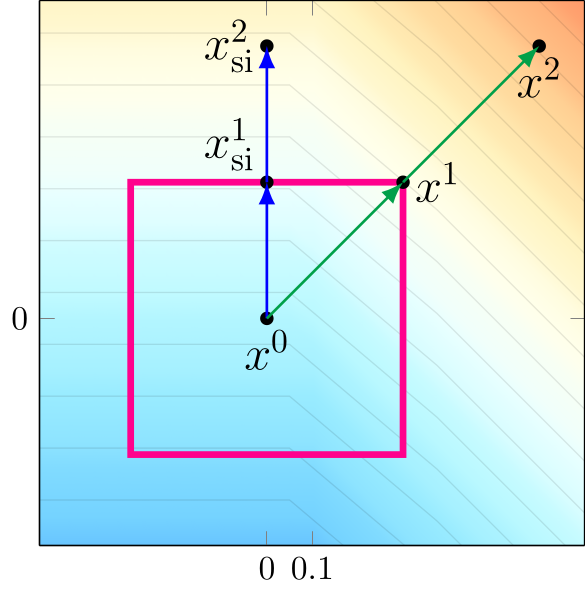}
}%
\captionsetup{width=.9\linewidth}
\caption{Minimizing movement scheme (blue arrows) and semi-implicit version (green arrows) for the function $\funccp(\x):= -\x_2 - \max\{\x_1, 0.1\}$ starting from $\x^0=(0,0)$.}\label{fig:relu}
\end{subfigure}
\caption{}
\end{figure}

\printbibliography

@book{ambrosio05,
  title={Gradient flows: in metric spaces and in the space of probability measures},
  author={Ambrosio, Luigi and Gigli, Nicola and Savar{\'e}, Giuseppe},
  year={2005},
  publisher={Springer Science \& Business Media}
}

@book{tao2011introduction,
  title={An introduction to measure theory},
  author={Tao, Terence},
  volume={126},
  year={2011},
  publisher={American Mathematical Soc.}
}

@article{Lisini07,
title={Characterization of absolutely continuous curves in Wasserstein spaces},
author={Lisini, Stefano},
year={2007},
month={01},
volume = {28},
pages={85-120},
journal = {Calculus of Variations and Partial Differential Equations},
doi = {10.1007/s00526-006-0032-2}
}

@online{lisini2014absolutely,
      title={Absolutely continuous curves in extended Wasserstein-Orlicz spaces}, 
      author={Stefano Lisini},
      year={2014},
      eprint={1402.7328},
      eprinttype={arXiv},
}

@book{Santambrogio15,
  title={Optimal Transport for Applied Mathematicians},
  author={Filippo Santambrogio},
  year={2015},
  publisher={Birkhäuser Cham}
}

@book{morrison2011functional,
  title={Functional analysis: An introduction to Banach space theory},
  author={Morrison, Terry J},
  year={2011},
  publisher={John Wiley \& Sons}
}

@inproceedings{madry2017towards,
title={Towards Deep Learning Models Resistant to Adversarial Attacks},
author={Aleksander Madry and Aleksandar Makelov and Ludwig Schmidt and Dimitris Tsipras and Adrian Vladu},
booktitle={International Conference on Learning Representations},
year={2018},
}

@book{cybenko1998mathematics,
  title={The mathematics of information coding, extraction and distribution},
  author={Cybenko, George and O'Leary, Dianne P and Rissanen, Jorma},
  volume={107},
  year={1998},
  publisher={Springer Science \& Business Media}
}

@article{boltzmann1868studien,
	title={{S}tudien über das {G}leichgewicht der lebenden {K}raft},
	author={Boltzmann, Ludwig},
	journal={Wissenschafiliche Abhandlungen},
	volume={1},
	pages={49--96},
	year={1868},
	publisher={Barth}
}

@book{euler1794institutiones,
  title={Institutiones calculi integralis},
  author={Euler, Leonhard},
  volume={4},
  year={1794},
  publisher={Academia Imperialis Scientiarum}
}

@article{CHEN2020109156,
title = {Sign projected gradient flow: A continuous-time approach to convex optimization with linear equality constraints},
journal = {Automatica},
volume = {120},
pages = {109156},
year = {2020},
issn = {0005-1098},
doi = {https://doi.org/10.1016/j.automatica.2020.109156},
url = {https://www.sciencedirect.com/science/article/pii/S000510982030354X},
author = {Fei Chen and Wei Ren},
}

@book{blanchard2015mathematical,
  title={Mathematical methods in Physics: Distributions, Hilbert space operators, variational methods, and applications in quantum physics},
  author={Blanchard, Philippe and Br{\"u}ning, Erwin},
  volume={69},
  year={2015},
  publisher={Birkh{\"a}user}
}

@article{fourier1827histoire,
  title={Histoire de l’Acad{\'e}mie, partie math{\'e}matique (1824)},
  author={Fourier, Joseph},
  journal={M{\'e}moires de l’Acad{\'e}mie des sciences de l’Institut de France},
  volume={7},
  year={1827}
}

@book{sierksma2015linear,
  title={Linear and integer optimization: theory and practice},
  author={Sierksma, Gerard and Zwols, Yori},
  year={2015},
  publisher={CRC Press}
}

@inproceedings{zhang2020sign,
  title={Sign gradient descent method based bat searching algorithm with application to the economic load dispatch problem},
  author={Zhang, Haopeng and Hui, Qing and Moulay, Emmanuel and Coirault, Patrick},
  booktitle={2020 59th IEEE Conference on Decision and Control (CDC)},
  pages={1140--1145},
  year={2020},
  organization={IEEE}
}

@book{Schaefer1999,
  title = {Topological Vector Spaces},
  ISBN = {9781461214687},
  ISSN = {0072-5285},
  url = {http://dx.doi.org/10.1007/978-1-4612-1468-7},
  DOI = {10.1007/978-1-4612-1468-7},
  journal = {Graduate Texts in Mathematics},
  publisher = {Springer New York},
  author = {Schaefer,  H. H. and Wolff,  M. P.},
  year = {1999}
}

@inproceedings{mohammadi2023sign,
  title={Sign Gradient Descent Algorithms for Kinetostatic Protein Folding},
  author={Mohammadi, Alireza and Al Janaideh, Mohammad},
  booktitle={2023 International Conference on Manipulation, Automation and Robotics at Small Scales (MARSS)},
  pages={1--6},
  year={2023},
  organization={IEEE}
}

@online{chzhen2023signsvrg,
      title={SignSVRG: fixing SignSGD via variance reduction}, 
      author={Evgenii Chzhen and Sholom Schechtman},
      year={2023},
      eprint={2305.13187},
      eprinttype={arXiv},
      primaryClass={math.OC}
}

@article{good1952rational,
  title={Rational decisions},
  author={Good, Irving John},
  journal={Journal of the Royal Statistical Society: Series B (Methodological)},
  volume={14},
  number={1},
  pages={107--114},
  year={1952},
  publisher={Wiley Online Library}
}

@article{MOULAY201929,
title = {Properties of the sign gradient descent algorithms},
journal = {Information Sciences},
volume = {492},
pages = {29-39},
year = {2019},
issn = {0020-0255},
doi = {https://doi.org/10.1016/j.ins.2019.04.012},
author = {Emmanuel Moulay and Vincent Léchappé and Franck Plestan},
}

@online{hendrycks2023gaussian,
      title={Gaussian Error Linear Units (GELUs)}, 
      author={Dan Hendrycks and Kevin Gimpel},
      year={2023},
      eprint={1606.08415},
      epprinttype={arXiv},
      primaryClass={cs.LG}
}

@article{bungert2024ratio,
      title={Ratio convergence rates for Euclidean first-passage percolation: Applications to the graph infinity Laplacian}, 
      author={Leon Bungert and Jeff Calder and Tim Roith},
      year={2024},
      journal={The Annals of Applied Probability},
      eprint={2210.09023},
      primaryClass={math.PR}
}

@online{Howard2022,
  author = {Sean T. Howard},
  title = {A 'Hello World' for PyTorch},
  year = 2022,
  url = {https://seanhoward.me/blog/2022/hello_world_pytorch/},
  urldate = {2024-05-06}
}

@article{paszke2017automatic,
  title={Automatic differentiation in PyTorch},
  author={Paszke, Adam and Gross, Sam and Chintala, Soumith and Chanan, Gregory and Yang, Edward and DeVito, Zachary and Lin, Zeming and Desmaison, Alban and Antiga, Luca and Lerer, Adam},
  year={2017}
}

@article{scikit-learn,
  title={Scikit-learn: Machine Learning in {P}ython},
  author={Pedregosa, F. and Varoquaux, G. and Gramfort, A. and Michel, V.
          and Thirion, B. and Grisel, O. and Blondel, M. and Prettenhofer, P.
          and Weiss, R. and Dubourg, V. and Vanderplas, J. and Passos, A. and
          Cournapeau, D. and Brucher, M. and Perrot, M. and Duchesnay, E.},
  journal={Journal of Machine Learning Research},
  volume={12},
  pages={2825--2830},
  year={2011}
}

@inproceedings{zheng2019distributionally,
  title={Distributionally adversarial attack},
  author={Zheng, Tianhang and Chen, Changyou and Ren, Kui},
  booktitle={Proceedings of the AAAI Conference on Artificial Intelligence},
  volume={33},
  number={01},
  pages={2253--2260},
  year={2019}
}

@inproceedings{dong2018boosting,
  title={Boosting adversarial attacks with momentum},
  author={Dong, Yinpeng and Liao, Fangzhou and Pang, Tianyu and Su, Hang and Zhu, Jun and Hu, Xiaolin and Li, Jianguo},
  booktitle={Proceedings of the IEEE conference on computer vision and pattern recognition},
  pages={9185--9193},
  year={2018}
}

@INPROCEEDINGS{Riedmiller93,
  author={Riedmiller, M. and Braun, H.},
  booktitle={IEEE International Conference on Neural Networks}, 
  title={A direct adaptive method for faster backpropagation learning: the RPROP algorithm}, 
  year={1993},
  volume={},
  number={},
  pages={586-591 vol.1},
  doi={10.1109/ICNN.1993.298623}}

@online{balles2020geometry,
      title={The Geometry of Sign Gradient Descent}, 
      author={Lukas Balles and Fabian Pedregosa and Nicolas Le Roux},
      year={2020},
      eprint={2002.08056},
      eprinttype={arXiv},
}

@article{li2023faster,
  title={On faster convergence of scaled sign gradient descent},
  author={Li, Xiuxian and Lin, Kuo-Yi and Li, Li and Hong, Yiguang and Chen, Jie},
  journal={IEEE Transactions on Industrial Informatics},
  year={2023},
  publisher={IEEE}
}

@inproceedings{ilyas2018black,
  title={Black-box adversarial attacks with limited queries and information},
  author={Ilyas, Andrew and Engstrom, Logan and Athalye, Anish and Lin, Jessy},
  booktitle={International conference on machine learning},
  pages={2137--2146},
  year={2018},
  organization={PMLR}
}

@article{brendel2017decision,
  title={Decision-based adversarial attacks: Reliable attacks against black-box machine learning models},
  author={Brendel, Wieland and Rauber, Jonas and Bethge, Matthias},
  journal={arXiv preprint arXiv:1712.04248},
  year={2017}
}

@inproceedings{jang2017objective,
  title={Objective metrics and gradient descent algorithms for adversarial examples in machine learning},
  author={Jang, Uyeong and Wu, Xi and Jha, Somesh},
  booktitle={Proceedings of the 33rd Annual Computer Security Applications Conference},
  pages={262--277},
  year={2017}
}

@article{pintor2021fast,
  title={Fast minimum-norm adversarial attacks through adaptive norm constraints},
  author={Pintor, Maura and Roli, Fabio and Brendel, Wieland and Biggio, Battista},
  journal={Advances in Neural Information Processing Systems},
  volume={34},
  pages={20052--20062},
  year={2021}
}

@article{brendel2019accurate,
  title={Accurate, reliable and fast robustness evaluation},
  author={Brendel, Wieland and Rauber, Jonas and K{\"u}mmerer, Matthias and Ustyuzhaninov, Ivan and Bethge, Matthias},
  journal={Advances in neural information processing systems},
  volume={32},
  year={2019}
}

@inproceedings{moosavi2016deepfool,
  title={Deepfool: a simple and accurate method to fool deep neural networks},
  author={Moosavi-Dezfooli, Seyed-Mohsen and Fawzi, Alhussein and Frossard, Pascal},
  booktitle={Proceedings of the IEEE conference on computer vision and pattern recognition},
  pages={2574--2582},
  year={2016}
}

@article{pauli2021training,
  title={Training robust neural networks using Lipschitz bounds},
  author={Pauli, Patricia and Koch, Anne and Berberich, Julian and Kohler, Paul and Allg{\"o}wer, Frank},
  journal={IEEE Control Systems Letters},
  volume={6},
  pages={121--126},
  year={2021},
  publisher={IEEE}
}

@inproceedings{Roth2020,
	title = {{Adversarial Training is a Form of Data-dependent Operator Norm Regularization}},
	year = {2019},
	booktitle = {NeurIPS},
	author = {Roth, Kevin and Kilcher, Yannic and Hofmann, Thomas}
}

@article{Krishnan2020,
  title={Lipschitz bounds and provably robust training by laplacian smoothing},
  author={Krishnan, Vishaal and Makdah, Al and AlRahman, Abed and Pasqualetti, Fabio},
  journal={Advances in Neural Information Processing Systems},
  volume={33},
  pages={10924--10935},
  year={2020}
}

@article{gouk2020regularisation,
	title={{Regularisation of neural networks by enforcing Lipschitz continuity}},
	author={Gouk, Henry and Frank, Eibe and Pfahringer, Bernhard and Cree, Michael J},
	journal={Machine Learning},
	pages={1--24},
	year={2020},
	publisher={Springer}
}

@online{bungert2024,
author = {Bungert, Leon and Laux, Tim and Stinson, Kerrek},
year = {2024},
title = {A mean curvature flow arising in adversarial training},
eprint = {2404.14402},
eprinttype={arXiv}
}

@inproceedings{bungert2021clip,
  title={CLIP: Cheap Lipschitz training of neural networks},
  author={Bungert, Leon and Raab, Ren{\'e} and Roith, Tim and Schwinn, Leo and Tenbrinck, Daniel},
  booktitle={Scale Space and Variational Methods in Computer Vision: 8th International Conference, SSVM 2021, Proceedings},
  pages={307--319},
  year={2021},
  doi={10.1007/978-3-030-75549-2_25},
  organization={Springer},
}

@article{shafahi2019adversarial,
  title={Adversarial training for free!},
  author={Shafahi, Ali and Najibi, Mahyar and Ghiasi, Mohammad Amin and Xu, Zheng and Dickerson, John and Studer, Christoph and Davis, Larry S and Taylor, Gavin and Goldstein, Tom},
  journal={Advances in neural information processing systems},
  volume={32},
  year={2019}
}

@online{wong2020fast,
  title={Fast is better than free: Revisiting adversarial training},
  author={Wong, Eric and Rice, Leslie and Kolter, J Zico},
  eprint={2001.03994},
  eprinttype = {arXiv},
  year={2020}
}

@book{dacorogna2007direct,
  title={Direct methods in the calculus of variations},
  author={Dacorogna, Bernard},
  volume={78},
  year={2007},
  publisher={Springer Science \& Business Media}
}

@incollection{kurakin2018adversarial,
  title={Adversarial examples in the physical world},
  author={Kurakin, Alexey and Goodfellow, Ian J and Bengio, Samy},
  booktitle={Artificial intelligence safety and security},
  pages={99--112},
  year={2018},
  publisher={Chapman and Hall/CRC}
}

@article{fleissner2019gamma,
  title={$\Gamma$-convergence and relaxations for gradient flows in metric spaces: a minimizing movement approach},
  author={Flei{\ss}ner, Florentine},
  journal={ESAIM: Control, Optimisation and Calculus of Variations},
  volume={25},
  pages={28},
  year={2019},
  publisher={EDP Sciences}
}

@article{stefanelli2022new,
  title={A new minimizing-movements scheme for curves of maximal slope},
  author={Stefanelli, Ulisse},
  journal={ESAIM: Control, Optimisation and Calculus of Variations},
  volume={28},
  pages={59},
  year={2022},
  publisher={EDP Sciences}
}

@article{BrLoSa2011,
author = {Brasco, Lorenzo and Santambrogio, Filippo},
year = {2011},
month = {03},
pages = {845-871},
title = {An equivalent path functional formulation of branched transportation problems},
volume = {29},
journal = {Discrete and Continuous Dynamical Systems},
doi = {10.3934/dcds.2011.29.845}
}

@article{Mielke2012,
title={Variational Convergence of Gradient Flows and Rate-Independent Evolutions in Metric Spaces},
author={Mielke, Alexander and Rossi, Riccarda and Savar{\'e}, Giuseppe},
year={2012},
journal={Milan Journal of Mathematics},
pages = {381--410},
volume = {80},
publisher={Springer Basel}
}

@article{bungert2023uniform,
  title={Uniform convergence rates for Lipschitz learning on graphs},
  author={Bungert, Leon and Calder, Jeff and Roith, Tim},
  journal={IMA Journal of Numerical Analysis},
  volume={43},
  number={4},
  pages={2445--2495},
  year={2023},
  publisher={Oxford University Press}
}

@article{bungert2023geometry,
  title={The geometry of adversarial training in binary classification},
  author={Bungert, Leon and Garc{\'i}a Trillos, Nicol{\'a}s and Murray, Ryan},
  journal={Information and Inference: A Journal of the IMA},
  volume={12},
  number={2},
  pages={921--968},
  year={2023},
  publisher={Oxford University Press}
}

@article{murray2019revisiting,
  title={Revisiting normalized gradient descent: Fast evasion of saddle points},
  author={Murray, Ryan and Swenson, Brian and Kar, Soummya},
  journal={IEEE Transactions on Automatic Control},
  volume={64},
  number={11},
  pages={4818--4824},
  year={2019},
  publisher={IEEE}
}

@article{levy2016power,
  title={The power of normalization: Faster evasion of saddle points},
  author={Levy, Kfir Y},
  journal={arXiv preprint arXiv:1611.04831},
  year={2016}
}

@article{hazan2015beyond,
  title={Beyond convexity: Stochastic quasi-convex optimization},
  author={Hazan, Elad and Levy, Kfir and Shalev-Shwartz, Shai},
  journal={Advances in neural information processing systems},
  volume={28},
  year={2015}
}

@inproceedings{cutkosky2020momentum,
  title={Momentum improves normalized sgd},
  author={Cutkosky, Ashok and Mehta, Harsh},
  booktitle={International conference on machine learning},
  pages={2260--2268},
  year={2020},
  organization={PMLR}
}

@inproceedings{turan2021robustness,
  title={On robustness of the normalized subgradient method with randomly corrupted subgradients},
  author={Turan, Berkay and Uribe, C{\'e}sar A and Wai, Hoi-To and Alizadeh, Mahnoosh},
  booktitle={2021 American Control Conference (ACC)},
  pages={965--971},
  year={2021},
  organization={IEEE}
}

@inproceedings{suzuki2021normalized,
  title={Normalized gradient descent for variational quantum algorithms},
  author={Suzuki, Yudai and Yano, Hiroshi and Raymond, Rudy and Yamamoto, Naoki},
  booktitle={2021 IEEE International Conference on Quantum Computing and Engineering (QCE)},
  pages={1--9},
  year={2021},
  organization={IEEE}
}

@article{stepanov2017three,
  title={Three superposition principles: currents, continuity equations and curves of measures},
  author={Stepanov, Eugene and Trevisan, Dario},
  journal={Journal of Functional Analysis},
  volume={272},
  number={3},
  pages={1044--1103},
  year={2017},
  publisher={Elsevier}
}

@article{armstrong2010easy,
  title={An easy proof of Jensen’s theorem on the uniqueness of infinity harmonic functions},
  author={Armstrong, Scott N and Smart, Charles K},
  journal={Calculus of Variations and Partial Differential Equations},
  volume={37},
  pages={381--384},
  year={2010},
  publisher={Springer}
}

@online{bungert2023begins,
  title={It begins with a boundary: A geometric view on probabilistically robust learning},
  author={Bungert, Leon and Trillos, Nicol{\'a}s Garc{\'\i}a and Jacobs, Matt and McKenzie, Daniel and Nikoli{\'c}, {\DJ}or{\dj}e and Wang, Qingsong},
  eprint={2305.18779},
  eprinttype={arXiv},
  year={2023}
}

@article{bungert2024gamma,
  title={Gamma-convergence of a nonlocal perimeter arising in adversarial machine learning},
  author={Bungert, Leon and Stinson, Kerrek},
  journal={Calculus of Variations and Partial Differential Equations},
  volume={63},
  number={5},
  pages={114},
  year={2024},
  publisher={Springer}
}

@article{bungert2020asymptotic,
  title={Asymptotic profiles of nonlinear homogeneous evolution equations of gradient flow type},
  author={Bungert, Leon and Burger, Martin},
  journal={Journal of Evolution Equations},
  volume={20},
  number={3},
  pages={1061--1092},
  year={2020},
  publisher={Springer}
}

@article{DeGiorgi80,
  author = {De Giorgi, Ennio and Marino, Antonio and Tosques, Mario},
  journal = {Atti della Accademia Nazionale dei Lincei. Classe di Scienze Fisiche, Matematiche e Naturali. Rendiconti},
  pages = {180-187},
  publisher = {Accademia Nazionale dei Lincei},
  title = {Problemi di evoluzione in spazi metrici e curve di massima pendenza},
  volume = {68},
  year = {1980},
}

@article{szegedy2013intriguing,
  title={Intriguing properties of neural networks},
  author={Szegedy, Christian and Zaremba, Wojciech and Sutskever, Ilya and Bruna, Joan and Erhan, Dumitru and Goodfellow, Ian and Fergus, Rob},
  journal={arXiv preprint arXiv:1312.6199},
  year={2013}
}

@book{ryan2002introduction,
  title={Introduction to tensor products of Banach spaces},
  author={\rev{Ryan, Raymond A}},
  volume={73},
  year={2002},
  publisher={Springer}
}

@book{fenchel1953convex,
  title={Convex cones, sets, and functions},
  author={Fenchel, Werner and Blackett, Donald W},
  year={1953},
  publisher={Princeton University, Department of Mathematics, Logistics Research Project}
}

@article{hausdorff1919halbstetige,
  title={{\"U}ber halbstetige Funktionen und deren Verallgemeinerung},
  author={Hausdorff, Felix},
  journal={Mathematische Zeitschrift},
  volume={5},
  number={3},
  pages={292--309},
  year={1919},
  publisher={Springer}
}

@article{moreau1965proximite,
  title={Proximit{\'e} et dualit{\'e} dans un espace hilbertien},
  author={Moreau, Jean-Jacques},
  journal={Bulletin de la Soci{\'e}t{\'e} math{\'e}matique de France},
  volume={93},
  pages={273--299},
  year={1965}
}

@article{cortes2006finite,
  title={Finite-time convergent gradient flows with applications to network consensus},
  author={Cort{\'e}s, Jorge},
  journal={Automatica},
  volume={42},
  number={11},
  pages={1993--2000},
  year={2006},
  publisher={Elsevier}
}

@inproceedings{duchi2008efficient,
  title={Efficient projections onto the l 1-ball for learning in high dimensions},
  author={Duchi, John and Shalev-Shwartz, Shai and Singer, Yoram and Chandra, Tushar},
  booktitle={Proceedings of the 25th international conference on Machine learning},
  pages={272--279},
  year={2008}
}

@article{ambrosio90,
  title={Metric space valued functions of bounded variation},
  author={Ambrosio, Luigi},
  journal={Annali della Scuola Normale Superiore di Pisa-Classe di Scienze},
  volume={17},
  number={3},
  pages={439--478},
  year={1990}
}

@book{saks37,
  title={Theory of the Integral},
  author={Saks, Stanis{\l}aw},
  year={1937}
}

@book{SchusterKaltenbacherHofmannKazimierski+2012,
url = {https://doi.org/10.1515/9783110255720},
title = {Regularization Methods in Banach Spaces},
author = {Thomas Schuster and Barbara Kaltenbacher and Bernd Hofmann and Kamil S. Kazimierski},
publisher = {De Gruyter},
address = {Berlin, Boston},
doi = {doi:10.1515/9783110255720},
isbn = {9783110255720},
year = {2012},
lastchecked = {2023-03-01}
}

@article{roith2023continuum,
  title={Continuum limit of Lipschitz learning on graphs},
  author={Roith, Tim and Bungert, Leon},
  journal={Foundations of Computational Mathematics},
  volume={23},
  number={2},
  pages={393--431},
  year={2023},
  publisher={Springer}
}

@book{barbu2012convexity,
  title={Convexity and optimization in Banach spaces},
  author={Barbu, Viorel and Precupanu, Teodor},
  year={2012},
  publisher={Springer Science \& Business Media}
}

@book{villani2009optimal,
  title={Optimal transport: old and new},
  author={Villani, C{\'e}dric and others},
  volume={338},
  year={2009},
  publisher={Springer}
}

@misc{dellacherie1978probabilities,
  title={Probabilities and potential, vol. 29 of North-Holland Mathematics Studies},
  author={Dellacherie, Claude and Meyer, Paul-Andr{\'e}},
  year={1978},
  publisher={North-Holland Publishing Co., Amsterdam}
}

@article{givens1984class,
  title={A class of Wasserstein metrics for probability distributions.},
  author={Givens, Clark R and Shortt, Rae Michael},
  journal={Michigan Mathematical Journal},
  volume={31},
  number={2},
  pages={231--240},
  year={1984},
  publisher={University of Michigan, Department of Mathematics}
}

@article{mielke2013nonsmooth,
  title={Nonsmooth analysis of doubly nonlinear evolution equations},
  author={Mielke, Alexander and Rossi, Riccarda and Savar{\'e}, Giuseppe},
  journal={Calculus of Variations and Partial Differential Equations},
  volume={46},
  pages={253--310},
  year={2013},
  publisher={Springer}
}

@article{degiorgi1993new,
    author = {E. De Giorgi},
    title = {New Problems on Minimizing Movements},
    journal = {Boundary Value Problems for PDEs and
Applications},
    year = {1993}
}

@article{rossi2008metric,
  title={A metric approach to a class of doubly nonlinear evolution equations and applications},
  author={Rossi, Riccarda and Mielke, Alexander and Savar{\'e}, Giuseppe},
  journal={Annali della Scuola Normale Superiore di Pisa-Classe di Scienze},
  volume={7},
  number={1},
  pages={97--169},
  year={2008}
}

@article{diestelvector,
  title={Vector measures (American Mathematical Society, Providence, RI, 1977)},
  author={Diestel, J and Uhl Jr, JJ},
  journal={With a foreword by BJ Pettis, Mathematical Surveys},
  number={15}
}

@online{kingma2017adam,
  title={Adam: A Method for Stochastic Optimization}, 
  author={Diederik P. Kingma and Jimmy Ba},
  year={2017},
  eprint={1412.6980},
  eprinttype={arXiv},
}

@article{rosenblatt1958perceptron,
  title={The perceptron: a probabilistic model for information storage and organization in the brain.},
  author={Rosenblatt, Frank},
  journal={Psychological review},
  volume={65},
  number={6},
  pages={386},
  year={1958},
  publisher={American Psychological Association}
}

@article{fukushima1980neocognitron,
  title={Neocognitron: A self-organizing neural network model for a mechanism of pattern recognition unaffected by shift in position},
  author={Fukushima, Kunihiko},
  journal={Biological cybernetics},
  volume={36},
  number={4},
  pages={193--202},
  year={1980},
  publisher={Springer}
}

@inproceedings{ioffe2015batch,
  title={Batch normalization: Accelerating deep network training by reducing internal covariate shift},
  author={Ioffe, Sergey and Szegedy, Christian},
  booktitle={International conference on machine learning},
  pages={448--456},
  year={2015},
  organization={pmlr}
}

@misc{bailo2024cbx,
      title={CBX: Python and Julia packages for consensus-based interacting particle methods}, 
      author={Rafael Bailo and Alethea Barbaro and Susana N. Gomes and Konstantin Riedl and Tim Roith and Claudia Totzeck and Urbain Vaes},
      year={2024},
      eprint={2403.14470},
      archivePrefix={arXiv},
      primaryClass={math.OC}
}

@article{pinnau2017,
author = {Pinnau, Ren\'{e} and Totzeck, Claudia and Tse, Oliver and Martin, Stephan},
title = {A consensus-based model for global optimization and its mean-field limit},
journal = {Mathematical Models and Methods in Applied Sciences},
volume = {27},
number = {01},
pages = {183-204},
year = {2017},
doi = {10.1142/S0218202517400061},
}

@article{Edg,
    author = "G. Edgar",
     title = "Measurability in a Banach Space",
   journal = "Indiana Univ. Math. J.",
  fjournal = "Indiana University Mathematics Journal",
    volume = 26,
      year = 1977,
     issue = 4,
     pages = "663--677",
      issn = "0022-2518",
     coden = "IUMJAB",
   mrclass = "",
}

@book{billingsley2013convergence,
  title={Convergence of probability measures},
  author={Billingsley, Patrick},
  year={2013},
  publisher={John Wiley \& Sons}
}

@article{marino1989curves,
  title={Curves of maximal slope and parabolic variational inequalities on non-convex constraints},
  author={Marino, Antonio and Saccon, Claudio and Tosques, Mario},
  journal={Annali della Scuola Normale Superiore di Pisa-Classe di Scienze},
  volume={16},
  number={2},
  pages={281--330},
  year={1989}
}

@article{degiovanni1985evolution,
  title={Evolution equations with lack of convexity},
  author={Degiovanni, Marco and Marino, Antonio and Tosques, Mario},
  journal={Nonlinear Analysis: Theory, Methods \& Applications},
  volume={9},
  number={12},
  pages={1401--1443},
  year={1985},
  publisher={Elsevier}
}

@book{guide2006infinite,
  title={Infinite dimensional analysis},
    subtitle={A Hitchhiker’s Guide},
  author={Charalambos D. Aliprantis,
Kim C. Border},
  year={2006},
  publisher={Springer},
version={Third Edition}
}

@article{gruntkowska2025ball,
  title={The Ball-Proximal (=" Broximal") Point Method: a New Algorithm, Convergence Theory, and Applications},
  author={Gruntkowska, Kaja and Li, Hanmin and Rane, Aadi and Richt{\'a}rik, Peter},
  journal={arXiv preprint arXiv:2502.02002},
  year={2025}
}

@article{rockafellar1970maximal,
  title={On the maximal monotonicity of subdifferential mappings},
  author={Rockafellar, Ralph},
  journal={Pacific Journal of Mathematics},
  volume={33},
  number={1},
  pages={209--216},
  year={1970},
  publisher={Mathematical Sciences Publishers}
}

@article{bungert2025mirrorcbo,
  title={MirrorCBO: A consensus-based optimization method in the spirit of mirror descent},
  author={Bungert, Leon and Hoffmann, Franca and Kim, Doh Yeon and Roith, Tim},
  journal={arXiv preprint arXiv:2501.12189},
  year={2025}
}

@inproceedings{staib2017distributionally,
  title={Distributionally robust deep learning as a generalization of adversarial training},
  author={Staib, Matthew and Jegelka, Stefanie},
  booktitle={NIPS workshop on Machine Learning and Computer Security},
  volume={3},
  pages={4},
  year={2017}
}

@inproceedings{meunier2021mixed,
  title={Mixed nash equilibria in the adversarial examples game},
  author={Meunier, Laurent and Scetbon, Meyer and Pinot, Rafael B and Atif, Jamal and Chevaleyre, Yann},
  booktitle={International Conference on Machine Learning},
  pages={7677--7687},
  year={2021},
  organization={PMLR}
}

@inproceedings{mehrabi2021fundamental,
  title={Fundamental tradeoffs in distributionally adversarial training},
  author={Mehrabi, Mohammad and Javanmard, Adel and Rossi, Ryan A and Rao, Anup and Mai, Tung},
  booktitle={International Conference on Machine Learning},
  pages={7544--7554},
  year={2021},
  organization={PMLR}
}

@article{sinha2017certifying,
  title={Certifying some distributional robustness with principled adversarial training},
  author={Sinha, Aman and Namkoong, Hongseok and Volpi, Riccardo and Duchi, John},
  journal={arXiv preprint arXiv:1710.10571},
  year={2017}
}

@article{pydi2021many,
  title={The many faces of adversarial risk},
  author={Pydi, Muni Sreenivas and Jog, Varun},
  journal={Advances in Neural Information Processing Systems},
  volume={34},
  pages={10000--10012},
  year={2021}
}

@article{bungert2021nonlinear,
  title={Nonlinear spectral decompositions by gradient flows of one-homogeneous functionals},
  author={Bungert, Leon and Burger, Martin and Chambolle, Antonin and Novaga, Matteo},
  journal={Analysis \& PDE},
  volume={14},
  number={3},
  pages={823--860},
  year={2021},
  publisher={Mathematical Sciences Publishers}
}

@inproceedings{pydi2020adversarial,
  title={Adversarial risk via optimal transport and optimal couplings},
  author={Pydi, Muni Sreenivas and Jog, Varun},
  booktitle={International Conference on Machine Learning},
  pages={7814--7823},
  year={2020},
  organization={PMLR}
}

@article{bui2022unified,
  title={A unified wasserstein distributional robustness framework for adversarial training},
  author={Bui, Tuan Anh and Le, Trung and Tran, Quan and Zhao, He and Phung, Dinh},
  journal={arXiv preprint arXiv:2202.13437},
  year={2022}
}

@article{goodfellow2014explaining,
  title={Explaining and harnessing adversarial examples},
  author={Goodfellow, Ian J and Shlens, Jonathon and Szegedy, Christian},
  journal={arXiv preprint arXiv:1412.6572},
  year={2014}
}

@article{kurakin2016adversarial,
  title={Adversarial machine learning at scale},
  author={Kurakin, Alexey and Goodfellow, Ian and Bengio, Samy},
  journal={arXiv preprint arXiv:1611.01236},
  year={2016}
}
\end{document}